\documentclass{article} 
\usepackage{iclr2026_conference,times}

\usepackage{amsmath,amsfonts,bm}

\def\eqref#1{equation~\ref{#1}}

\def\1{\bm{1}}

\DeclareMathAlphabet{\mathsfit}{\encodingdefault}{\sfdefault}{m}{sl}
\SetMathAlphabet{\mathsfit}{bold}{\encodingdefault}{\sfdefault}{bx}{n}

\usepackage{hyperref}
\usepackage{url}

\usepackage[utf8]{inputenc} 
\usepackage[T1]{fontenc}    
\usepackage{hyperref}       
\usepackage{url}            
\usepackage{booktabs}       
\usepackage{amsfonts}       
\usepackage{nicefrac}       
\usepackage{microtype}      
\usepackage{xcolor}         
\usepackage{tikz}
\usepackage{subcaption}

\usepackage{amsmath}
\usepackage{amssymb}
\usepackage{mathtools}
\usepackage{amsthm}
\usepackage{bm}
\usepackage{bbm}
\usepackage{dsfont}
\usepackage{enumitem}
\usepackage[capitalize,noabbrev]{cleveref}

\theoremstyle{plain}
\newtheorem{theorem}{Theorem}[section]
\newtheorem{proposition}[theorem]{Proposition}
\newtheorem{condition}[theorem]{Condition}
\newtheorem{lemma}[theorem]{Lemma}
\newtheorem{corollary}[theorem]{Corollary}
\theoremstyle{definition}
\newtheorem{definition}[theorem]{Definition}

\theoremstyle{remark}

\title{Understanding Private Learning From Feature Perspective}

\author{
Meng Ding$^{1}$,
Mingxi Lei$^{1}$,
Shaopeng Fu$^{2}$,
Shaowei Wang$^{3}$,
Di Wang$^{2}$\thanks{Correspondence to: Di Wang <di.wang@kaust.edu.sa>, Jinhui Xu <jhxu@ustc.edu.cn>.},
Jinhui Xu$^{4}$\\[4pt]
$^{1}$Department of Computer Science and Engineering, State University of New York at Buffalo, USA\\
$^{2}$Division of CEMSE, King Abdullah University of Science and Technology, Saudi Arabia\\
$^{3}$Institute of Artificial Intelligence and Blockchain, Guangzhou University, China\\
$^{4}$School of Information Science and Technology, University of Science and Technology of China, China
}

\iclrfinalcopy 
\begin{document}

\maketitle

\begin{abstract}
Differentially private Stochastic Gradient Descent (DP-SGD) has become integral to privacy-preserving machine learning, ensuring robust privacy guarantees in sensitive domains. Despite notable empirical advances leveraging features from non-private, pre-trained models to enhance DP-SGD training, a theoretical understanding of feature dynamics in private learning remains underexplored. This paper presents the first theoretical framework to analyze private training through a feature learning perspective. Building on the multi-patch data structure from prior work, our analysis distinguishes between label-dependent feature signals and label-independent noise—a critical aspect overlooked by existing analyses in the DP community. Employing a two-layer CNN with polynomial ReLU activation, we theoretically characterize both feature signal learning and data noise memorization in private training via noisy gradient descent. Our findings reveal that (1) Effective private signal learning requires a higher signal-to-noise ratio (SNR) compared to non-private training, and (2) When data noise memorization occurs in non-private learning, it will also occur in private learning, leading to poor generalization despite small training loss. Our findings highlight the challenges of private learning and prove the benefit of feature enhancement to improve SNR. Experiments on synthetic and real-world datasets also validate our theoretical findings.
\end{abstract}

\section{Introduction}

Differentially private (DP) learning has emerged as a cornerstone of privacy-preserving machine learning, addressing growing concerns about data privacy in sensitive domains such as healthcare \cite{lundervold2019overview,chlap2021review,shamshad2023transformers}, finance 
 \cite{ozbayoglu2020deep,bi2024advanced}, and user-centric applications \cite{oroojlooy2023review}. Differential Privacy, introduced by \cite{dwork2006calibrating}, provides robust privacy guarantees by limiting the impact of any individual data points on the model's output. Among DP learning methods, differentially private stochastic gradient descent (DP-SGD) \cite{abadi2016deep} has emerged as a canonical algorithm for training private machine learning models.

 However, DP-SGD often comes with a significant cost in model accuracy \cite{shokri2015privacy,abadi2016deep,bagdasaryan2019differential}. To improve the performance, recent work \cite{tramer2020differentially} shows that DP-SGD training benefits from handcrafted features and can achieve better performance by leveraging features learned from public data in a similar domain. Similarly, \cite{tang2024differentially} highlights the advantages of transferring features learned from synthetic data to private training, while \cite{sun2023importance} and \cite{bao2023dp} illustrate the importance of feature preprocessing in private learning.
 These findings suggest that improving feature quality is essential for effective private learning. Benefiting from this principle and the advent of large-scale foundation models, DP-SGD has demonstrated significant performance boosts by learning features from non-private models pre-trained on large public datasets \cite{tramer2020differentially,li2021large,de2022unlocking,arora2022can,kurakin2022toward,mehta2023towards,nasr2023effectively,tang2024private,bu2024pre}. 

 Despite the empirical success of DP-SGD from enhanced features, the theoretical understanding of these phenomena remains in its infancy. Previous work on DP learning has primarily focused on analyzing the utility bounds of private models, such as DP-SGD and its variants, with a particular emphasis on both convex  \cite{bassily2014private, wang2017differentially, bassily2019private, feldman2020private, song2020characterizing, su2021faster, asi2021private, bassily2021non, kulkarni2021private,DBLP:conf/ijcai/TaoW0W22,su2023differentially,su2024faster} and non-convex models \cite{zhang2017efficient,wang2017differentially,wang2019differentially,zhang2021private, bassily2021differentially,wang2023efficient,dingrevisiting}, leaving the role and explanation from the feature learning perspective largely unexplored.

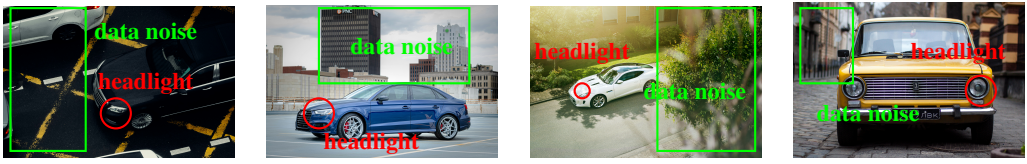
\begin{figure*}[t]
\vspace{-0.2in}
\centering
\begin{tikzpicture}
    \node[anchor=south west, inner sep=0] (img1) at (0, 0) {\includegraphics[width=0.22\textwidth]{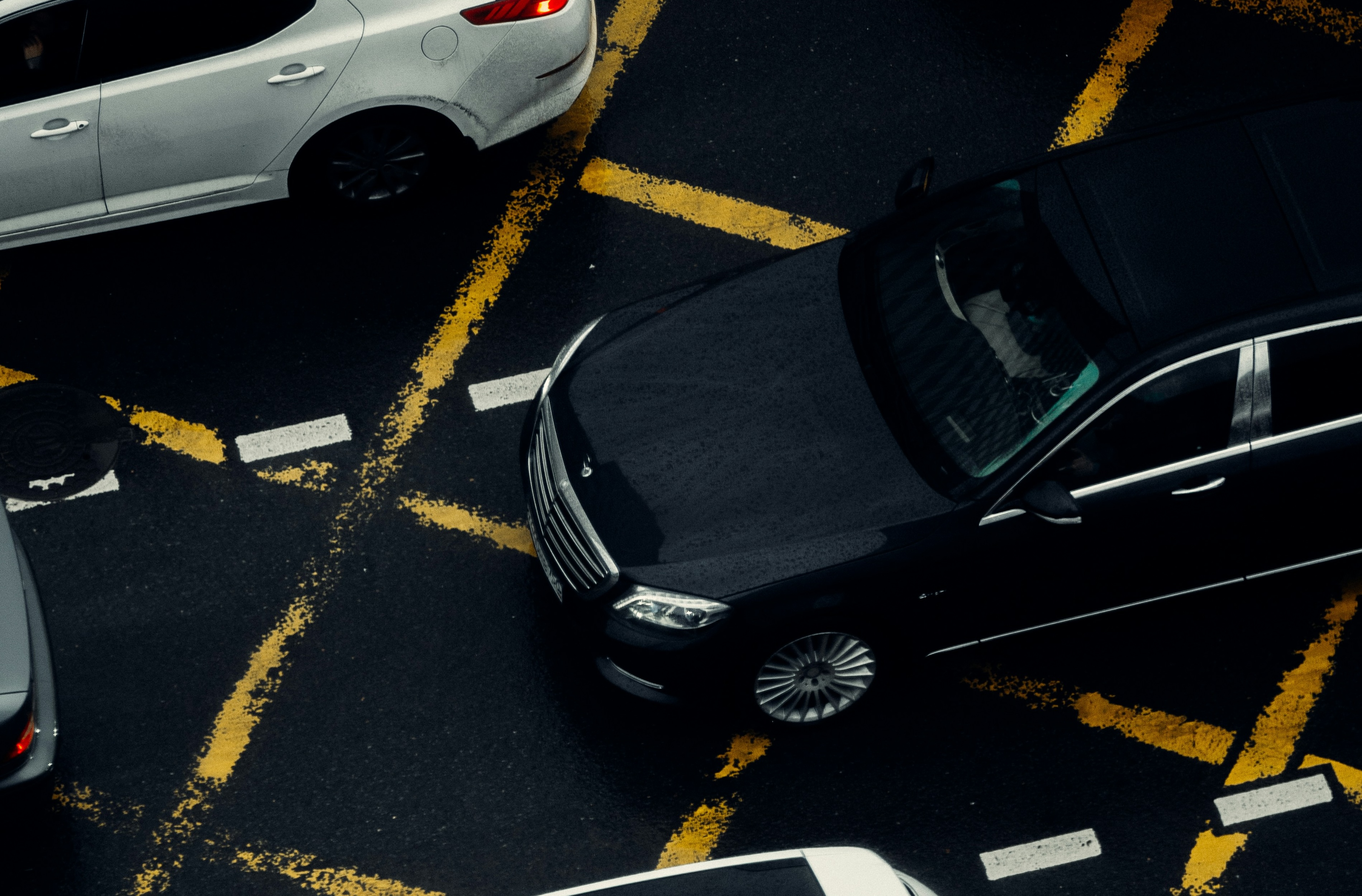}};
    \draw[red, thick] (1.5, 0.6) circle (0.2);
    \node[red] at (1.9, 1.0) {\textbf{\small headlight}};
    \draw[green, thick] (0.1,2.0) rectangle (1.1,0.1);
    \node[green] at (1.9, 1.7) {\textbf{\small data noise}};
    
    \node[anchor=south west, inner sep=0] (img2) at (3.5, 0) {\includegraphics[width=0.22\textwidth]{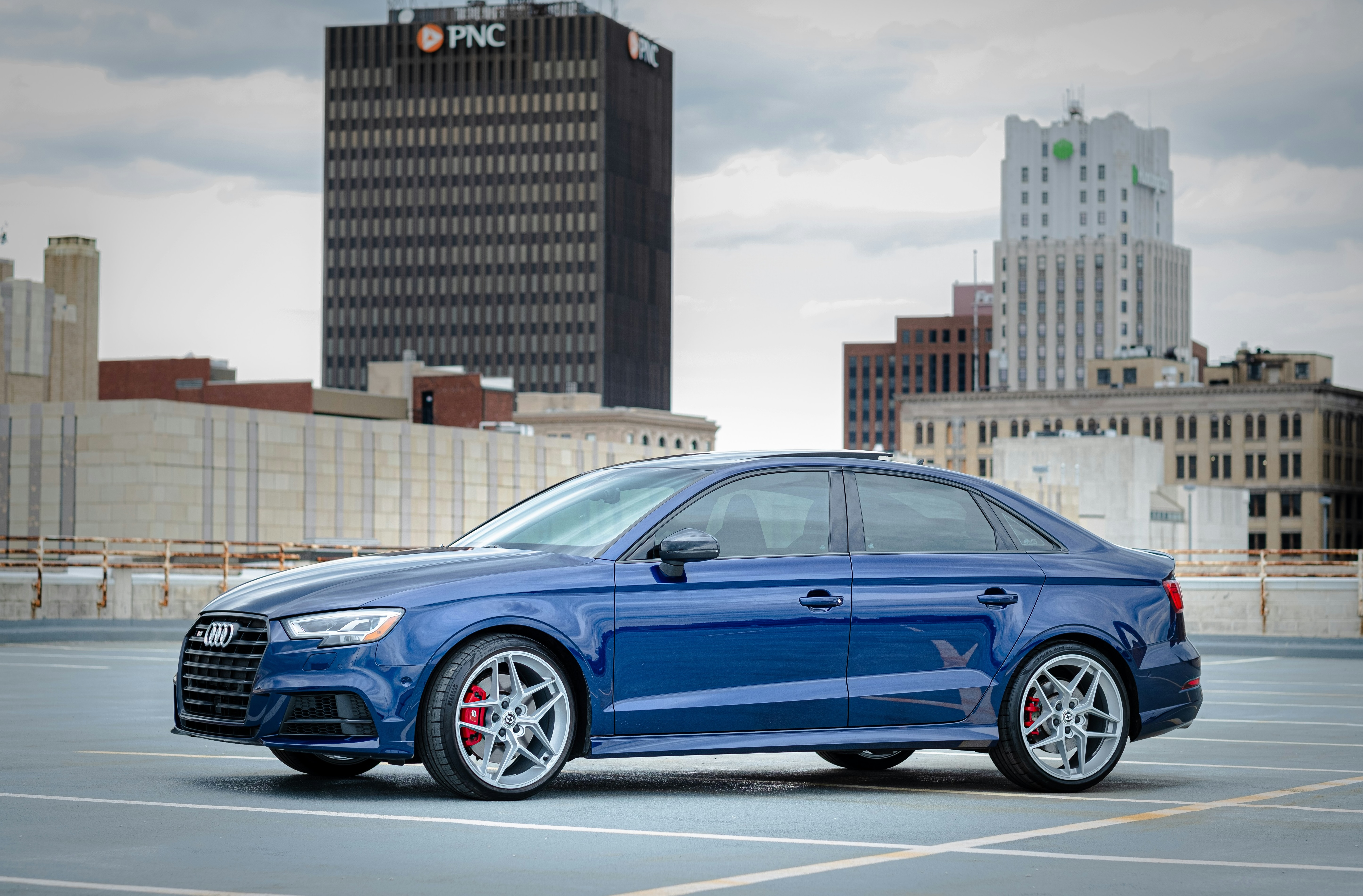}};
    \draw[red, thick] (4.2, 0.6) circle (0.2);
    \node[red] at (4.9, 0.2) {\textbf{\small headlight}};
    \draw[green, thick] (4.2,2.0) rectangle (6.2,1.0);
    \node[green] at (5.3, 1.5) {\textbf{\small data noise}};
    
    \node[anchor=south west, inner sep=0] (img3) at (7.0, 0) {\includegraphics[width=0.22\textwidth]{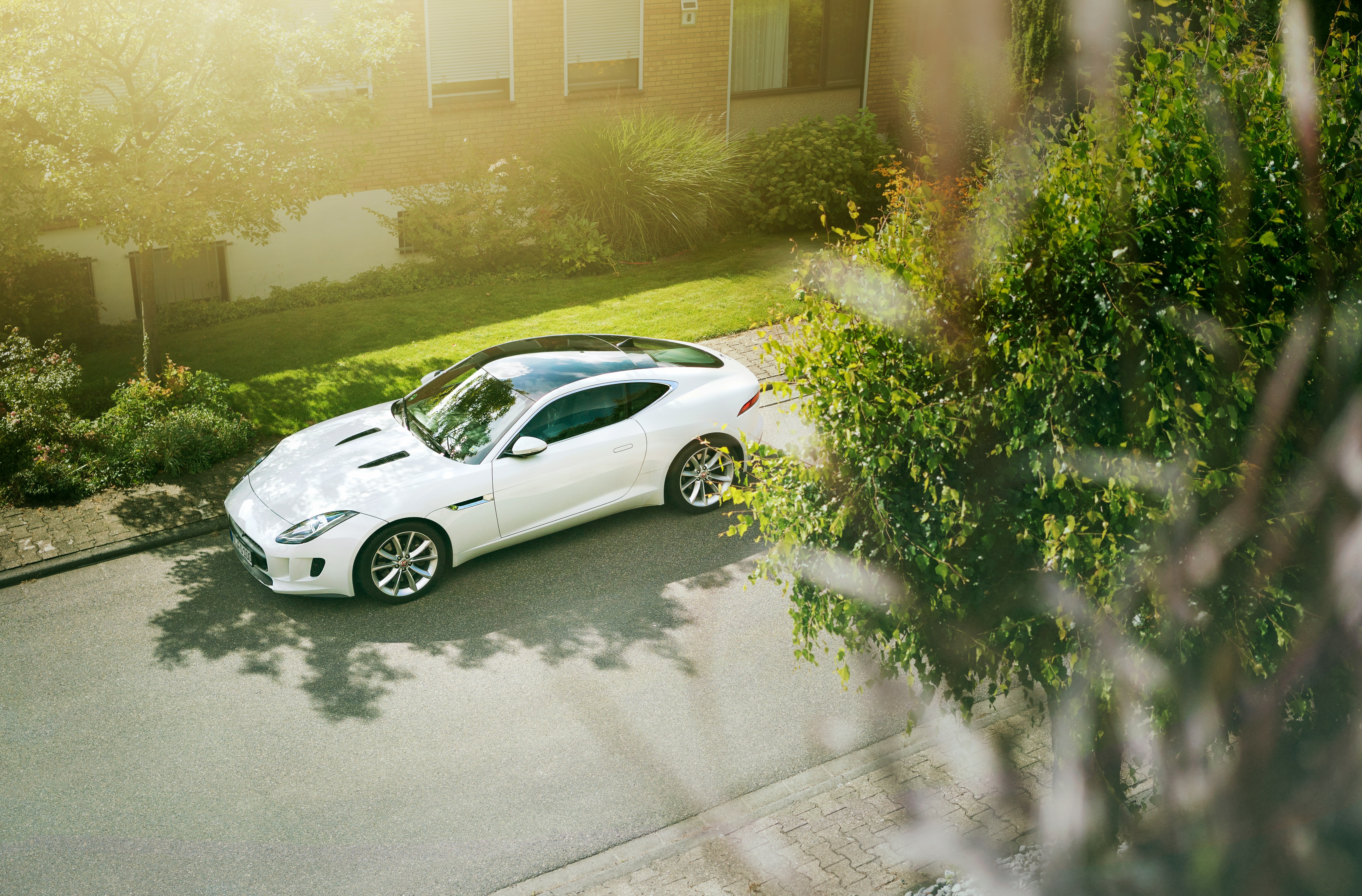}};
    \draw[red, thick] (7.7, 0.9) circle (0.1);
    \node[red] at (7.7, 1.4) {\textbf{\small headlight}};
    \draw[green, thick] (8.7,2.0) rectangle (10.0,0.1);
    \node[green] at (9.2, 0.9) {\textbf{\small data noise}};
    
    \node[anchor=south west, inner sep=0] (img4) at (10.5, 0) {\includegraphics[width=0.22\textwidth]{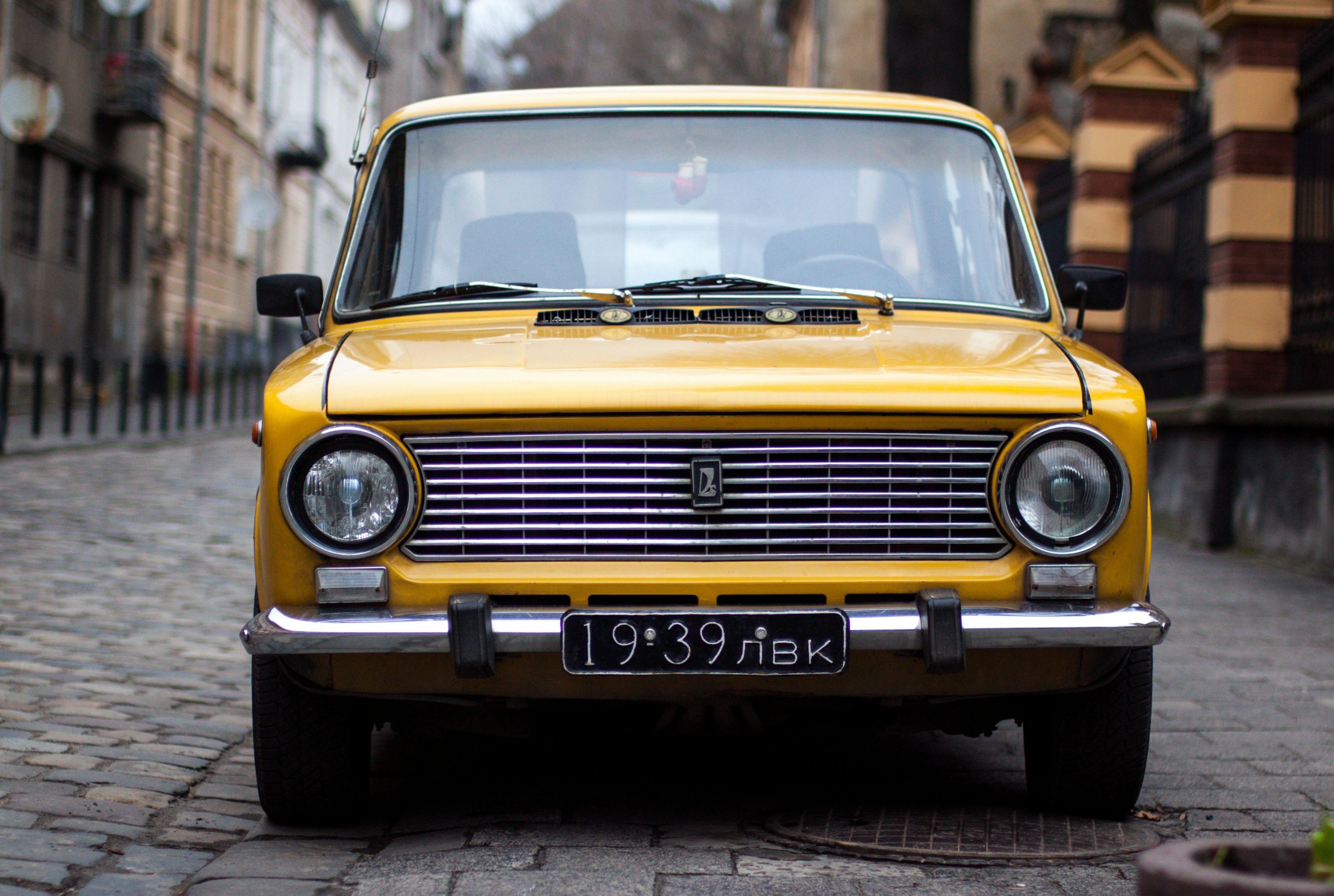}};
    \draw[red, thick] (13.0, 0.9) circle (0.2);
    \node[red] at (12.7, 1.4) {\textbf{\small headlight}};
    \draw[green, thick] (10.6,2.0) rectangle (11.3,1.0);
    \node[green] at (11.5, 0.6) {\textbf{\small data noise}};
\end{tikzpicture}
\caption{Illustration of images with feature signal and data noise.}
\label{fig:image_features}
\vspace{-0.1in}
\end{figure*}
 
Only two recent works have studied the theoretical aspects of features in private learning, both with several limitations. \cite{sun2023importance} confines its analysis to simple tasks using linear classification models without addressing applicability to neural networks. \cite{wang2024neural} investigates feature shifts during private fine-tuning of the last layer under the framework of neural collapse \cite{papyan2020prevalence}, simplifying the private model with the assumption of the equiangular tight frame (ETF). Furthermore, both works focus exclusively on utility, providing limited explanations of the learning dynamics of features in private learning. We will provide more discussions later.

In this paper, we develop a novel theoretical framework
that studies the private learning dynamics from a feature learning perspective in noisy gradient descent, a simple version of DP-SGD. Inspired by the structure of image data, we consider a data distribution modeled as a multiple-patch structure, $\mathbf{x} = [ y \cdot \mathbf{v}, \bm{\xi}] \in (\mathbb{R}^d)^2$, where $y \in \{+1,-1\}$ is the label, $\mathbf{v}$ represents the useful label-dependent feature signals, $\bm{\xi}$ refers to label-independent data noise randomly sampled from a Gaussian distribution with standard deviation $\sigma_{\xi}$ and $d$ is the dimension. For example, as illustrated in \cref{fig:image_features}, the wheel serves as a feature for the class 'car,' while the background acts as label-independent data noise. 

Beyond the linear classification model, we utilize a two-layer convolutional neural network (CNN) with a polynomial ReLU activation function: $\sigma(z)=\max \{0, z\}^q$, where $q>2$ is a hyperparameter.  Given a training dataset of $n$ samples, we quantify noisy gradient descent in terms of feature signal learning and data noise memorization, measured through the private model $\mathbf{w}$ with signal and data noise. Specifically, we present the following (informal) results:
\begin{theorem}[Informal]\label{thm:informal}
    Let $\operatorname{SNR}:=\|\mathbf{v}\|_2 /\|\bm{\xi}\|_2$ \footnote{{Our data model is analogous to the patch model in \cite{cao2022benign}: the signal vector $\mathbf{v}$ is fixed, while the noise $\boldsymbol{\xi}$ is drawn from an isotropic Gaussian. When the dimension $d$ is large, $\|\boldsymbol{\xi}\|_2 \approx (\mathbb{E}\|\boldsymbol{\xi}\|_2^2)^{1/2}$ by standard concentration, so we treat $\|\mathbf{v}\|_2/\|\boldsymbol{\xi}\|_2$ (and hence $\|\mathbf{v}\|_2/(\sigma_\xi \sqrt{d})$) as the signal-to-noise ratio for ease of discussion similar to \cite{cao2022benign}.}} be the signal-to-noise ratio and $\varepsilon$ be the privacy budget. Under appropriate conditions, it holds that
    \begin{itemize}[leftmargin=1.5em]
        \item When $\min \{\operatorname{SNR} \cdot n \varepsilon, \operatorname{SNR}^q \cdot n\} \geq \widetilde{\Omega}(1)$, the private CNN model can capture the feature signal.
        \item When $\min \{\operatorname{SNR}^{-1} \cdot \varepsilon, \operatorname{SNR}^{-q} \cdot n^{-1} \} \geq \widetilde{\Omega}(1)$, the private CNN model can capture the data noise.
    \end{itemize}
\end{theorem}
\cref{thm:informal} demonstrates the two interesting results during private training: $\mathbf{1)}$ When $\min \{\operatorname{SNR} \cdot n \varepsilon, \operatorname{SNR}^q \cdot n\} \geq \widetilde{\Omega}(1)$ and $n^{-1/2} \leq \varepsilon \leq \operatorname{SNR}^{q-1}$, a lower privacy budget requires a higher $\operatorname{SNR}$ compared to standard non-private training to effectively capture the signal, emphasizing the need for feature enhancement to improve $\operatorname{SNR}$. $\mathbf{2)}$ When data noise memorization occurs in standard non-private learning, it will also occur in private learning as long as $\varepsilon\geq \operatorname{SNR}^{1-q}n^{-1}$.


Moreover, under additional data assumptions, we  have the following results:
\begin{corollary}[Informal]\label{coro:informal}
    Let $\operatorname{SNR}:=\|\mathbf{v}\|_2 /\|\bm{\xi}\|_2$ be the signal-to-noise ratio and $\varepsilon$ be the privacy budget. Under appropriate conditions and assumptions, for any $\kappa > 0$, it holds that
    \begin{itemize}[leftmargin=1.5em]
        \item When $\min \{\operatorname{SNR} \cdot n \varepsilon, \operatorname{SNR}^q \cdot n\} \geq \widetilde{\Omega}(1)$, the training loss can converge to $\kappa$, and the trained CNN achieves a test loss of $6 \kappa+\exp ({(n\varepsilon)}^{-1-1/q})$.
        \item When $\min \{\operatorname{SNR}^{-1} \cdot \varepsilon, \operatorname{SNR}^{-q} \cdot n^{-1} \} \geq \widetilde{\Omega}(1)$, the training loss can converge to $\kappa$, but the trained CNN incurs a constant-order test loss regardless of the sample size $n$ and privacy budget $\varepsilon$.
    \end{itemize}
\end{corollary}
\cref{coro:informal} demonstrate two conclusions. First, in an ideal scenario, private learning can achieve an arbitrarily small training loss, and its test performance is influenced by both the sample size $n$ and privacy budgets $\varepsilon$. Second, even if private learning achieves a small training loss, it may still fail to deliver good test performance, regardless of how the sample size and privacy budget are chosen. This limitation arises because the private model primarily learns label-independent data noise rather than label-dependent feature signals. We summarize our contributions below:
\begin{itemize}[leftmargin=1.5em]
    \item We present the first theoretical framework of the DP-SGD dynamic through a feature learning perspective. Building on the multi-patch data structure introduced in \cite{allen2020towards, cao2022benign}, our framework explicitly distinguishes between label-dependent features and label-independent noise—a critical aspect overlooked by existing analyses of DP-SGD. This distinction enables us to introduce the signal-to-noise ratio (SNR) as an essential factor influencing the private training process.
    \item We provide a detailed theoretical analysis of feature signal learning and data noise memorization in the private setting with a two-layer convolutional neural network model. Specifically, based on the signal-to-noise ratio, we show that $\mathbf{1)}$ Effective private signal learning requires a higher signal-to-noise ratio compared to non-private training.
    $\mathbf{2)}$ When data noise memorization occurs in standard non-private learning, it will also occur in private learning as long as $\varepsilon\geq \operatorname{SNR}^{1-q}n^{-1}$. Consequently, the private model fails to generalize well, even when achieving a small training loss, regardless of sample size and privacy budget.
    \item Our findings underscore the importance of feature enhancement techniques in improving SNR for effective private learning, aligning with the principles established in previous empirical work \cite{tramer2020differentially,sun2023importance,bao2023dp,tang2024differentially}. We conduct simulation experiments on CNNs and validate our theoretical analysis across various privacy budgets and signal-to-noise ratios. 
\end{itemize}
Due to space constraints, we defer additional related work, extended experimental results, and the conclusion to the Appendix.

\section{Related Work}
\label{sec:related_work}


\noindent {\bf Theory on Differentially Private Learning}
A growing body of work has focused on the differential private optimization problems, including standard results for private empirical risk minimization \cite{chaudhuri2011differentially,bassily2014private,wang2017differentially,wang2019differentially} and private stochastic convex optimization \cite{bassily2019private,feldman2020private,bassily2021differentially}. These studies have also been extended under various assumptions, such as heavy-tailed data \cite{wang2020differentially, hu2022high, kamath2022improved} and non-Euclidean spaces \cite{bassily2021differentially, asi2021private,su2023differentially}. 
Despite extensive research on DP optimization theory, the theoretical understanding of private deep learning remains largely unexplored, particularly from the feature perspective. Only two recent studies have explored the theoretical aspects of features in private learning. Specifically, \cite{sun2023importance} focuses on a linear classification model, whereas we analyze a more challenging two-layer neural network model with a polynomial ReLU activation function. \cite{wang2024neural} considers a last-layer model converging to the columns of an equiangular tight frame (ETF), which simplifies the learned features to normal vectors via a rotation map and the model is still in a linear form. In contrast, our approach goes beyond this simplification. Moreover, while \cite{wang2024neural} emphasizes feature shift behavior, our work primarily focuses on training dynamics and the importance of feature enhancement.

\textbf{Feature Learning Theory} Feature learning has emerged as a powerful framework for the theoretical understanding of deep learning in recent years. It was first introduced by \cite{allen2022feature} to explain the advantages of adversarial training in achieving robustness and was then extended by \cite{allen2020towards} to demonstrate the role of ensemble methods in improving generalization. Then, the framework has been applied to a broad range of model architectures, including graph neural networks \cite{huang2023graph}, convolutional neural networks \cite{cao2022benign,kou2023benign}, vision transformers \cite{jelassi2022vision,li2023theoretical}, and diffusion models \cite{han2024feature}. In addition to architectural studies, the feature learning perspective has also been employed to analyze the behavior of various optimization algorithms and training strategies, such as Adam \cite{zou2023benefits}, momentum-based methods \cite{jelassi2022towards}, and data augmentation techniques like Mixup \cite{zou2023benefits}. To the best of our knowledge, this work is the first to examine private learning through the lens of feature learning. Unlike standard learning, private learning introduces additional challenges—most notably, the injection of private noise can destabilize training, necessitating a careful analysis of the learning dynamics.

\section{Preliminaries}
\label{sec:preliminaries}
In this section, we introduce the necessary definitions and formally describe private learning under the multi-patch data distribution and the convolutional neural network (CNN). Our analysis focuses on binary classification, and the data distribution is defined as follows.



\begin{definition}[Data Distribution]\label{def:data_dis}
Let $\mathbf{v} \in \mathbb{R}^d$ be a fixed vector representing the feature signal contained in each data point. Each data point $(\mathbf{x},y)$ with input $\mathbf{x} = [\mathbf{x}_1, \mathbf{x}_2] \in (\mathbb{R}^d)^2$ and label $y \in \{+1, -1\}$ is generated from the following distribution:
\begin{itemize}
    \item[(1)] The label $y$ is generated as a Rademacher random variable;
    \item[(2)] The input $\mathbf{x}$ is generated as a vector of $2$ patches, i.e., $\mathbf{x} = [\mathbf{x}_1, \mathbf{x}_2] \in (\mathbb{R}^d)^2$. The first patch is given by $\mathbf{x}_1= y \cdot \mathbf{v}$ and the second patch is given by $\mathbf{x}_2=$ $\bm{\xi}$, where $\bm{\xi} \sim \mathcal{N} (0, \sigma_{\xi}^2 \cdot \mathbf{H} )$ and is independent of the label $y$, where $\mathbf{H}=(\mathbf{I}-\mathbf{v} \mathbf{v}^{\top} \cdot\|\mathbf{v}\|_2^{-2})$.
    \vspace{-0.05in}
\end{itemize}
\end{definition}
Note that here $\mathbf{H}$ is designed to ensure that $\mathbf{v}$ is orthogonal to $\bm{\xi}$, i.e., the data noise is unrelated to the feature. 
Our data generation model is inspired by the structure of image data, which has been widely utilized in the feature learning area \cite{allen2020towards,cao2022benign,jelassi2022towards,kou2023benign,zou2023benefits}. Notably, we introduce a term, data noise $\bm{\xi}$, into the data distribution, which is often overlooked in analyses within the differential privacy community. However, this seemingly `negligible' component significantly influences the model's generalization ability, as underscored by the signal-to-noise ratio. 

\textbf{Learner Model.} We consider a two-layer convolutional neural network (CNN) that processes input data by applying convolutional filters to two patches, $\mathbf{x}_1$ and $\mathbf{x}_2$, separately. The second-layer parameters of this network are fixed as $+1 / m$ and $-1 / m$, respectively, leading to the following network representation:
\vspace{-0.05in}
\begin{equation*}
    f(\mathbf{W}, \mathbf{x})=F_{+1}(\mathbf{W}_{+1}, \mathbf{x})-F_{-1}(\mathbf{W}_{-1}, \mathbf{x}),
\end{equation*}
where $F_{+1}(\mathbf{W}_{+1}, \mathbf{x})$ and $F_{-1}(\mathbf{W}_{-1}, \mathbf{x})$ are defined as:
\begin{equation}
    F_j(\mathbf{W}_j, \mathbf{x})=\frac{1}{m} \sum_{r=1}^m[\sigma(\langle\mathbf{w}_{j, r}, \mathbf{x}_1\rangle)+\sigma(\langle\mathbf{w}_{j, r}, \mathbf{x}_2\rangle)],
    \vspace{-0.05in}
\end{equation}
for $j \in\{ \pm 1\}$, where $m$ denotes the number of convolutional filters in each of $F_{+1}$ and $F_{-1}$. Here, $ \sigma(z)=(\max \{0, z\})^q$ represents the polynomial $\operatorname{ReLU}$ activation function with $q > 2$, $\mathbf{W}_j$ denotes the set of model weights associated with $F_j$, corresponding to the positive or negative filters. Each weight vector $\mathbf{w}_{j, r} \in \mathbb{R}^d$ is the parameters of the $r$-th neuron/filter in $\mathbf{W}_j$. We use $\mathbf{W}$ to represent the complete set of model weights across all filters. 

\textbf{Differential Private Learning.} Given a training dataset $D=\{(\mathbf{x}_i, y_i)\}_{i=1}^n$, sampled from a joint distribution $\mathcal{D}$ over $\mathbf{x} \times y$, the goal is to train the learner model by minimizing the following empirical risk, measured by logistic loss, while simultaneously preserving privacy:
\begin{equation}
\label{eq:L_D}
\vspace{-0.05in}
    L_D(\mathbf{W})=\frac{1}{n} \sum_{i=1}^n \ell[y_i \cdot f(\mathbf{W}, \mathbf{x}_i)],
\end{equation}
where $\ell(z)=\log (1+\exp (-z))$. More formally, the trained private model $\mathbf{W}$ should satisfy the mathematical definition of differential privacy as follows:
\begin{definition}[\cite{dwork2006calibrating}]
    Two datasets $D,D'$ are neighbors if they differ by only one element, which is denoted as $D \sim D'$. A randomized algorithm $\mathcal{A}$ is $(\varepsilon,\delta)$-differentially private (DP) if for all adjacent datasets $D,D'$ and for all events $S$ in the output space of $\mathcal{A}$, we have 
	$\mathbb{P}(\mathcal{A}(D)\in S)\leq e^{\varepsilon} \cdot \mathbb{P}(\mathcal{A}(D')\in S)+\delta.$
\end{definition}

\begin{definition}[Gaussian Mechanism \cite{dwork2010differential}]
    For a function $f: \mathcal{X}^n \mapsto \mathbb{R}^d$ with $L_2$-sensitivity $\Delta_2(f) = \max_{D, D'} \|f(D) - f(D')\|_2$, where $D$ and $D'$ are neighboring datasets, the Gaussian Mechanism outputs $f(D) + \mathbf{z}$. Here, $\mathbf{z} \sim \mathcal{N}(0, \sigma_z^2 \mathbb{I}_d)$ is Gaussian noise with scale $\sigma_z \geq \frac{\Delta_2(f) \sqrt{2 \ln (1.25 / \delta)}}{\varepsilon}.$
    This mechanism satisfies $(\varepsilon, \delta)$-differential privacy.
\end{definition}


\textbf{Noisy Gradient Descent.} Noisy Gradient Descent (NoisyGD) and its stochastic counterpart, Noisy Stochastic Gradient Descent \cite{song2013stochastic,abadi2016deep}, are fundamental algorithms in differentially private deep learning. In this paper, we apply the NoisyGD algorithm to optimize \cref{eq:L_D} and to update the filters in the CNN with the Gaussian mechanism.\footnote{Note that, similar to previous studies on the theory of DP-SGD,  we assume there is no clipping on gradients.} Specifically, 
\begin{align} \label{eq:dpsgd}
    \mathbf{w}_{j, r}^{(t+1)}&=\mathbf{w}_{j, r}^{(t)}-\eta \cdot ( \nabla_{\mathbf{w}_{j, r}} L_D(\bm{W}^{(t)}) + \mathbf{z}_t  )\\ \nonumber 
    & =  \mathbf{w}_{j, r}^{(t)}-\frac{\eta}{n m} \sum_{i=1}^n \ell_i^{\prime(t)} \cdot \sigma^{\prime}(\langle\mathbf{w}_{j, r}^{(t)}, \bm{\xi}_i\rangle) \cdot j y_i \bm{\xi}_i  
     -\frac{\eta}{n m} \sum_{i=1}^n \ell_i^{\prime (t)} \cdot \sigma^{\prime}(\langle\mathbf{w}_{j, r}^{(t)}, {y}_i \mathbf{v}\rangle) \cdot j \mathbf{v}  -\eta \mathbf{z}_t.  
\end{align}
where $\mathbf{z}_t$ is the private noise sampled from $\mathcal{N}(0, \sigma_z^2 \mathbb{I}_d)$, $\ell_i^{\prime(t)}$ is a shorthand notation of $\ell^{\prime}[y_i \cdot f(\mathbf{W}^{(t)}, \mathbf{x}_i)]$.
We assume that the noisy gradient descent algorithm starts from a Gaussian initialization, where each element of $\mathbf{W}_{+1}$ and $\mathbf{W}_{-1}$ is drawn from a Gaussian distribution $N(0, \sigma_0^2)$ with $\sigma_0^2$ representing the variance. 
\vspace{-0.1in}
\section{Main Results}
\vspace{-0.1in}
In this section, we present our main theoretical results, demonstrating how the signal-noise-decomposition \cite{cao2022benign,jelassi2022towards,kou2023benign,zou2023benefits} behaves during private learning using noisy gradient descent. It is clear that \cref{eq:dpsgd} can be represented as a linear combination of random initialization, the signal feature, the data noise, and the accumulation of private noise, which can be formulated as the following definition.

\begin{definition}\label{def:decom_coefficient_main}
    Let $\mathbf{w}_{j, r}^{(t)}$ for $j \in\{ \pm 1\}, r \in[m]$ be the convolution filters of the CNN at the $t$-th iteration of noisy gradient descent. Then there exist unique coefficients $\Gamma_{j, r}^{(t)} \geq 0$ and $\Phi_{j, r, i}^{(t)}$ such that
    \begin{equation*}
        \mathbf{w}_{j, r}^{(t)}=\mathbf{w}_{j, r}^{(0)}+j \cdot \Gamma_{j, r}^{(t)} \cdot \frac{\mathbf{v}}{\|\mathbf{v}\|_2^{2}} +\sum_{i=1}^n \Phi_{j, r, i}^{(t)} \cdot \frac{\bm{\xi}_i}{\|\bm{\xi}_i\|_2^{2} } - \eta \sum_{s=1}^{t} \mathbf{z}_s. 
    \end{equation*}
    We further denote $\bar{\Phi}_{j, r, i}^{(t)}:=\Phi_{j, r, i}^{(t)} \mathds{1}(\Phi_{j, r, i}^{(t)} \geq 0), \underline{\Phi}_{j, r, i}^{(t)}:=\Phi_{j, r, i}^{(t)} \mathds{1}(\underline{\Phi}_{j, r, i}^{(t)} \leq 0)$. Then, we have that
    \begin{equation}\label{eq:w_t_decom_main_2}
        \mathbf{w}_{j, r}^{(t)}=\mathbf{w}_{j, r}^{(0)}+j \cdot \Gamma_{j, r}^{(t)} \cdot\|\mathbf{v}\|_2^{-2} \cdot \mathbf{v}+\sum_{i=1}^n \bar{\Phi}_{j, r, i}^{(t)} \cdot\|\bm{\xi}_i\|_2^{-2} \cdot \bm{\xi}_i 
        +\sum_{i=1}^n \underline{\Phi}_{j, r, i}^{(t)} \cdot\|\bm{\xi}_i\|_2^{-2} \cdot \bm{\xi}_i - \eta \sum_{s=1}^{t} \mathbf{z}_s. 
    \end{equation}
    \vspace{-0.1in}
\end{definition}
In the decomposition of \cref{eq:w_t_decom_main_2}, $\Gamma_{j, r}^{(t)}$ represents the extent to which the model learns the feature signal from data, whereas $\bar{\Phi}_{j, r, i}^{(t)}$ quantifies the degree of data noise memorization by the model. Both components are influenced by the interplay of private noise, as shown in the following lemma.
\begin{lemma}\label{lem:the_evolution_of_coefficient_main}
    The coefficients $\Gamma_{j, r}^{(t)}, \bar{\Phi}_{j, r, i}^{(t)}, \underline{\Phi}_{j, r, i}^{(t)}$ in \cref{def:decom_coefficient_main} satisfy the following equations:
    \begin{align*}
    & \Gamma_{j, r}^{(0)}, \bar{\Phi}_{j, r, i}^{(0)}, \underline{\Phi}_{j, r, i}^{(0)}=0 \\
    &\Gamma_{j, r}^{(t+1)}=\Gamma_{j, r}^{(t)}-\frac{\eta}{n m} \cdot \sum_{i=1}^n \ell_i^{\prime (t)} \cdot \sigma^{\prime}(\langle\mathbf{w}_{j, r}^{(t)}, y_i \cdot \mathbf{v}\rangle) \cdot\|\mathbf{v}\|_2^2\\
    & \bar{\Phi}_{j, r, i}^{(t+1)}=\bar{\Phi}_{j, r, i}^{(t)}-\frac{\eta}{n m} \cdot \ell_i^{\prime (t)} \cdot \sigma^{\prime}(\langle\mathbf{w}_{j, r}^{(t)}, \bm{\xi}_i\rangle) \cdot\|\bm{\xi}_i\|_2^2 \cdot \mathds{1}(y_i=j), \\
    & \underline{\Phi}_{j, r, i}^{(t+1)}=\underline{\Phi}_{j, r, i}^{(t)}+\frac{\eta}{n m} \cdot \ell_i^{\prime (t)} \cdot \sigma^{\prime}(\langle\mathbf{w}_{j, r}^{(t)}, \bm{\xi}_i\rangle) \cdot\|\bm{\xi}_i\|_2^2 \cdot \mathds{1}(y_i=-j)
    \end{align*}
\end{lemma}

\cref{lem:the_evolution_of_coefficient_main} reveals that the process of private learning can be explored by the iterative dynamics of $\Gamma_{j, r}^{(t)}$, $\bar{\Phi}_{j, r, i}^{(t)}$ and $\underline{\Phi}_{j, r, i}^{(t)}$. It is noticed that private noise only influences the interior of  $\sigma^{\prime}(\cdot)$. Since $\ell_i^{\prime (t)} < 0$, \cref{lem:the_evolution_of_coefficient_main} provides favorable properties for the dynamics of $\Gamma_{j, r}^{(t)}$, $\bar{\Phi}_{j, r, i}^{(t)}$: $\Gamma_{j, r}^{(t)}$ and $\bar{\Phi}_{j, r, i}^{(t)}$ increase monotonically, while $\underline{\Phi}_{j, r, i}^{(t)}$ decreases monotonically. Then, we could demonstrate that these coefficients remain bounded throughout the private training process.

\begin{proposition}\label{pro:coeff_main}
Under certain conditions, let $T^*=\eta^{-1} \operatorname{poly}(\kappa^{-1},\|\mathbf{v}\|_2^{-1}, d^{-1} \sigma_{\xi}^{-2}, \sigma_0^{-1}, n, m, d)$ and $T_p^* = \min \{ T^*, \eta^{-1} {C mn \varepsilon \sigma_0 \mu^{-1} (\|\mathbf{v}\|_2 + \|\bm{\xi}\|_2)^{-1}}
\},$ where $C = 4\log(T^*) = \widetilde{O}(1)$ and $\mu = \max \{1, \| \mathbf{v}\|_2, \|\bm{\xi}\|_2\}$.
Then, with at least probability $1-1/d$, it holds that, for $t \leq  T_p^*$ :
\begin{itemize}
    \item $0 \leq \Gamma_{j, r}^{(t)}, \bar{\Phi}_{j, r, i}^{(t)} \leq 4 \log (T_p^*)$ for all $j \in\{ \pm 1\}, r \in[m]$ and $i \in[n]$.
    \item $0 \geq \underline{\Phi}_{j, r, i}^{(t)} \geq-2 \max _{i, j, r}\{|\langle\mathbf{w}_{j, r}^{(0)}, \mathbf{v}\rangle|,|\langle\mathbf{w}_{j, r}^{(0)}, \bm{\xi}_i\rangle|\}-16 n \sqrt{\frac{\log (4 n^2 / \delta)}{d}} -0.2 \geq - 4 \log (T_p^*)$ for all $j \in\{ \pm 1\}$, $r \in[m]$ and $i \in[n]$.
\end{itemize}
\end{proposition}
Compared to standard training \cite{cao2022benign,kou2023benign}, private learning also ensures bounded coefficients but further restricts the number of training iterations due to the cumulative effect of private noise. 

\begin{lemma}\label{lem:xi_zt2_main}
    For any iteration $t $, with at least probability $1-1/d$, it holds that 
    {\small \begin{align}
        |\eta \sum_{s=1}^t\langle\mathbf{z}_s, \mathbf{v}\rangle | \leq {\frac{ \eta C \sqrt{t T}\|\mathbf{v}\|_2^2 \log(d)}{mn \varepsilon} (1+{\frac{1}{\operatorname{SNR}}})},  \label{eq:zt_xi_main}
        |\eta \sum_{s=1}^t\langle\mathbf{z}_s, \bm{\xi}_i \rangle | \leq  {\frac{\eta C \sqrt{t T} \|\bm{\xi}\|_2^2 \log(d)}{mn \varepsilon} (1+{\operatorname{SNR}})}. 
    \end{align}}
\end{lemma}

{\cref{lem:xi_zt2_main} characterizes the influence of private noise $(\mathbf{z}_t)$ on the training process by decomposing its interaction with the feature signal $(\mathbf{v})$ and data noise $(\bm{\xi})$. According to \cref{eq:dpsgd}, let $E$ denote the \emph{good event} given by the intersection of the events in \cref{lem:xi_zt2_main} and \cref{pro:bounds_of_coeeficients} (Appendix), on which all internal coefficients remain bounded. On this event $E$, the per-step gradient has $\ell_2$-sensitivity at most $\frac{C(\|\mathbf{v}\|_2+\|\bm{\xi}\|_2)}{nm}$, where $C = \widetilde{O}(1)$ follows from the existing work \cite{cao2022benign}.
Under \cref{def:data_dis}, we show that $\underset{\text { data} \sim \text {\cref{def:data_dis}} }{\operatorname{Pr}}(E) \geq 1-\beta$ for a small $\beta$ (e.g., polynomially small in $1 / d$ ). Conditioned on $E$, the Gaussian noise added by NoisyGD is calibrated exactly as in the Gaussian mechanism, so the algorithm is $(\varepsilon,\delta)$-DP according to our DP definition. Thus, under the generative model of Definition~\ref{def:data_dis}, we obtain the following \emph{distributional} DP guarantee:
$\underset{\text { data} \sim \text {\cref{def:data_dis}} }{\operatorname{Pr}}[\text { NoisyGD is }(\varepsilon, \delta)-\mathrm{DP}] \geq 1-\beta.$}

Moreover, $T$ denotes the total training iterations, which differs from $T_p^*$, the maximum permissible iterations. Moreover, recall the definition of $\operatorname{SNR}:=\|\mathbf{v}\|_2/\|\bm{\xi}\|_2$, then the \cref{eq:zt_xi_main} can be further represented as 
{\small$
 |\eta \sum_{s=1}^t\langle\mathbf{z}_s, \bm{\xi}_i \rangle | \leq  {\frac{\eta C \sqrt{t T} \|\mathbf{v}\|_2^2 \log(1/\delta)}{mn \varepsilon} (\frac{1}{\operatorname{SNR}^2} +\frac{1}{\operatorname{SNR}})}. $}
It indicates that the cumulative influences of private noise on both the feature signal and data noise are affected by the SNR. 

Our analysis considers an over-parameterized model, which guarantees sufficient capacity to capture meaningful feature signals. However, this expressive power also raises a critical question: given the over-parameterization, the model may also learn substantial data noise. Thus, under what conditions does the model prioritize learning feature signals over data noise? To answer this, we examine two scenarios characterized by the signal-to-noise ratio, highlighting how it governs the trade-off between effective signal learning and data noise memorization.
\subsection{Feature Signal Learning}
Next, we first introduce conditions that guarantee the model's ability to learn feature signals.
\begin{condition}[Conditions of Signal Learning]\label{con:signal_learning}
    Suppose that:
    \begin{itemize}
        \item Dimension $d$ is sufficiently large, specifically $d=\widetilde{\Omega}(m^{2 \vee[4 /(q-2)]} n^{4 \vee[(2 q-2) /(q-2)]})$.
        \item Training sample size $n$ and neural network width $m$ satisfy $n, m=\Omega(\operatorname{poly} \log (d))$.
        \item The learning rate $\eta \leq \widetilde{O}(\min \{\|\mathbf{v}\|_2^{-2}, \|\bm{\xi}\|_2^{-2}\})$.
        \item The standard deviation of Gaussian initialization $\sigma_0$ is appropriately chosen such that $\widetilde{O}((n \varepsilon)^{-\frac{1}{q}} \|\mathbf{v}\|_2^{-1}) \leq \sigma_0 \leq \widetilde{O}(\min \{\varepsilon^{-\frac{1}{q}}\|\bm{\xi}\|_2^{-1}, (n )^{-\frac{1}{2q}}\|\mathbf{v}\|_2^{-1}, (n \varepsilon)^{\frac{-q-1}{q}} \|\bm{\xi}\|_2^{-1})\})$.
    \end{itemize}
\end{condition}
The conditions on $d,m,n$ are set to ensure that the learning problem is in a sufficiently over-parameterized setting, similar to the assumptions adopted in \cite{cao2022benign,frei2022benign,chatterji2023deep,kou2023benign}. Additionally, the conditions on initialization $\sigma_0$ and step size $\eta$ are to guarantee that gradient descent can effectively minimize the training loss. In private deep learning, the privacy budget is typically moderately larger compared to private optimization theory \cite{abadi2016deep,tramer2020differentially,de2022unlocking,sun2023importance,bao2023dp,tang2024differentially,tang2024private}. Therefore, we assume that the privacy budget remains larger than $1/\sqrt{n}$ here.

\begin{theorem}\label{lem:signal_first_stage}
    Under the same conditions as signal learning, if  $ \min \{\operatorname{SNR }\cdot n\varepsilon , \operatorname{SNR }^q \cdot n \}  \geq \widetilde{\Omega}(1)$, with at least probability $1-1/d$, there exists $T_1=O(\frac{\log (1  / \sigma_0\|\mathbf{v}\|_2) 4^{q-1} m}{ \eta q \sigma_0^{q-2}\|\mathbf{v}\|_2^q})$ such that
    \begin{itemize}
    \vspace{-0.05in}
        \item $\max _r \Gamma_{j, r}^{(T_1)} \geq 2$ for $j \in\{ \pm 1\}$.
        \item $|\Phi_{j, r, i}^{(t)}|=O(\sigma_0\sigma_{\xi}\sqrt{d}) $ for all $j \in\{ \pm 1\}, r \in[m], i \in[n]$ and $0 \leq t \leq T_1$.
    \end{itemize}
\end{theorem}

Here, we present one of our formal results. Based on Prop. \ref{pro:coeff_main} and Theorem \ref{lem:signal_first_stage}, at the end of the training stage $T_1$, when $\min \{\operatorname{SNR} \cdot n \varepsilon, \operatorname{SNR}^q \cdot n\} \geq \widetilde{\Omega}(1)$, the maximum signal learning, $\max _r \Gamma_{j, r}^{(T_1)}$, achieves $\widetilde{\Theta}(1)$. 
Additionally, as the initialization scale $\sigma_0$ satisfies  $\sigma_0 \leq \widetilde{O}((n \varepsilon)^{-1-1/q} \|\bm{\xi}\|_2^{-1})$ in Condition \ref{con:signal_learning},
this indicates the memorization of data noise, $|\Phi_{j, r, i}^{(t)}|$ remains bounded by $\widetilde{O}((n \varepsilon)^{-1-1/q})$ and it is smaller than the feature signal when $\varepsilon\geq 1/\sqrt{n}$. Moreover, compared to non-private learning, the condition $\operatorname{SNR} \cdot n \varepsilon \geq \widetilde{\Omega}(1)$  introduces more challenges. Even when the feature learning conditions ($\operatorname{SNR}^q \cdot n \geq \widetilde{\Omega}(1)$) for standard non-private learning are satisfied, private learning may still fail to capture the feature signal if $\operatorname{SNR} \cdot n \varepsilon \leq \widetilde{\Omega}(1) \leq \operatorname{SNR}^q \cdot n$. It demonstrates stronger feature signals are required in private learning compared to non-private learning, which aligns the empirical principles in previous work \cite{tramer2020differentially,sun2023importance,bao2023dp,tang2024differentially}.

Moreover, we can further show that if some stronger assumptions are satisfied, the following corollary ensures that private learning can achieve training loss comparable to those of non-private learning.

\begin{corollary}\label{thm:main_signal_loss}
    Let $T, T_1$ be defined as above. Then, under the same conditions as signal learning, for any $t \in[T_1, T]$, it holds that $|\Phi_{j, r, i}^{(t)}| \leq \sigma_0 \sigma_{\xi} \sqrt{d}$ for all $j \in\{ \pm 1\}$, $r \in[m]$ and $i \in[n]$ if $(n\varepsilon)^{q+1} \geq m$. Moreover, let $\mathbf{W}^*$ denote the collection of CNN parameters with convolution filters defined as $\mathbf{w}_{j, r}^*=\mathbf{w}_{j, r}^{(0)}+2 q m \log (2 q / \kappa) \cdot j \cdot\|\mathbf{v}\|_2^{-2} \cdot \mathbf{v}$ for a constant $\kappa>0$. 
    Then, with at least probability $1-1/d$, the following bound holds
    {\small
    \begin{equation}\label{eq:main_signal_loss_em}
          \sum_{s=T_1}^t L_D(\mathbf{W}^{(s)}) \leq \underbrace{\frac{\|\mathbf{W}^{(T_1)}-\mathbf{W}^*\|_F^2}{(2 q-1) \eta}+\frac{(t-T_1+1)\kappa}{(2 q-1)}}_{\text{Non-private terms}}    + \underbrace{(t-T_1+1) \cdot \frac{\eta d \sigma_z^2 + \widetilde{O}(\sigma_z m^{3/2} \|\mathbf{v}\|_2^{-1})}{(2 q-1)} }_{\text{Private terms}}
    \end{equation}}
    for some $t \in[T_1, T]$, where we denote $\|\mathbf{W}\|_F=\sqrt{\|\mathbf{W}_{+1}\|_F^2+\|\mathbf{W}_{-1}\|_F^2}$.
\end{corollary}

\cref{thm:main_signal_loss} characterizes the empirical risk of private learning under signal learning conditions, which can be decomposed into two terms. For standard non-private learning, by setting $T = T_1+\lfloor\frac{\|\mathbf{W}^{(T_1)}-\mathbf{W}^*\|_F^2}{2 \eta \kappa}\rfloor$ and dividing $t-T_1+1$ into both sides, the non-private term in the empirical loss can be bounded by $\frac{3\kappa}{2q-1}$, allowing the empirical loss to converge to $\kappa$. However, private learning introduces two key differences: 1) There are stricter limitations on the total training time, which may prevent $T$ from being as large as necessary for stable training. 2) In addition to the non-private term, the empirical risk involves a private term that appears unbounded, posing additional challenges to achieving convergence. Nonetheless, if we can assume the data satisfies:
$T=T_p^*=O(\frac{mn\varepsilon \sigma_0}{\eta \mu (\|\mathbf{v}\|_2 + \|\bm{\xi}\|_2)}) =  T^* \geq \kappa^{-1},$
where we denote $\mu = \max \{1, \| \mathbf{v}\|_2, \|\bm{\xi}\|_2\}$. Recalling the scale of private noise, we can obtain that $\sigma_z = {\sigma_0}/{\eta \mu\sqrt{T} }$. Then, it can be verified that the empirical risk in \cref{eq:main_signal_loss_em} will be upper bounded by $O(\kappa)$, provided there exists a step size $\eta$ satisfying:
$
\eta \geq \max \left\{\frac{2d \sigma_0^2}{\mu^2 T \kappa}, \frac{2 m^{2/3} \|\mathbf{v}\|_2^{-1} \sigma_0}{\mu \sqrt{T} \kappa} \right\}.
$

By combining the above results, we derive the following corollary, which states that the private CNN can achieve a test loss related to privacy budget $\varepsilon$ under \cref{con:signal_learning}.



\begin{corollary}\label{coro:popu_sign}
    Under the same conditions as above, suppose $\operatorname{SNR}\cdot n\varepsilon\geq \widetilde{\Omega}(1)$ and $(n\varepsilon)^{1/q+1}\geq \widetilde{\Omega}(1)$. Then, with at least probability $1-1/d$ and with $L_D(\mathbf{W}^{(t)}) \leq O(\kappa)$ for some $t \leq T$, the test error satisfies $L_{\mathcal{D}}(\mathbf{W}^{(t)}) \leq O( \kappa+\exp ({(n\varepsilon)}^{-1-1/q}))$.
\end{corollary}
Here, the test error is defined as $L_{\mathcal{D}}(\mathbf{W}):=\mathbb{P}_{(\mathbf{x}, y) \sim \mathcal{D}}[y \cdot (f(\mathbf{W}, \mathbf{x})) < 0].$
\subsection{Data Noise Memorization}
Next, we explore the scenario where the model primarily learns label-independent noise rather than the feature signal.
\begin{condition}[Conditions of Data Noise Memorization]\label{con:noise_memo}
    Suppose that the initial three conditions of data noise memorization are the same as \cref{con:signal_learning}. Additionally, we have 
    \begin{itemize}
        \item The standard deviation of $\sigma_0$ is appropriately chosen such that $\widetilde{O}( \max\{\varepsilon^{-1/q} \|\bm{\xi}\|_2^{-1}, (n/\sqrt{d})\|\mathbf{v}\|_2^{-1} \}) \leq \sigma_0 \leq \widetilde{O}(\min \{(n\varepsilon)^{-1/q}\|\mathbf{v}\|_2^{-1}, \|\bm{\xi}\|_2^{-1}\})$. 
    \end{itemize}
\end{condition}
Similar to feature signal learning, we also establish conditions for data noise memorization, but with a difference in the initialization setup and requirement on $\varepsilon \geq 1$.
\begin{theorem}\label{lem:noise_first_stage}
    Under the same conditions as data noise memorization 
    , if  $ \min \{\operatorname{SNR }^{-1}\cdot \varepsilon , \operatorname{SNR }^{-q} \cdot n^{-1} \}  \geq \widetilde{\Omega}(1)$, with at least probability $1-1/d$, there exists $T_1=O(\frac{ \log (1 /(\sigma_0 \sigma_{\xi} \sqrt{d})) m n}{0.15^{q-2} \eta q \sigma_0^{q-2}(\sigma_{\xi}^2 \sqrt{d})^q})$ such that
    \begin{itemize}
        \item $\max _{j, r} \bar{\Phi}_{j, r, i}^{(T_1)} \geq 2$ for all $i \in[n]$.
        \item $\max _{j, r} \Gamma_{j, r}^{(t)}=\widetilde{O}(\sigma_0\|\mathbf{v}\|_2)$ for all $0 \leq t \leq T_1$.
        \item $\max _{j, r, i}|\underline{\Phi}_{j, r, i}^{(t)}|=\widetilde{O}(\sigma_0 \sigma_{\xi} \sqrt{d})$ for all $0 \leq t \leq T_1$.
\end{itemize}
\end{theorem}
\cref{lem:noise_first_stage} shows that the data noise memorization, $\max _{j, r} \bar{\Phi}_{j, r, i}^{(T_1)}$, exceeds the feature signal learning at the end of the first training stage $T_1$, provided that $ \min \{\operatorname{SNR }^{-1}\cdot \varepsilon, \operatorname{SNR }^{-q} \cdot n^{-1} \}  \geq \widetilde{\Omega}(1)$. In contrast to \cref{lem:signal_first_stage}, here demonstrates that
when $\operatorname{SNR }^{-1}\cdot \varepsilon \leq    \widetilde{\Omega}(1) \leq \operatorname{SNR }^{-q} \cdot n^{-1} $, the memorization of data noise does not occur in private learning, even though it may happen in standard non-private learning. However, it is important to note that this scenario only arises under the highly restrictive condition $\varepsilon \leq  n^{-q}$, which is an overly stringent condition and unlikely to be met in practical private deep learning scenarios. In other words, when data noise memorization occurs in standard non-private learning, it will also occur in private learning as long as $\varepsilon\geq \operatorname{SNR}^{1-q}n^{-1}$. 
Moreover, it is noticed that $\max _{j, r, i}|\underline{\Phi}_{j, r, i}^{(t)}|$ is bounded by $\widetilde{O}(\sigma_0 \sigma_{\xi} \sqrt{d})$, while under stricter privacy budget (smaller $\varepsilon$), this term is much larger than the non-private learning, indicating that private learning may amplifying the data noise memorization of the other filter.


Since we assume the model is over-parameterized, even if it fails to learn a good feature signal, it can still fit the data noise well enough for the training loss to converge to a small value, similar to the case of feature signal learning. 
\begin{corollary}\label{thm:main_noise_loss}
    Let $T, T_1$ be defined as above.
    Then under the same conditions as data noise memorization, for any $t \in[T_1, T]$, it holds that $|\Gamma_{j, r, i}^{(t)}| \leq \sigma_0 \|\mathbf{v}\|_2$ for all $j \in\{ \pm 1\}$ and $r \in[m]$ if $n^{1/q}\varepsilon \geq m$. Moreover, let $\mathbf{W}^*$ be the collection of CNN parameters with convolution filters $\mathbf{w}_{j, r}^*=\mathbf{w}_{j, r}^{(0)}+2 q m \log (2 q / \kappa))[\sum_{i=1}^n \mathbbm{1}(j=y_i) \cdot \frac{\bm{\xi}_i}{\|\bm{\xi}_i\|_2}]$. Then, with at least probability $1-1/d$, the following bound holds
    {\small \begin{equation}\label{eq:noise_loss}
          \sum_{s=T_1}^t L_D(\mathbf{W}^{(s)}) \leq  \underbrace{\frac{\|\mathbf{W}^{(T_1)}-\mathbf{W}^*\|_F^2}{(2 q-1) \eta}+\frac{(t-T_1+1)\kappa}{(2 q-1)}}_{\text{Non-private terms}} +   \underbrace{(t-T_1+1) \cdot \frac{\eta d \sigma_z^2 + \widetilde{O}(\sigma_z m^2 n^{1 / 2} \|\bm{\xi}\|_2^{-1})}{(2 q-1)} }_{\text{Private terms}}
    \end{equation}}
    for some $t \in[T_1, T]$, where we denote $\|\mathbf{W}\|_F=\sqrt{\|\mathbf{W}_{+1}\|_F^2+\|\mathbf{W}_{-1}\|_F^2}$.
\end{corollary}

\cref{thm:main_noise_loss} shares the same empirical loss structure as \cref{thm:main_signal_loss}, with the only difference being the bound on $\|\mathbf{W}^{(T_1)}-\mathbf{W}^*\|_F$. This results in the term $\widetilde{O}(\sigma_z m^2 n^{1 / 2} \|\bm{\xi}\|_2^{-1})$ appearing here. Therefore, under the same data assumptions as \cref{thm:main_signal_loss}, the empirical risk in \cref{eq:noise_loss} can also be upper bounded by $O(\kappa)$, provided there exists a step size $\eta$ satisfying:
$
\eta \geq \max \left\{\frac{2d \sigma_0^2}{\mu^2 T \kappa}, \frac{2 m^{2}n^{1/2} \|\bm{\xi}\|_2^{-1}\sigma_0}{\mu \sqrt{T} \kappa} \right\}.
$
However, even if the private CNN model achieves a sufficiently small training loss under the data noise memorization scenario, it still fails to exhibit good generalization ability, as it primarily learns the label-independent data noise rather than the label-dependent feature signals.
\begin{corollary}\label{coro:popu_noise}
    Under the same conditions as data noise memorization, within $T$ iterations, regardless of how the sample size $n$ and privacy budget $\varepsilon$ chosen, with at least probability $1-1/d$, we can find $\mathbf{W}^{(\widetilde{T})}$ such that $L_D(\mathbf{W}^{(\widetilde{T})}) \leq O(\kappa)$. Additionally, for some $0 \leq t \leq \widetilde{T}$ we have that $L_{\mathcal{D}}(\mathbf{W}^{(t)}) \geq 0.1$.
\end{corollary}

\begin{figure*}[t]
\vspace{-0.15in}
    \centering
    \begin{subfigure}[b]{0.23\textwidth}
        \includegraphics[width=\textwidth]{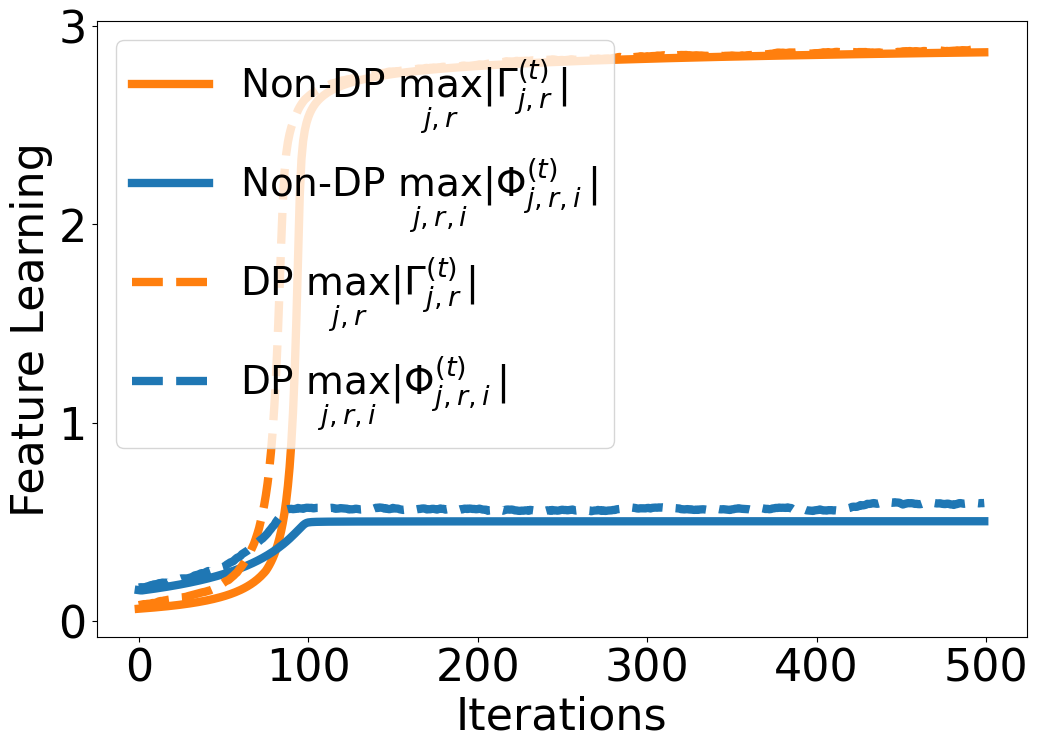}
        \caption{$\operatorname{SNR}=0.6,\ \varepsilon=5$}
        \label{fig:snr06_eps5}
    \end{subfigure}\hfill
    \begin{subfigure}[b]{0.23\textwidth}
        \includegraphics[width=\textwidth]{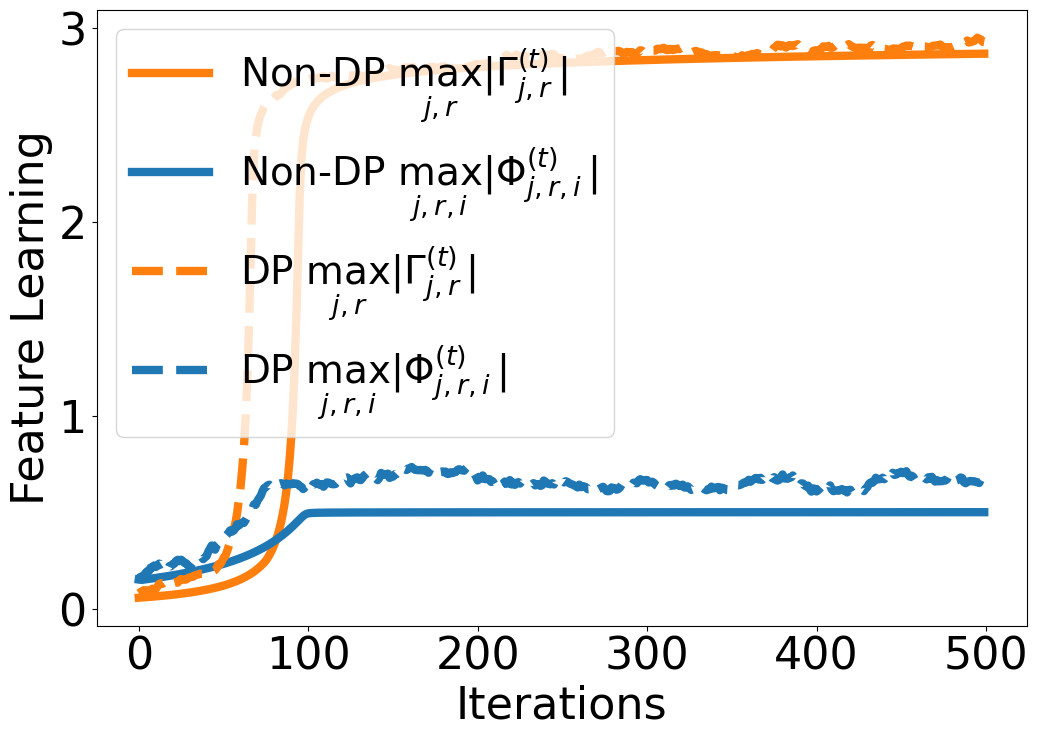}
        \caption{$\operatorname{SNR}=0.6,\ \varepsilon=1$}
        \label{fig:snr06_eps1}
    \end{subfigure}\hfill
    \begin{subfigure}[b]{0.23\textwidth}
        \includegraphics[width=\textwidth]{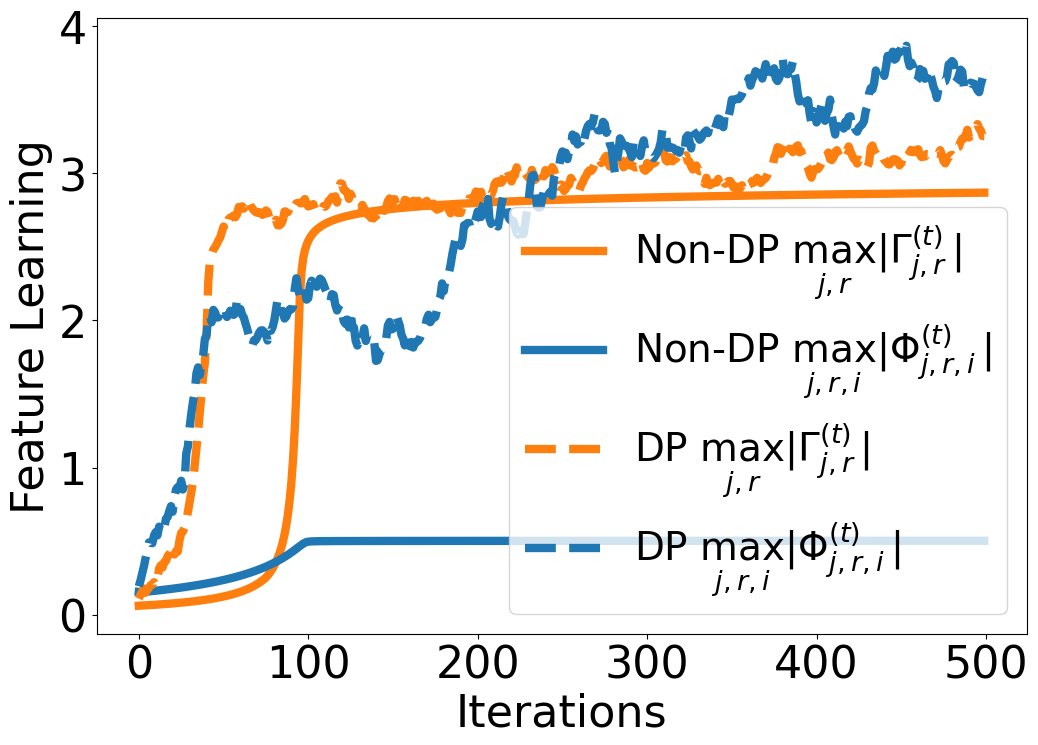}
        \caption{$\operatorname{SNR}=0.6,\ \varepsilon=0.2$}
        \label{fig:snr06_eps02}
    \end{subfigure}\hfill
    \begin{subfigure}[b]{0.23\textwidth}
        \includegraphics[width=\textwidth]{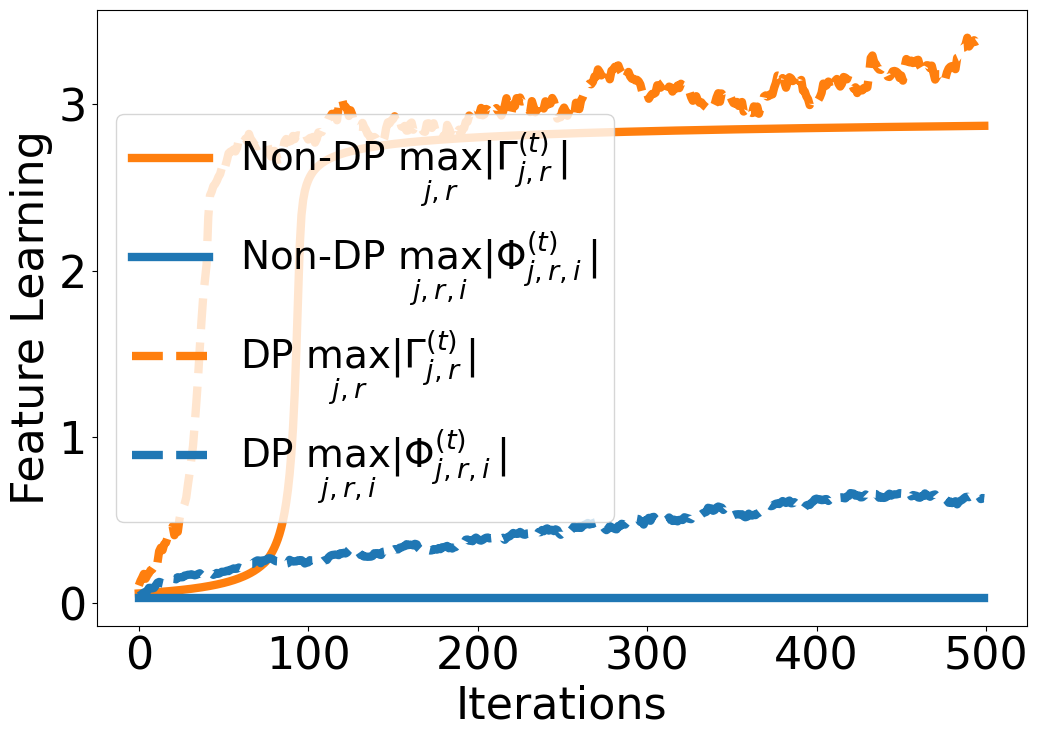}
        \caption{$\operatorname{SNR}=3,\ \varepsilon=0.2$}
        \label{fig:snr3_eps02}
    \end{subfigure}

    \vspace{0.6em}

    \begin{subfigure}[b]{0.23\textwidth}
        \includegraphics[width=\textwidth]{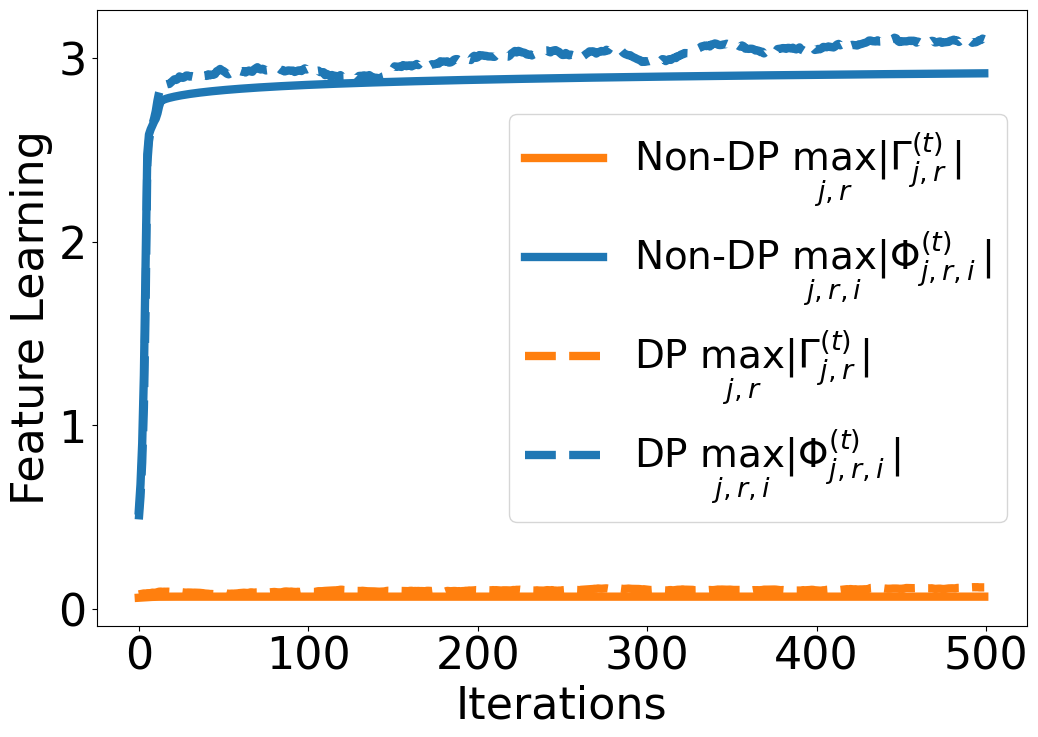}
        \caption{$\operatorname{SNR}=0.2,\ \varepsilon=5$}
        \label{fig:snr02_eps5}
    \end{subfigure}\hfill
    \begin{subfigure}[b]{0.23\textwidth}
        \includegraphics[width=\textwidth]{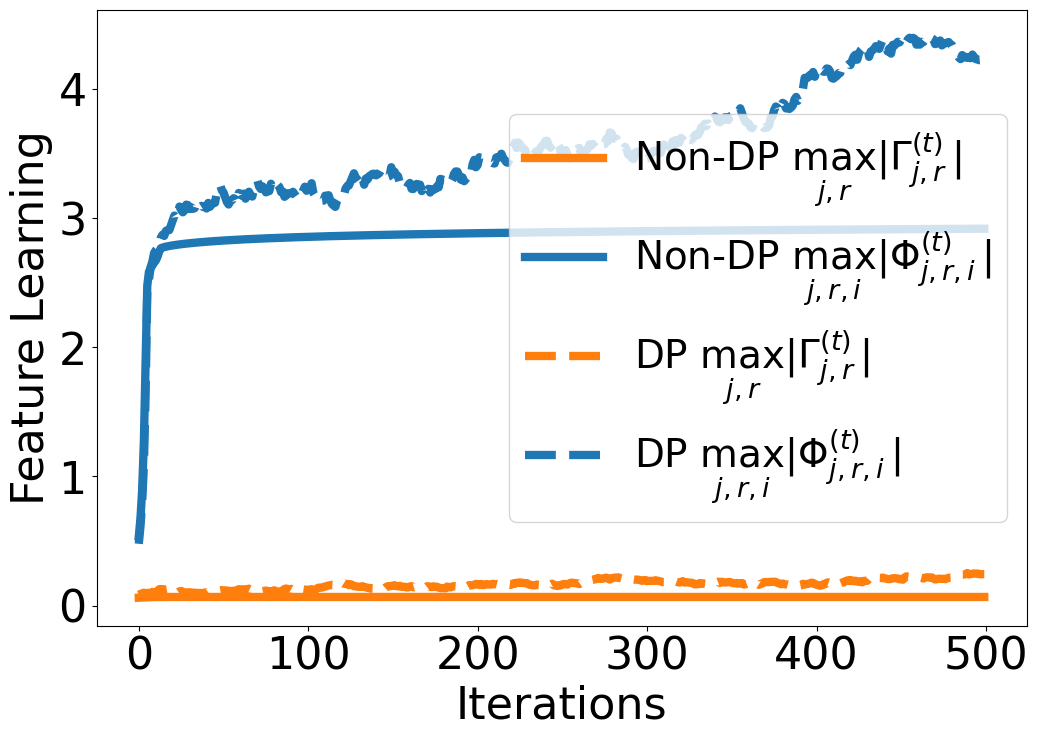}
        \caption{$\operatorname{SNR}=0.2,\ \varepsilon=1$}
        \label{fig:snr02_eps1}
    \end{subfigure}\hfill
    \begin{subfigure}[b]{0.23\textwidth}
        \includegraphics[width=\textwidth]{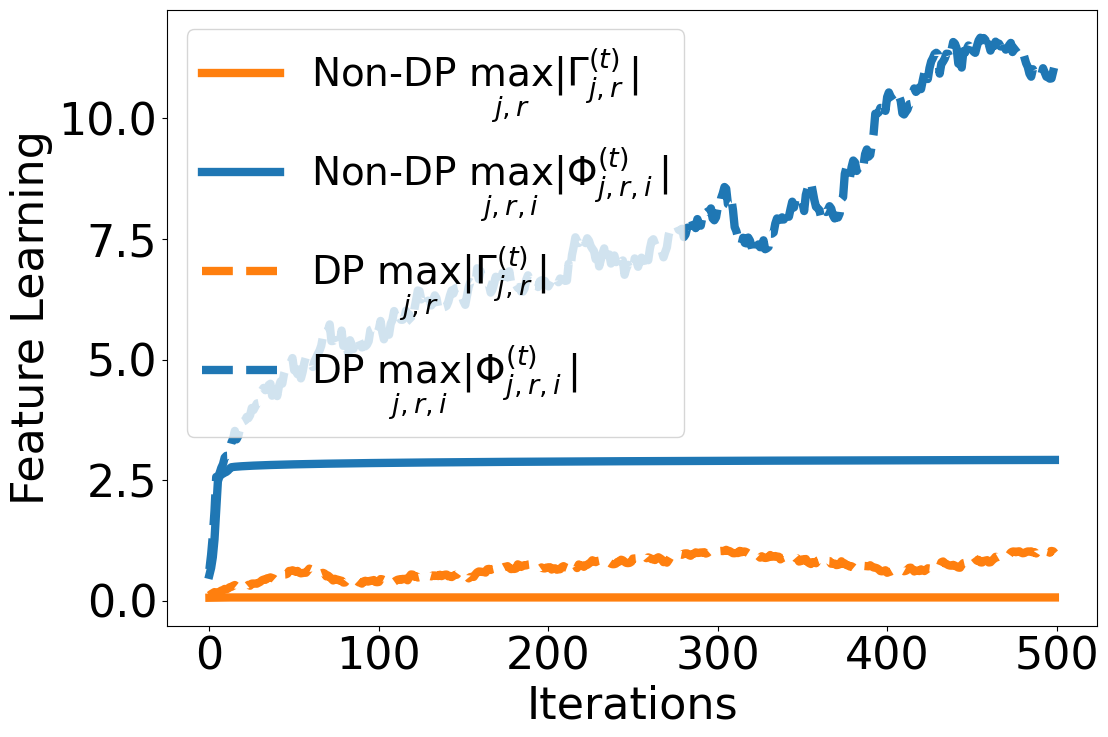}
        \caption{$\operatorname{SNR}=0.2,\ \varepsilon=0.2$}
        \label{fig:snr02_eps02}
    \end{subfigure}\hfill
    \begin{subfigure}[b]{0.23\textwidth}
        \includegraphics[width=\textwidth]{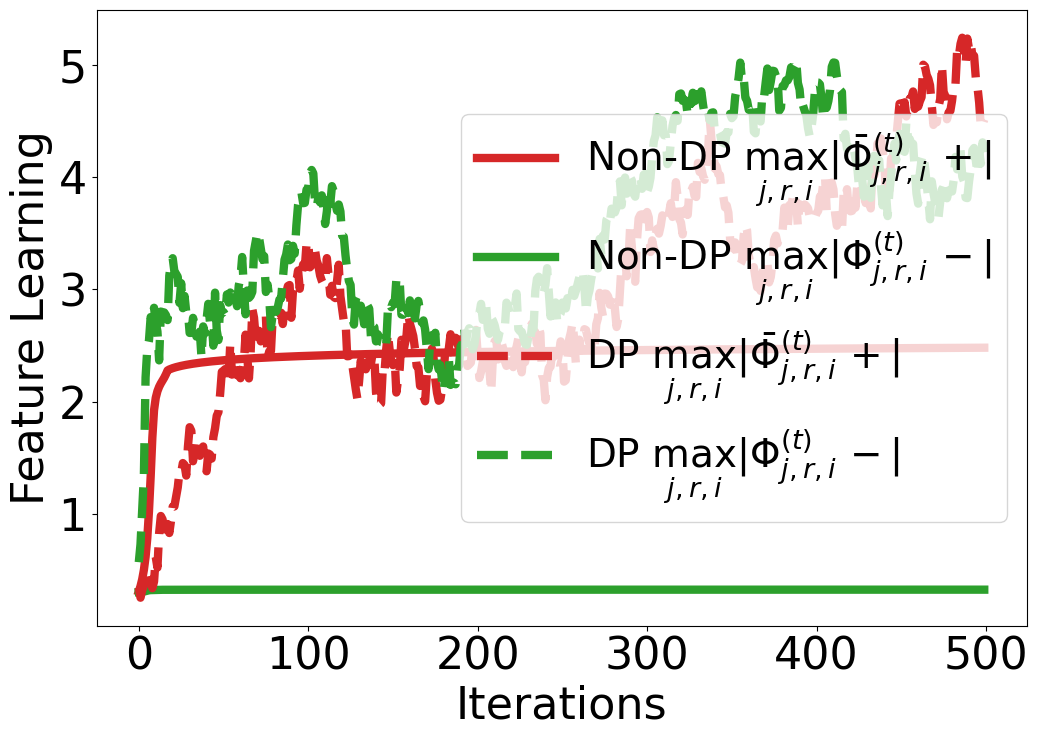}
        \caption{$\operatorname{SNR}=0.2,\ \varepsilon=0.2$ }
        \label{fig:snr02_eps02_plus}
    \end{subfigure}

    \caption{Comparison of Feature Signal and Data Noise in (Non) Private Learning. The figure compares the dynamics of feature learning $\max_{j,r}\lvert\Gamma^{(t)}_{j,r}\rvert$ and data noise $\max_{j,r,i}\lvert\Phi^{(t)}_{j,r,i}\rvert$ across varying privacy budgets $\varepsilon$ and SNR. Subfigures \ref{fig:snr06_eps5}–\ref{fig:snr06_eps02} correspond to $\operatorname{SNR}=0.6$ with $\varepsilon\in\{5,1,0.2\}$, subfigure \ref{fig:snr3_eps02} corresponds to $\operatorname{SNR}=3$ with $\varepsilon=0.2$, and subfigures \ref{fig:snr02_eps5}–\ref{fig:snr02_eps02} correspond to $\operatorname{SNR}=0.2$ with $\varepsilon\in\{5,1,0.2\}$. Subfigure \ref{fig:snr02_eps02_plus} shows a decomposition of $\max_{j,r,i}\lvert\Phi^{(t)}_{j,r,i}\rvert$ at $\operatorname{SNR}=0.2,\ \varepsilon=0.2$.}
    \label{fig:snr_all}
    \vspace{-0.2in}
\end{figure*}

\section{Experiment}
\vspace{-0.1in}
\label{sec:experiment}
Due to space limitations, we present only a synthetic data experiment here. Additional results on the CIFAR-10 dataset \cite{Krizhevsky2009}, including an exploration of SNR, are provided in the Appendix. 
In the synthetic experiment, we compare the evolution of feature signal and data noise between NoisyGD and standard (non-private) training.

\textbf{Experimental Setup.}
We conducted experiments by generating synthetic data defined in \cref{def:data_dis} with a controlled signal-to-noise ratio ($\operatorname{SNR} = 0.2, 0.6, 3$). The dataset was constructed using a fixed signal vector $\mathbf{v}$ and a data noise component $\bm{\xi}$. We implemented a two-layer CNN with an input dimension of $d=1000, m=10$ filters, and a polynomial ReLU activation function with a power parameter $q=3$. The model weights were initialized randomly from a normal distribution with a small variance ($\sigma_0=0.001$), and training was performed using gradient descent with a learning rate of $\eta =0.01$ over $500$ epochs. For private learning, we consider the noisy gradient descent during weight updates, with a privacy budget of $\epsilon=0.2,1,5$ and $\delta=10^{-5}$. Throughout the training process, we monitored the maximum inner products between the learned weights and the signal and data noise components, denoted as $\max _{j, r}|\Gamma_{j, r}^{(t)}|$ and $\max _{j, r, i}|\Phi_{j, r, i}^{(t)}|$, respectively. 

\textbf{Private learning requires stronger feature signal.} In \cref{fig:snr_all}, we compare the dynamics of feature signal learning and data noise memorization during the non-private and private training process under the higher SNR. In subfigures \ref{fig:snr06_eps5} and \ref{fig:snr06_eps1}, it can be observed that when the $\operatorname{SNR} = 0.6$ is sufficient for non-private training and the privacy budgets $\varepsilon =1, 5$ are moderately larger, the feature learning trajectories in private training closely align with those of non-private training. However, in subfigures \ref{fig:snr06_eps02}, when the $\operatorname{SNR} = 0.6$ is sufficient for non-private training and $\varepsilon =0.2$ is relatively smaller, the data noise memorization (represented by the blue dashed line) no longer remains below the feature signal like the non-private training. This indicates the $\operatorname{SNR} = 0.6$ is insufficient for private training under stronger privacy constraints. When we increase $\operatorname{SNR} $ to $3$ in \cref{fig:snr3_eps02}, we observe that private training regains its ability to perform feature learning, exhibiting trends similar to non-private training. This demonstrates that as the privacy budget $\varepsilon$ decreases (i.e., stronger privacy guarantees), the requirement for a higher SNR becomes more pronounced.
 

\textbf{Private learning may amplify data noise memorization.} 
In \cref{fig:snr02_eps5}-\cref{fig:snr02_eps02}, we observe that if non-private training exhibits data noise memorization {(the blue line is higher than the orange line)}, private training is also prone to this behavior {(corresponding to the dashed line)}. Moreover, as illustrated in Figure \ref{fig:snr02_eps02_plus}, the green DP line (representing noise memorization under private learning) shows a noticeable increase, while the corresponding non-DP green line remains unchanged. This indicates that private training, particularly with smaller privacy budgets, not only fails to suppress data noise but also amplifies its memorization for other filters, aligning with our theoretical analysis in \cref{lem:noise_first_stage}.

\section{Conclusion}
In this paper, we introduced the first theoretical framework to analyze the dynamics of feature learning in differentially private learning, focusing on the trade-offs between feature signals and data noise through a decomposition of these components. Using a two-layer CNN, we demonstrated that private learning necessitates a higher signal-to-noise ratio (SNR) compared to non-private training to effectively capture features, particularly under stringent privacy budgets. Additionally, we showed that data noise memorization, if present in non-private learning, persists in private learning, resulting in poor generalization even when training losses are minimized. Our findings highlight the critical role of feature enhancement in private learning, aligning with prior empirical studies and providing valuable insights for designing effective privacy-preserving learning systems.

\bibliography{iclr2026_conference}

@article{cao2022benign,
  title={Benign overfitting in two-layer convolutional neural networks},
  author={Cao, Yuan and Chen, Zixiang and Belkin, Misha and Gu, Quanquan},
  journal={Advances in neural information processing systems},
  volume={35},
  pages={25237--25250},
  year={2022}
}

@inproceedings{kou2023benign,
  title={Benign overfitting in two-layer ReLU convolutional neural networks},
  author={Kou, Yiwen and Chen, Zixiang and Chen, Yuanzhou and Gu, Quanquan},
  booktitle={International Conference on Machine Learning},
  pages={17615--17659},
  year={2023},
  organization={PMLR}
}

@inproceedings{jelassi2022towards,
  title={Towards understanding how momentum improves generalization in deep learning},
  author={Jelassi, Samy and Li, Yuanzhi},
  booktitle={International Conference on Machine Learning},
  pages={9965--10040},
  year={2022},
  organization={PMLR}
}

@inproceedings{zou2023benefits,
  title={The benefits of mixup for feature learning},
  author={Zou, Difan and Cao, Yuan and Li, Yuanzhi and Gu, Quanquan},
  booktitle={International Conference on Machine Learning},
  pages={43423--43479},
  year={2023},
  organization={PMLR}
}

@article{allen2020towards,
  title={Towards understanding ensemble, knowledge distillation and self-distillation in deep learning},
  author={Allen-Zhu, Zeyuan and Li, Yuanzhi},
  journal={arXiv preprint arXiv:2012.09816},
  year={2020}
}

@inproceedings{song2013stochastic,
  title={Stochastic gradient descent with differentially private updates},
  author={Song, Shuang and Chaudhuri, Kamalika and Sarwate, Anand D},
  booktitle={2013 IEEE global conference on signal and information processing},
  pages={245--248},
  year={2013},
  organization={IEEE}
}

@article{lundervold2019overview,
  title={An overview of deep learning in medical imaging focusing on MRI},
  author={Lundervold, Alexander Selvikv{\aa}g and Lundervold, Arvid},
  journal={Zeitschrift f{\"u}r Medizinische Physik},
  volume={29},
  number={2},
  pages={102--127},
  year={2019},
  publisher={Elsevier}
}

@article{chlap2021review,
  title={A review of medical image data augmentation techniques for deep learning applications},
  author={Chlap, Phillip and Min, Hang and Vandenberg, Nym and Dowling, Jason and Holloway, Lois and Haworth, Annette},
  journal={Journal of Medical Imaging and Radiation Oncology},
  volume={65},
  number={5},
  pages={545--563},
  year={2021},
  publisher={Wiley Online Library}
}

@article{shamshad2023transformers,
  title={Transformers in medical imaging: A survey},
  author={Shamshad, Fahad and Khan, Salman and Zamir, Syed Waqas and Khan, Muhammad Haris and Hayat, Munawar and Khan, Fahad Shahbaz and Fu, Huazhu},
  journal={Medical Image Analysis},
  volume={88},
  pages={102802},
  year={2023},
  publisher={Elsevier}
}

@article{ozbayoglu2020deep,
  title={Deep learning for financial applications: A survey},
  author={Ozbayoglu, Ahmet Murat and Gudelek, Mehmet Ugur and Sezer, Omer Berat},
  journal={Applied soft computing},
  volume={93},
  pages={106384},
  year={2020},
  publisher={Elsevier}
}

@article{oroojlooy2023review,
  title={A review of cooperative multi-agent deep reinforcement learning},
  author={Oroojlooy, Afshin and Hajinezhad, Davood},
  journal={Applied Intelligence},
  volume={53},
  number={11},
  pages={13677--13722},
  year={2023},
  publisher={Springer}
}

@article{bi2024advanced,
  title={Advanced portfolio management in finance using deep learning and artificial intelligence techniques: Enhancing investment strategies through machine learning models},
  author={Bi, Shuochen and Lian, Yufan},
  journal={Journal of Artificial Intelligence Research},
  volume={4},
  number={1},
  pages={233--298},
  year={2024}
}

@inproceedings{dwork2006calibrating,
  title={Calibrating noise to sensitivity in private data analysis},
  author={Dwork, Cynthia and McSherry, Frank and Nissim, Kobbi and Smith, Adam},
  booktitle={Theory of Cryptography: Third Theory of Cryptography Conference, TCC 2006, New York, NY, USA, March 4-7, 2006. Proceedings 3},
  pages={265--284},
  year={2006},
  organization={Springer}
}

@article{bagdasaryan2019differential,
  title={Differential privacy has disparate impact on model accuracy},
  author={Bagdasaryan, Eugene and Poursaeed, Omid and Shmatikov, Vitaly},
  journal={Advances in neural information processing systems},
  volume={32},
  year={2019}
}

@inproceedings{shokri2015privacy,
  title={Privacy-preserving deep learning},
  author={Shokri, Reza and Shmatikov, Vitaly},
  booktitle={Proceedings of the 22nd ACM SIGSAC conference on computer and communications security},
  pages={1310--1321},
  year={2015}
}

@article{tramer2020differentially,
  title={Differentially private learning needs better features (or much more data)},
  author={Tramer, Florian and Boneh, Dan},
  journal={arXiv preprint arXiv:2011.11660},
  year={2020}
}

@article{de2022unlocking,
  title={Unlocking high-accuracy differentially private image classification through scale},
  author={De, Soham and Berrada, Leonard and Hayes, Jamie and Smith, Samuel L and Balle, Borja},
  journal={arXiv preprint arXiv:2204.13650},
  year={2022}
}

@article{li2021large,
  title={Large language models can be strong differentially private learners},
  author={Li, Xuechen and Tramer, Florian and Liang, Percy and Hashimoto, Tatsunori},
  journal={arXiv preprint arXiv:2110.05679},
  year={2021}
}

@article{arora2022can,
  title={Can Foundation Models Help Us Achieve Perfect Secrecy?},
  author={Arora, Simran and R{\'e}, Christopher},
  journal={arXiv preprint arXiv:2205.13722},
  year={2022}
}

@article{mehta2023towards,
  title={Towards large scale transfer learning for differentially private image classification},
  author={Mehta, Harsh and Thakurta, Abhradeep Guha and Kurakin, Alexey and Cutkosky, Ashok},
  journal={Transactions on Machine Learning Research},
  year={2023}
}

@article{kurakin2022toward,
  title={Toward training at imagenet scale with differential privacy},
  author={Kurakin, Alexey and Song, Shuang and Chien, Steve and Geambasu, Roxana and Terzis, Andreas and Thakurta, Abhradeep},
  journal={arXiv preprint arXiv:2201.12328},
  year={2022}
}

@inproceedings{nasr2023effectively,
  title={Effectively using public data in privacy preserving machine learning},
  author={Nasr, Milad and Mahloujifar, Saeed and Tang, Xinyu and Mittal, Prateek and Houmansadr, Amir},
  booktitle={International Conference on Machine Learning},
  pages={25718--25732},
  year={2023},
  organization={PMLR}
}

@article{bu2024pre,
  title={Pre-training Differentially Private Models with Limited Public Data},
  author={Bu, Zhiqi and Zhang, Xinwei and Hong, Mingyi and Zha, Sheng and Karypis, George},
  journal={arXiv preprint arXiv:2402.18752},
  year={2024}
}

@article{tang2024private,
  title={Private fine-tuning of large language models with zeroth-order optimization},
  author={Tang, Xinyu and Panda, Ashwinee and Nasr, Milad and Mahloujifar, Saeed and Mittal, Prateek},
  journal={arXiv preprint arXiv:2401.04343},
  year={2024}
}

@inproceedings{bassily2014private,
  title={Private empirical risk minimization: Efficient algorithms and tight error bounds},
  author={Bassily, Raef and Smith, Adam and Thakurta, Abhradeep},
  booktitle={2014 IEEE 55th annual symposium on foundations of computer science},
  pages={464--473},
  year={2014},
  organization={IEEE}
}

@article{wang2017differentially,
  title={Differentially private empirical risk minimization revisited: Faster and more general},
  author={Wang, Di and Ye, Minwei and Xu, Jinhui},
  journal={Advances in Neural Information Processing Systems},
  volume={30},
  year={2017}
}

@inproceedings{feldman2020private,
  title={Private stochastic convex optimization: optimal rates in linear time},
  author={Feldman, Vitaly and Koren, Tomer and Talwar, Kunal},
  booktitle={Proceedings of the 52nd Annual ACM SIGACT Symposium on Theory of Computing},
  pages={439--449},
  year={2020}
}

@article{andrew2021differentially,
  title={Differentially private learning with adaptive clipping},
  author={Andrew, Galen and Thakkar, Om and McMahan, Brendan and Ramaswamy, Swaroop},
  journal={Advances in Neural Information Processing Systems},
  volume={34},
  pages={17455--17466},
  year={2021}
}

@inproceedings{asi2021private,
  title={Private adaptive gradient methods for convex optimization},
  author={Asi, Hilal and Duchi, John and Fallah, Alireza and Javidbakht, Omid and Talwar, Kunal},
  booktitle={International Conference on Machine Learning},
  pages={383--392},
  year={2021},
  organization={PMLR}
}

@inproceedings{abadi2016deep,
  title={Deep learning with differential privacy},
  author={Abadi, Martin and Chu, Andy and Goodfellow, Ian and McMahan, H Brendan and Mironov, Ilya and Talwar, Kunal and Zhang, Li},
  booktitle={Proceedings of the 2016 ACM SIGSAC conference on computer and communications security},
  pages={308--318},
  year={2016}
}

@article{mcmahan2017learning,
  title={Learning differentially private recurrent language models},
  author={McMahan, H Brendan and Ramage, Daniel and Talwar, Kunal and Zhang, Li},
  journal={arXiv preprint arXiv:1710.06963},
  year={2017}
}

@inproceedings{kairouz2021practical,
  title={Practical and private (deep) learning without sampling or shuffling},
  author={Kairouz, Peter and McMahan, Brendan and Song, Shuang and Thakkar, Om and Thakurta, Abhradeep and Xu, Zheng},
  booktitle={International Conference on Machine Learning},
  pages={5213--5225},
  year={2021},
  organization={PMLR}
}

@inproceedings{wang2019differentially,
  title={Differentially private empirical risk minimization with smooth non-convex loss functions: A non-stationary view},
  author={Wang, Di and Xu, Jinhui},
  booktitle={Proceedings of the AAAI Conference on Artificial Intelligence},
  volume={33},
  number={01},
  pages={1182--1189},
  year={2019}
}

@article{bassily2019private,
  title={Private stochastic convex optimization with optimal rates},
  author={Bassily, Raef and Feldman, Vitaly and Talwar, Kunal and Guha Thakurta, Abhradeep},
  journal={Advances in neural information processing systems},
  volume={32},
  year={2019}
}

@article{song2020characterizing,
  title={Characterizing private clipped gradient descent on convex generalized linear problems},
  author={Song, Shuang and Thakkar, Om and Thakurta, Abhradeep},
  journal={arXiv preprint arXiv:2006.06783},
  year={2020}
}

@article{su2021faster,
  title={Faster rates of differentially private stochastic convex optimization},
  author={Su, Jinyan and Wang, Di},
  journal={arXiv preprint arXiv},
  volume={2108},
  year={2021}
}

@inproceedings{bassily2021non,
  title={Non-euclidean differentially private stochastic convex optimization},
  author={Bassily, Raef and Guzm{\'a}n, Crist{\'o}bal and Nandi, Anupama},
  booktitle={Conference on Learning Theory},
  pages={474--499},
  year={2021},
  organization={PMLR}
}

@article{kulkarni2021private,
  title={Private non-smooth empirical risk minimization and stochastic convex optimization in subquadratic steps},
  author={Kulkarni, Janardhan and Lee, Yin Tat and Liu, Daogao},
  journal={arXiv preprint arXiv:2103.15352},
  year={2021}
}

@article{zhang2017efficient,
  title={Efficient private ERM for smooth objectives},
  author={Zhang, Jiaqi and Zheng, Kai and Mou, Wenlong and Wang, Liwei},
  journal={arXiv preprint arXiv:1703.09947},
  year={2017}
}

@inproceedings{zhang2021private,
  title={Private stochastic non-convex optimization with improved utility rates},
  author={Zhang, Qiuchen and Ma, Jing and Lou, Jian and Xiong, Li},
  booktitle={Proceedings of the Thirtieth International Joint Conference on Artificial Intelligence},
  year={2021}
}

@inproceedings{wang2023efficient,
  title={Efficient privacy-preserving stochastic nonconvex optimization},
  author={Wang, Lingxiao and Jayaraman, Bargav and Evans, David and Gu, Quanquan},
  booktitle={Uncertainty in Artificial Intelligence},
  pages={2203--2213},
  year={2023},
  organization={PMLR}
}

@article{bassily2021differentially,
  title={Differentially private stochastic optimization: New results in convex and non-convex settings},
  author={Bassily, Raef and Guzm{\'a}n, Crist{\'o}bal and Menart, Michael},
  journal={Advances in Neural Information Processing Systems},
  volume={34},
  pages={9317--9329},
  year={2021}
}

@inproceedings{dwork2010differential,
  title={Differential privacy under continual observation},
  author={Dwork, Cynthia and Naor, Moni and Pitassi, Toniann and Rothblum, Guy N},
  booktitle={Proceedings of the forty-second ACM symposium on Theory of computing},
  pages={715--724},
  year={2010}
}

@inproceedings{hu2022high,
  title={High dimensional differentially private stochastic optimization with heavy-tailed data},
  author={Hu, Lijie and Ni, Shuo and Xiao, Hanshen and Wang, Di},
  booktitle={Proceedings of the 41st ACM SIGMOD-SIGACT-SIGAI Symposium on Principles of Database Systems},
  pages={227--236},
  year={2022}
}

@inproceedings{su2023differentially,
  title={Differentially private stochastic convex optimization in (non)-Euclidean space revisited},
  author={Su, Jinyan and Zhao, Changhong and Wang, Di},
  booktitle={Uncertainty in Artificial Intelligence},
  pages={2026--2035},
  year={2023},
  organization={PMLR}
}

@article{su2024faster,
  title={Faster Rates of Differentially Private Stochastic Convex Optimization},
  author={Su, Jinyan and Hu, Lijie and Wang, Di},
  journal={Journal of Machine Learning Research},
  volume={25},
  number={114},
  pages={1--41},
  year={2024}
}

@inproceedings{DBLP:conf/ijcai/TaoW0W22,
  author       = {Youming Tao and
                  Yulian Wu and
                  Xiuzhen Cheng and
                  Di Wang},
  title        = {Private Stochastic Convex Optimization and Sparse Learning with Heavy-tailed
                  Data Revisited},
  booktitle    = {{IJCAI}},
  pages        = {3947--3953},
  publisher    = {ijcai.org},
  year         = {2022}
}

@article{dingrevisiting,
  title={Revisiting differentially private relu regression},
  author={Ding, Meng and Lei, Mingxi and Zhu, Liyang and Wang, Shaowei and Wang, Di and Xu, Jinhui},
  journal={Advances in Neural Information Processing Systems},
  volume={37},
  pages={55470--55506},
  year={2024}
}

@article{sun2023importance,
  title={The importance of feature preprocessing for differentially private linear optimization},
  author={Sun, Ziteng and Suresh, Ananda Theertha and Menon, Aditya Krishna},
  journal={arXiv preprint arXiv:2307.11106},
  year={2023}
}

@article{tang2024differentially,
  title={Differentially private image classification by learning priors from random processes},
  author={Tang, Xinyu and Panda, Ashwinee and Sehwag, Vikash and Mittal, Prateek},
  journal={Advances in Neural Information Processing Systems},
  volume={36},
  year={2024}
}

@article{wang2024neural,
  title={Neural Collapse Meets Differential Privacy: Curious Behaviors of NoisyGD with Near-perfect Representation Learning},
  author={Wang, Chendi and Zhu, Yuqing and Su, Weijie J and Wang, Yu-Xiang},
  journal={arXiv preprint arXiv:2405.08920},
  year={2024}
}

@article{papyan2020prevalence,
  title={Prevalence of neural collapse during the terminal phase of deep learning training},
  author={Papyan, Vardan and Han, XY and Donoho, David L},
  journal={Proceedings of the National Academy of Sciences},
  volume={117},
  number={40},
  pages={24652--24663},
  year={2020},
  publisher={National Acad Sciences}
}

@article{bao2023dp,
  title={Dp-mix: mixup-based data augmentation for differentially private learning},
  author={Bao, Wenxuan and Pittaluga, Francesco and BG, Vijay Kumar and Bindschaedler, Vincent},
  journal={Advances in Neural Information Processing Systems},
  volume={36},
  pages={12154--12170},
  year={2023}
}

@article{chatterji2023deep,
  title={Deep linear networks can benignly overfit when shallow ones do},
  author={Chatterji, Niladri S and Long, Philip M},
  journal={Journal of Machine Learning Research},
  volume={24},
  number={117},
  pages={1--39},
  year={2023}
}

@inproceedings{frei2022benign,
  title={Benign overfitting without linearity: Neural network classifiers trained by gradient descent for noisy linear data},
  author={Frei, Spencer and Chatterji, Niladri S and Bartlett, Peter},
  booktitle={Conference on Learning Theory},
  pages={2668--2703},
  year={2022},
  organization={PMLR}
}

@inproceedings{papernot2021tempered,
  title={Tempered sigmoid activations for deep learning with differential privacy},
  author={Papernot, Nicolas and Thakurta, Abhradeep and Song, Shuang and Chien, Steve and Erlingsson, {\'U}lfar},
  booktitle={Proceedings of the AAAI Conference on Artificial Intelligence},
  volume={35},
  number={10},
  pages={9312--9321},
  year={2021}
}

@article{bu2024automatic,
  title={Automatic clipping: Differentially private deep learning made easier and stronger},
  author={Bu, Zhiqi and Wang, Yu-Xiang and Zha, Sheng and Karypis, George},
  journal={Advances in Neural Information Processing Systems},
  volume={36},
  year={2024}
}

@article{yu2021differentially,
  title={Differentially private fine-tuning of language models},
  author={Yu, Da and Naik, Saurabh and Backurs, Arturs and Gopi, Sivakanth and Inan, Huseyin A and Kamath, Gautam and Kulkarni, Janardhan and Lee, Yin Tat and Manoel, Andre and Wutschitz, Lukas and others},
  journal={arXiv preprint arXiv:2110.06500},
  year={2021}
}

@inproceedings{luo2021scalable,
  title={Scalable differential privacy with sparse network finetuning},
  author={Luo, Zelun and Wu, Daniel J and Adeli, Ehsan and Fei-Fei, Li},
  booktitle={Proceedings of the IEEE/CVF conference on computer vision and pattern recognition},
  pages={5059--5068},
  year={2021}
}

@article{zhang2024disk,
  title={Disk: Differentially private optimizer with simplified kalman filter for noise reduction},
  author={Zhang, Xinwei and Bu, Zhiqi and Balle, Borja and Hong, Mingyi and Razaviyayn, Meisam and Mirrokni, Vahab},
  journal={arXiv preprint arXiv:2410.03883},
  year={2024}
}

@article{chaudhuri2011differentially,
  title={Differentially private empirical risk minimization.},
  author={Chaudhuri, Kamalika and Monteleoni, Claire and Sarwate, Anand D},
  journal={Journal of Machine Learning Research},
  volume={12},
  number={3},
  year={2011}
}

@inproceedings{wang2020differentially,
  title={On differentially private stochastic convex optimization with heavy-tailed data},
  author={Wang, Di and Xiao, Hanshen and Devadas, Srinivas and Xu, Jinhui},
  booktitle={International Conference on Machine Learning},
  pages={10081--10091},
  year={2020},
  organization={PMLR}
}

@inproceedings{kamath2022improved,
  title={Improved rates for differentially private stochastic convex optimization with heavy-tailed data},
  author={Kamath, Gautam and Liu, Xingtu and Zhang, Huanyu},
  booktitle={International Conference on Machine Learning},
  pages={10633--10660},
  year={2022},
  organization={PMLR}
}

@techreport{Krizhevsky2009,
  author = {Krizhevsky, Alex},
  title = {Learning Multiple Layers of Features from Tiny Images},
  institution = {University of Toronto},
  year = {2009},
  number = {TR-2009},
  url = {https://www.cs.toronto.edu/~kriz/learning-features-2009-TR.pdf}
}

@article{hendrycks2019robustness,
  title={Benchmarking Neural Network Robustness to Common Corruptions and Perturbations},
  author={Dan Hendrycks and Thomas Dietterich},
  journal={Proceedings of the International Conference on Learning Representations},
  year={2019}
}

@inproceedings{he2016deep,
  title={Deep residual learning for image recognition},
  author={He, Kaiming and Zhang, Xiangyu and Ren, Shaoqing and Sun, Jian},
  booktitle={Proceedings of the IEEE conference on computer vision and pattern recognition},
  pages={770--778},
  year={2016}
}

@article{selvaraju2020grad,
  title={Grad-CAM: visual explanations from deep networks via gradient-based localization},
  author={Selvaraju, Ramprasaath R and Cogswell, Michael and Das, Abhishek and Vedantam, Ramakrishna and Parikh, Devi and Batra, Dhruv},
  journal={International journal of computer vision},
  volume={128},
  pages={336--359},
  year={2020},
  publisher={Springer}
}

@inproceedings{allen2022feature,
  title={Feature purification: How adversarial training performs robust deep learning},
  author={Allen-Zhu, Zeyuan and Li, Yuanzhi},
  booktitle={2021 IEEE 62nd Annual Symposium on Foundations of Computer Science (FOCS)},
  pages={977--988},
  year={2022},
  organization={IEEE}
}

@article{huang2023graph,
  title={Graph neural networks provably benefit from structural information: A feature learning perspective},
  author={Huang, Wei and Cao, Yuan and Wang, Haonan and Cao, Xin and Suzuki, Taiji},
  journal={arXiv preprint arXiv:2306.13926},
  year={2023}
}

@article{jelassi2022vision,
  title={Vision transformers provably learn spatial structure},
  author={Jelassi, Samy and Sander, Michael and Li, Yuanzhi},
  journal={Advances in Neural Information Processing Systems},
  volume={35},
  pages={37822--37836},
  year={2022}
}

@article{li2023theoretical,
  title={A theoretical understanding of shallow vision transformers: Learning, generalization, and sample complexity},
  author={Li, Hongkang and Wang, Meng and Liu, Sijia and Chen, Pin-Yu},
  journal={arXiv preprint arXiv:2302.06015},
  year={2023}
}

@article{han2024feature,
  title={On the feature learning in diffusion models},
  author={Han, Andi and Huang, Wei and Cao, Yuan and Zou, Difan},
  journal={arXiv preprint arXiv:2412.01021},
  year={2024}
}
\bibliographystyle{iclr2026_conference}

\newpage
\appendix

\section{The Use of Large Language Models}
We used LLMs as assistive tools for writing. For writing, LLMs were used to polish language (grammar, wording, and flow) and suggest alternative phrasings; the research problem setup, preliminaries, methods, analyses, and conclusions were conceived and written by the authors.

\section{Notations Table}

\begin{table}[ht]
\centering
\caption{Notation Summary}
\label{tab:notation}
\begin{tabular}{@{}lp{10cm}@{}}
\toprule
\textbf{Symbol} & \textbf{Description} \\ 
\midrule
$\mathbf{x}$ & Input data point with multi-patch structure $\mathbf{x} = [y\mathbf{v}, \boldsymbol{\xi}] \in (\mathbb{R}^d)^2$ \\
$y$ & Binary label ($\pm 1$) \\
$y \cdot \mathbf{v}$ & Label-dependent feature vector (signal component) \\
$\boldsymbol{\xi}$ & Label-independent Gaussian noise $\sim \mathcal{N}(0, \sigma_\xi^2\mathbf{H})$ \\
$\mathbf{z}_t$ & Gaussian privacy noise added at iteration $t$ \\
$\mathrm{SNR}$ & Signal-to-noise ratio $\|\mathbf{v}\|_2/\|\boldsymbol{\xi}\|_2$ \\
$T$ & Total number of training iterations \\
$T_p^*$ & Maximum number of private training iterations \\
$m$ & Number of convolutional filters per class \\
$d$ & Dimension of feature/noise vectors \\
$n$ & Number of training samples \\
$\eta$ & Learning rate in noisy gradient descent \\
$\sigma_0$ & Standard deviation of Gaussian weight initialization \\
$\sigma_{\xi}$ & Standard deviation of Gaussian data noise \\
$\sigma_{z}$ & Standard deviation of Gaussian private noise \\
$\varepsilon,\delta$ & $(\varepsilon,\delta)$-differential privacy parameters \\
$\Gamma_{j,r}^{(t)}$ & Signal learning coefficient for filter $r$ in class $j$ at iteration $t$ \\
$\Phi_{j,r,i}^{(t)}$ & Noise memorization coefficient for sample $i$ and filter $r$ in class $j$ \\
$\sigma(z)$ & Polynomial ReLU activation: $\max\{0,z\}^q$ with $q>2$ \\
$L_D(\mathbf{W})$ & Empirical risk with logistic loss over dataset $D$ \\
$L_{\mathcal{D}}(\mathbf{W})$ & Population risk with logistic loss over data distribution $\mathcal{D}$ \\
$q$ & Polynomial degree in activation function ($q>2$) \\
$\mathbf{H}$ & Orthogonal projection matrix $\mathbf{I}-\mathbf{v}\mathbf{v}^\top/\|\mathbf{v}\|_2^2$ \\
$\kappa$ & Convergence threshold for training loss \\
\bottomrule
\end{tabular}
\end{table}

\section{Additional Related Work}
{\bf Differentially Private Learning}
The most widely used technique for differentially private training in deep learning is differentially private stochastic gradient descent (DP-SGD). However, the accuracy of private deep learning still significantly lags behind that of standard non-private learning across several benchmarks \cite{mcmahan2017learning, papernot2021tempered, tramer2020differentially, de2022unlocking}. To bridge this gap, various techniques have been proposed to enhance DP learning, including adaptive gradient clipping methods that dynamically adjust clipping thresholds \cite{andrew2021differentially, bu2024automatic}, feature extraction or pre-processing before applying DP-SGD \cite{abadi2016deep, tramer2020differentially, de2022unlocking, sun2023importance, bao2023dp, tang2024differentially}, parameter-efficient training strategies via adapters, low-rank weights, or quantization \cite{yu2021differentially, luo2021scalable}, and private noise reduction techniques using tree aggregation mechanisms or filters \cite{kairouz2021practical, zhang2024disk}.

\section{Additional Experimental Details}
\subsection{Real-world Data Experiment}\label{sec:real_exp}
In this experiment, we explore the impact of $\operatorname{SNR}$ in the private learning. Due to the space limitation, we provide experimental setup and more results in \cref{app:exp}.

\textbf{Higher SNR Improves Accuracy Across Various Privacy Budgets.}

\begin{figure}[h!]
    \centering
    \includegraphics[width=.7\linewidth]{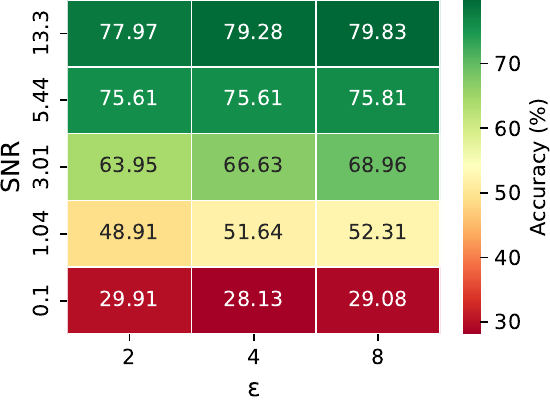}
    \vspace{-0.1in}
    \caption{Impact of $\operatorname{SNR}$ and $\varepsilon$ on CIFAR-10 Accuracy.}
    \label{fig:heatmap}
\end{figure}
The results, illustrated in \cref{fig:heatmap}, reveal that higher $\operatorname{SNR}$ values consistently lead to improved model accuracy under various privacy budgets, as the cleaner signal allows the model to better learn useful feature signals. Moreover, as the privacy budget $\varepsilon$ decreases (indicating stronger privacy guarantees), the model's accuracy degrades, particularly under low $\operatorname{SNR}$ conditions. This degradation is attributed to the combined effects of data noise and the additional noise introduced by private learning.

\textbf{Experimental Setup for \cref{sec:real_exp}.} We conducted experiments on the CIFAR-10 dataset \cite{Krizhevsky2009}, training on a version of the dataset corrupted with Gaussian noise applied at varying scales to control the signal-to-noise ratio \cite{hendrycks2019robustness}. The training data was corrupted with noise levels corresponding to different SNR values, while the test set remained clean. A ResNet-20 architecture \cite{he2016deep} was employed as the baseline model, designed for CIFAR-10 with an input image size of 32$\times$32 pixels. Training was performed using noisy gradient descent with an initial learning rate of $0.1$. The model was trained for $100$ epochs with a batch size of $1000$.

\label{app:exp}
\Cref{fig:cam8}, \Cref{fig:cam4} and \Cref{fig:cam2} depict the Class Activation Maps (CAMs) using GradCAM \cite{selvaraju2020grad} for different classes in the CIFAR-10 dataset under varying SNR. In CAMs, the colors represent the intensity of activation in specific regions of the input image. These activations indicate how strongly the model associates different areas of the image with a specific class. CAMs highlight the most important regions contributing to the model's prediction.
\begin{figure}
    \centering
    \includegraphics[width=.7\textwidth]{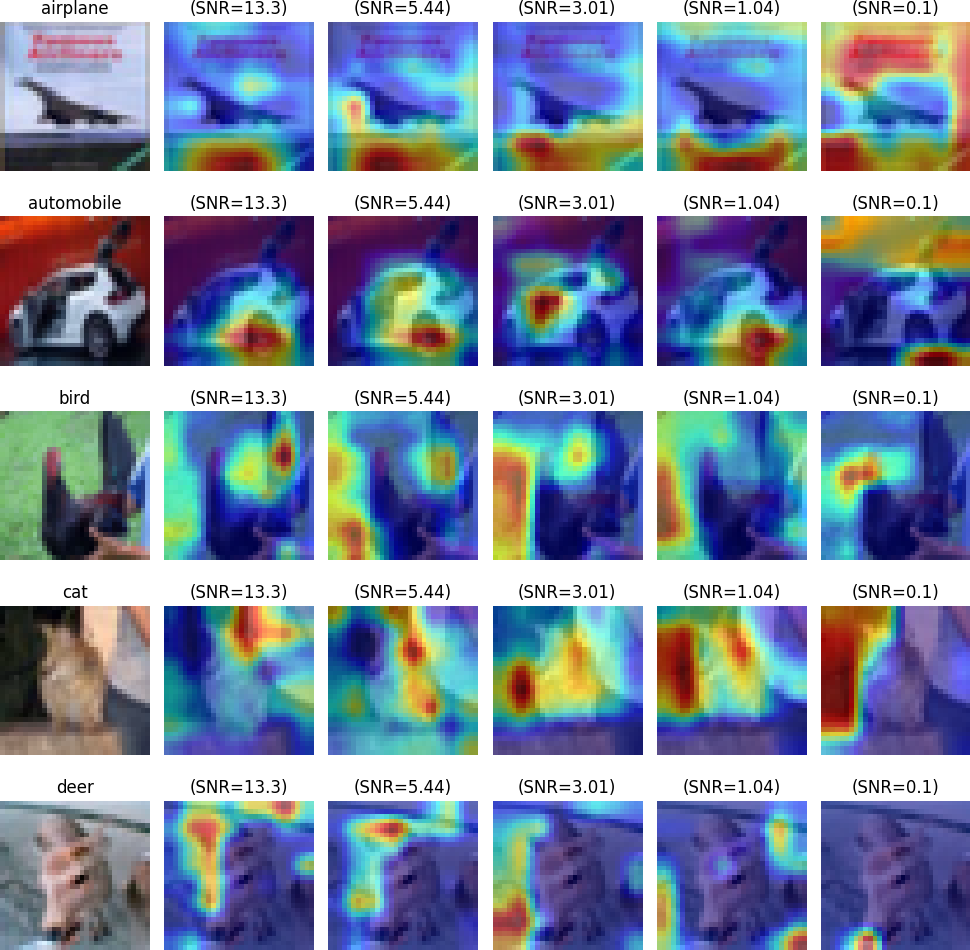}
    \caption{Class Activation Mappings for CIFAR-10 Across Different SNRs ($\epsilon = 8$).}
    \label{fig:cam8}
\end{figure}

\begin{figure}
    \centering
    \includegraphics[width=.7\textwidth]{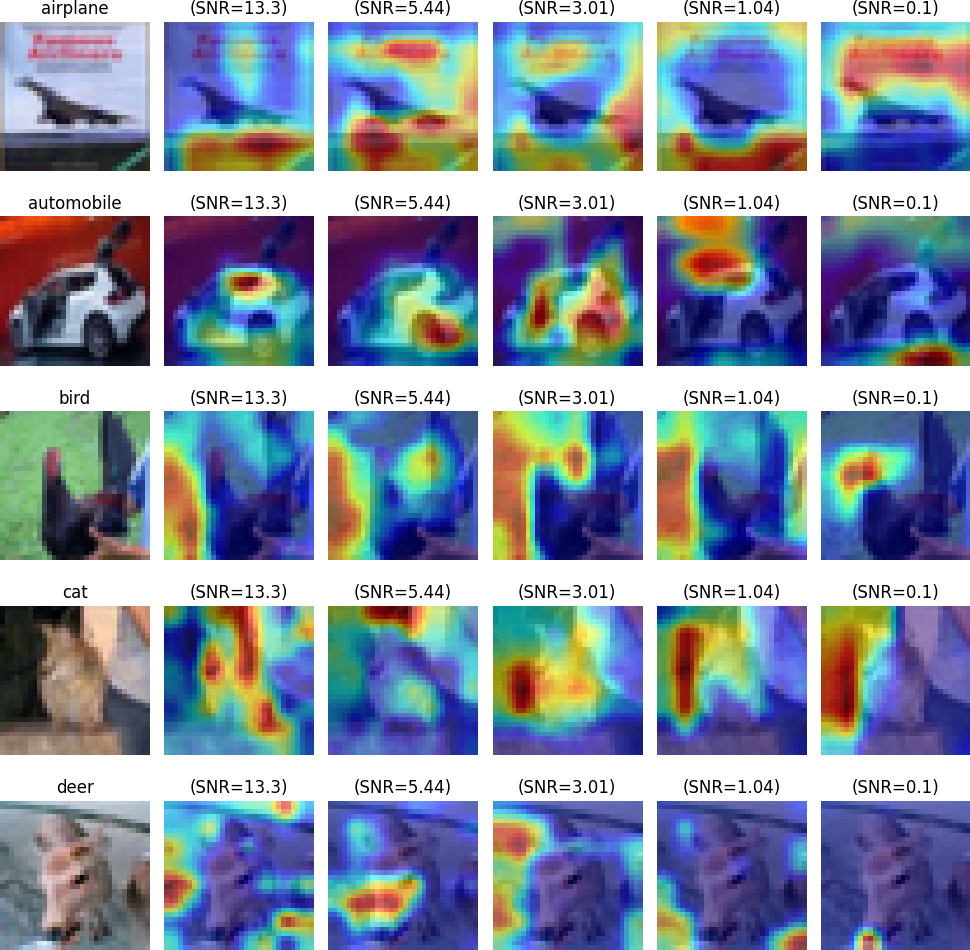}
    \caption{Class Activation Mappings for CIFAR-10 Across Different SNRs ($\epsilon = 4$).}
    \label{fig:cam4}
\end{figure}

\begin{figure}
    \centering
    \includegraphics[width=.7\textwidth]{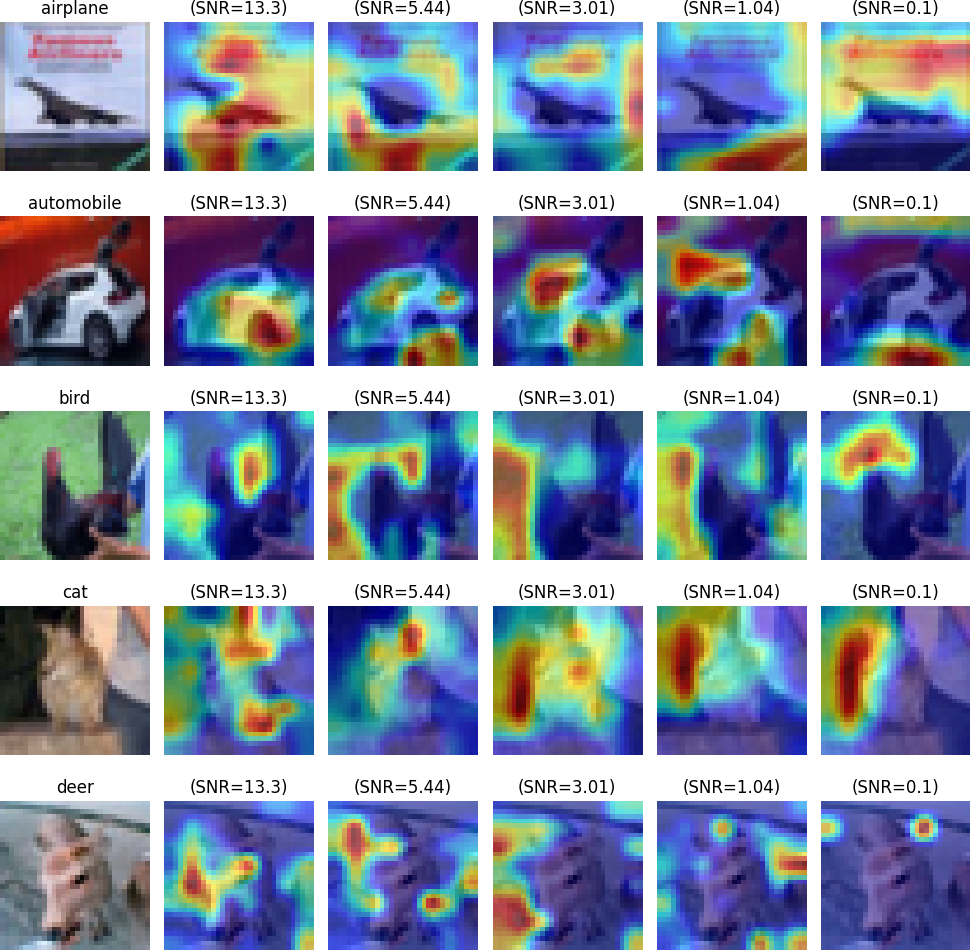}
    \caption{Class Activation Mappings for CIFAR-10 Across Different SNRs ($\epsilon = 2$).}
    \label{fig:cam2}
\end{figure}

\newpage

\section{Support Lemmas}
\begin{lemma}\label{lem:pos_data_number}
Suppose that $\delta_x>0$ and $n \geq 8 \log (4 / \delta_x)$. Then with probability at least $1-\delta_x$,
$$
|\{i \in[n]: y_i=1\}|,|\{i \in[n]: y_i=-1\}| \geq n / 4 .
$$
\end{lemma}
\begin{proof}[\bf Proof of \cref{lem:pos_data_number}]
    We first establish the bound for $|\{i \in[n]: y_i=1\}|$ and the bound for $|\{i \in[n]: y_i=-1\}|$ follows identically.
    Using Hoeffding's inequality, we know that with probability at least $1-\delta / 2$, the following holds:
    $$
    |\frac{1}{n} \sum_{i=1}^n 1\{y_i=1\}-\frac{1}{2}| \leq \sqrt{\frac{\log (4 / \delta)}{2 n}}.
    $$
    Thus, for $n \geq 8 \log (4 / \delta)$, the size of the subset where $y_i=1$ satisfies:
    $$
    |\{i \in[n]: y_i=1\}|=\sum_{i=1}^n 1\{y_i=1\} \geq \frac{n}{2}-n \cdot \sqrt{\frac{\log (4 / \delta)}{2 n}} \geq \frac{n}{4}.
    $$
\end{proof}

\begin{lemma}\label{lem:xi_bound}
Suppose that $\delta_{\xi}>0$ and $d=\Omega(\log (4 n / \delta_{\xi}))$. Then, for all $i, i^{\prime} \in[n]$, with probability at least $1-\delta_{\xi}$,
$$
\begin{aligned}
& \sigma_{\xi}^2 d / 2 \leq\|\bm{\xi}_i\|_2^2 \leq 3 \sigma_{\xi}^2 d / 2 \\
& |\langle\bm{\xi}_i, \bm{\xi}_{i^{\prime}}\rangle| \leq 2 \sigma_{\xi}^2 \cdot \sqrt{d \log (4 n^2 / \delta_{\xi})}.
\end{aligned}
$$
Meanwhile, there is $\delta_{z}$ such that $\delta_{z}>0$ and $d=\Omega(\log (4 n / \delta_{z}))$. It holds, with probability at least $1-\delta_{z}$,
$$
\begin{aligned}
& G^2\sigma_z^2 d / 2 \leq\|\mathbf{z}_i\|_2^2 \leq 3 G^2 \sigma_z^2 d / 2 \\
& |\langle \mathbf{z}_i, \mathbf{z}_{i^{\prime}}\rangle| \leq 2 G^2 \sigma_z^2 \cdot \sqrt{d \log (4 n^2 / \delta_{z})}.
\end{aligned}
$$
\end{lemma}
\begin{proof}[Proof of \cref{lem:xi_bound}]
    Both $\bm{\xi}$ and $\mathbf{z}$ follow Gaussian distributions; therefore, it suffices to provide the proof for one case.
    Using Bernstein's inequality, we find that with probability at least $1-\delta /(2 n)$, the following holds:
    $$|\|\boldsymbol{\xi}_i\|_2^2-\sigma_{\xi}^2 d|=O(\sigma_{\xi}^2 \cdot \sqrt{d \log (4 n / \delta)}).$$
    Thus, when $d=\Omega(\log (4 n / \delta))$, we have:
    $\frac{\sigma_{\xi}^2 d}{2} \leq\|\boldsymbol{\xi}_i\|_2^2 \leq \frac{3 \sigma_{\xi}^2 d}{2}.$
    Next, note that $\langle\boldsymbol{\xi}_i, \boldsymbol{\xi}_{i^{\prime}}\rangle$ has a mean of zero for any $i \neq i^{\prime}$. Again, by Bernstein's inequality, with probability at least $1-\delta /(2 n^2)$, the following bound holds:
    $|\langle\boldsymbol{\xi}_i, \boldsymbol{\xi}_{i^{\prime}}\rangle| \leq 2 \sigma_{\xi}^2 \cdot \sqrt{d \log (4 n^2 / \delta)}.$
    Finally, applying a union bound over all $i$ and $i^{\prime}$ completes the proof.
\end{proof}

\begin{lemma}\label{lem:w0_v_xi}
    Suppose that $d \geq \Omega(\log (m n / \delta)), m=\Omega(\log (1 / \delta))$. Then with probability at least $1-\delta$,
    $$
    \begin{aligned}
    & |\langle\mathbf{w}_{j, r}^{(0)}, \mathbf{v}\rangle| \leq \sqrt{2 \log (8 m / \delta)} \cdot \sigma_0\|\mathbf{v}\|_2 \\
    & |\langle\mathbf{w}_{j, r}^{(0)}, \bm{\xi}_i\rangle| \leq 2 \sqrt{\log (8 m n / \delta)} \cdot \sigma_0 \sigma_{\xi} \sqrt{d}
    \end{aligned}
    $$
    for all $r \in[m], j \in\{ \pm 1\}$ and $i \in[n]$. Moreover,
    $$
    \begin{aligned}
    & \sigma_0\|\mathbf{v}\|_2 / 2 \leq \max _{r \in[m]} j \cdot\langle\mathbf{w}_{j, r}^{(0)}, \mathbf{v}\rangle \leq \sqrt{2 \log (8 m / \delta)} \cdot \sigma_0\|\mathbf{v}\|_2, \\
    & \sigma_0 \sigma_{\xi} \sqrt{d} / 4 \leq \max _{r \in[m]} j \cdot\langle\mathbf{w}_{j, r}^{(0)}, \bm{\xi}_i\rangle \leq 2 \sqrt{\log (8 m n / \delta)} \cdot \sigma_0 \sigma_{\xi} \sqrt{d}
    \end{aligned}
    $$
    for all $j \in\{ \pm 1\}$ and $i \in[n]$.
\end{lemma}

\begin{proof}[\bf Proof of \cref{lem:w0_v_xi}]
    For each $r \in[m]$, the term $j \cdot\langle\mathbf{w}_{j, r}^{(0)}, \mathbf{v}\rangle$ is a Gaussian random variable with mean zero and variance $\sigma_0^2\|\mathbf{v}\|_2^2$. Applying the Gaussian tail bound and the union bound, we conclude that with probability at least $1-\delta / 4$,
    $$
    j \cdot\langle\mathbf{w}_{j, r}^{(0)}, \mathbf{v}\rangle \leq|\langle\mathbf{w}_{j, r}^{(0)}, \mathbf{v}\rangle| \leq \sqrt{2 \log (8 m / \delta)} \cdot \sigma_0\|\mathbf{v}\|_2.
    $$
    Furthermore, the probability $\mathbb{P}(\sigma_0\|\mathbf{v}\|_2 / 2>j \cdot\langle\mathbf{w}_{j, r}^{(0)}, \mathbf{v}\rangle)$ is an absolute constant. Thus, under the given condition on $m$, we have
    $$
    \begin{aligned}
    \mathbb{P}(\sigma_0\|\mathbf{v}\|_2 / 2 \leq \max _{r \in[m]} j \cdot\langle\mathbf{w}_{j, r}^{(0)}, \mathbf{v}\rangle) & =1-\mathbb{P}(\sigma_0\|\mathbf{v}\|_2 / 2>\max _{r \in[m]} j \cdot\langle\mathbf{w}_{j, r}^{(0)}, \mathbf{v}\rangle) \\
    & =1-\mathbb{P}(\sigma_0\|\mathbf{v}\|_2 / 2>j \cdot\langle\mathbf{w}_{j, r}^{(0)}, \mathbf{v}\rangle)^{2 m} \\
    & \geq 1-\delta / 4.
    \end{aligned}
    $$
    By \cref{lem:xi_bound}, with probability at least $1-\delta / 4$, the inequality $\sigma_{\xi} \sqrt{d} / \sqrt{2} \leq\|\boldsymbol{\xi}_i\|_2 \leq \sqrt{3 / 2}$. $\sigma_{\xi} \sqrt{d}$ holds for all $i \in[n]$. Consequently, the result for $\langle\mathbf{w}_{j, r}^{(0)}, \boldsymbol{\xi}_i\rangle$ can be derived using the same argument as for $j \cdot\langle\mathbf{w}_{j, r}^{(0)}, \mathbf{v}\rangle$.
\end{proof}

\section{Decomposition}

\begin{definition}[Restatement of \cref{def:decom_coefficient_main}]\label{def:decom_coefficient}
    Let $\mathbf{w}_{j, r}^{(t)}$ for $j \in\{ \pm 1\}, r \in[m]$ be the convolution filters of the CNN at the $t$-th iteration of noisy gradient descent. Then there exist unique coefficients $\Gamma_{j, r}^{(t)} \geq 0$ and $\Phi_{j, r, i}^{(t)}$ such that
    \begin{equation}\label{eq:w_t_decom}
        \mathbf{w}_{j, r}^{(t)}=\mathbf{w}_{j, r}^{(0)}+j \cdot \Gamma_{j, r}^{(t)} \cdot\|\mathbf{v}\|_2^{-2} \cdot \mathbf{v}+\sum_{i=1}^n \Phi_{j, r, i}^{(t)} \cdot\|\bm{\xi}_i\|_2^{-2} \cdot \bm{\xi}_i - \eta \sum_{s=1}^{t} \mathbf{z}_s
    \end{equation}

We further denote $\bar{\Phi}_{j, r, i}^{(t)}:=\Phi_{j, r, i}^{(t)} \mathds{1}(\Phi_{j, r, i}^{(t)} \geq 0), \underline{\Phi}_{j, r, i}^{(t)}:=\Phi_{j, r, i}^{(t)} \mathds{1}(\Phi_{j, r, i}^{(t)} \leq 0)$. Then we have that

$$
\mathbf{w}_{j, r}^{(t)}=\mathbf{w}_{j, r}^{(0)}+j \cdot \Gamma_{j, r}^{(t)} \cdot\|\mathbf{v}\|_2^{-2} \cdot \mathbf{v}+\sum_{i=1}^n \bar{\Phi}_{j, r, i}^{(t)} \cdot\|\bm{\xi}_i\|_2^{-2} \cdot \bm{\xi}_i+\sum_{i=1}^n \underline{\Phi}_{j, r, i}^{(t)} \cdot\|\bm{\xi}_i\|_2^{-2} \cdot \bm{\xi}_i - \eta \sum_{s=1}^{t} \mathbf{z}_s.
$$
\end{definition}

\begin{lemma}[Restatement of \cref{lem:the_evolution_of_coefficient_main}]\label{lem:the_evolution_of_coefficient}
    The coefficients $\Gamma_{j, r}^{(t)}, \bar{\Phi}_{j, r, i}^{(t)}, \underline{\Phi}_{j, r, i}^{(t)}$ in Definition 4.1 satisfy the following equations:
    $$
    \begin{aligned}
    & \Gamma_{j, r}^{(0)}, \bar{\Phi}_{j, r, i}^{(0)}, \underline{\Phi}_{j, r, i}^{(0)}=0 \\
    & \Gamma_{j, r}^{(t+1)}=\Gamma_{j, r}^{(t)}-\frac{\eta}{n m} \cdot \sum_{i=1}^n \ell_i^{\prime (t)} \cdot \sigma^{\prime}(\langle\mathbf{w}_{j, r}^{(t)}, y_i \cdot \mathbf{v}\rangle) \cdot\|\mathbf{v}\|_2^2\\
    & \bar{\Phi}_{j, r, i}^{(t+1)}=\bar{\Phi}_{j, r, i}^{(t)}-\frac{\eta}{n m} \cdot \ell_i^{\prime (t)} \cdot \sigma^{\prime}(\langle\mathbf{w}_{j, r}^{(t)}, \bm{\xi}_i\rangle) \cdot\|\bm{\xi}_i\|_2^2 \cdot \mathds{1}(y_i=j), \\
    & \underline{\Phi}_{j, r, i}^{(t+1)}=\underline{\Phi}_{j, r, i}^{(t)}+\frac{\eta}{n m} \cdot \ell_i^{\prime (t)} \cdot \sigma^{\prime}(\langle\mathbf{w}_{j, r}^{(t)}, \bm{\xi}_i\rangle) \cdot\|\bm{\xi}_i\|_2^2 \cdot \mathds{1}(y_i=-j)
    \end{aligned}
    $$
\end{lemma}

\begin{proof}[\bf Proof of \cref{lem:the_evolution_of_coefficient}]
    We prove the statement by induction. For the base case $t=0$, it holds that $\Gamma_{j, r}^{(0)}=0$ and $\Phi_{j, r, i}^{(0)}=0$. Now, assume the statement holds for $t=k$. We proceed to the inductive step, considering $t=k+1$ :
    $$
    \begin{aligned}
    \mathbf{w}_{j, r}^{(k+1)}   & =\mathbf{w}_{j, r}^{(k)} - \frac{1}{n m} \sum_{i=1}^n \ell_i^{\prime(k)} \cdot \sigma^{\prime}(\langle\mathbf{w}_{j, r}^{(k)}, \bm{\xi}_i\rangle) \cdot j y_i \bm{\xi}_i + \frac{\eta}{n m} \sum_{i=1}^n \ell_i^{\prime(k)} \cdot \sigma^{\prime}(\langle\mathbf{w}_{j, r}^{(k)}, y_i \mathbf{v}\rangle) \cdot j \mathbf{v} - \eta \mathbf{z}_{k+1} \\
    & = \mathbf{w}_{j, r}^{(0)}+j \cdot \Gamma_{j, r}^{(k)} \cdot\|\mathbf{v}\|_2^{-2} \cdot \mathbf{v}+\sum_{i=1}^n {\Phi}_{j, r, i}^{(k)} \cdot\|\bm{\xi}_i\|_2^{-2} \cdot \bm{\xi}_i - \eta \sum_{s=1}^{k} \mathbf{z}_s \\
    & - \frac{1}{n m} \sum_{i=1}^n \ell_i^{\prime(k)} \cdot \sigma^{\prime}(\langle\mathbf{w}_{j, r}^{(k)}, \bm{\xi}_i\rangle) \cdot j y_i \bm{\xi}_i - \frac{\eta}{n m} \sum_{i=1}^n \ell_i^{\prime (k)} \cdot \sigma^{\prime}(\langle\mathbf{w}_{j, r}^{(k)}, y_i \mathbf{v}\rangle) \cdot j \mathbf{v} - \eta \mathbf{z}_{k+1} \\
    & = \mathbf{w}_{j, r}^{(0)} + j \cdot \|\mathbf{v}\|_2^{-2} \cdot \mathbf{v} ( \Gamma_{j, r}^{(k)} -\frac{\eta}{n m} \cdot \sum_{i=1}^n \ell_i^{\prime(k)} \cdot \sigma^{\prime}(\langle\mathbf{w}_{j, r}^{(k)}, y_i \cdot \mathbf{v}\rangle) \cdot\|\mathbf{v}\|_2^2 ) \\
    & + \sum_{i=1}^n \|\bm{\xi}_i\|_2^{-2} \cdot \bm{\xi}_i \cdot ( {\Phi}_{j, r, i}^{(k)} -\frac{\eta}{n m} \cdot \ell_i^{\prime (k)} \cdot \sigma^{\prime}(\langle\mathbf{w}_{j, r}^{(k)}, \bm{\xi}_i\rangle) \cdot\|\bm{\xi}_i\|_2^2  ) - \sum_{s=1}^{k+1} \mathbf{z}_s \\
    & = \mathbf{w}_{j, r}^{(0)}+j \cdot \Gamma_{j, r}^{(k+1)} \cdot\|\mathbf{v}\|_2^{-2} \cdot \mathbf{v}+\sum_{i=1}^n {\Phi}_{j, r, i}^{(k+1)} \cdot\|\bm{\xi}_i\|_2^{-2} \cdot \bm{\xi}_i - \eta \sum_{s=1}^{k+1} \mathbf{z}_s.
    \end{aligned}
    $$
    The last equality follows directly from the data distribution and it is clear that the vectors involved are linearly independent. Thus, the decomposition is unique. Then, we have:
    $$ 
    {\Phi}_{j, r, i}^{(t)}=- \sum_{s=0}^{t-1} \frac{\eta}{n m} \cdot \ell_i^{\prime(s)} \cdot \sigma^{\prime}(\langle\mathbf{w}_{j, r}^{(s)}, \bm{\xi}_i\rangle) \cdot\|\bm{\xi}_i\|_2^2 \cdot y_i \cdot j.
    $$
    Moreover, note that $\ell_i^{\prime (t)}<0$ due to the definition of the cross-entropy loss. Consequently,
    $$
    \begin{aligned}
    & \bar{\Phi}_{j, r, i}^{(t)}=-\sum_{s=0}^{t-1} \frac{\eta}{n m} \cdot \ell_i^{\prime(s)} \cdot \sigma^{\prime}(\langle\mathbf{w}_{j, r}^{(s)}, \bm{\xi}_i\rangle) \cdot\|\bm{\xi}_i\|_2^2 \cdot \mathds{1}(y_i=j), \\
    & \underline{\Phi}_{j, r, i}^{(t)}=-\sum_{s=0}^{t-1} \frac{\eta}{n m} \cdot \ell_i^{\prime (s)} \cdot \sigma^{\prime}(\langle\mathbf{w}_{j, r}^{(s)}, \bm{\xi}_i\rangle) \cdot\|\bm{\xi}_i\|_2^2 \cdot \mathds{1}(y_i=-j),
    \end{aligned}
    $$
    which completes the proof.

\end{proof}

\textbf{Parameters.} Let 
In the following analysis, we demonstrate that for effective private learning, the learning of feature signals and noise will remain controlled throughout the training process. Let $$T_{p}^*=\eta^{-1} \min \{ (\operatorname{poly}(\Gamma^{-1},\|\mathbf{v}\|_2^{-1}, d^{-1} \sigma_n^{-2}, \sigma_0^{-1}, n, m, d)), \frac{m n \varepsilon \log(d) }{5C  \mu (\|\mathbf{v}\|_2 + \|\bm{\xi}\|_2)} , $$ represent the maximum allowable number of iterations. Denote $\alpha = \log(T_{p}^*)$.
\begin{align} \label{eq:eta_domain}
& \eta=O(\min \{n m /(q \sigma_{\xi}^2 d), n m /(q 2^{q+2} \alpha^{q-2} \sigma_{\xi}^2 d),  m n /(q 2^{q+2} \alpha^{q-2}\|\mathbf{v}\|_2^2)\}) \\ \label{eq:d_domain}
& d \geq 1024 \log (4 n^2 / \delta) \alpha^2 n^2
\end{align}
Let $\beta=2 \max _{i, j, r}\{|\langle\mathbf{w}_{j, r}^{(0)}, \mathbf{v}\rangle|,|\langle\mathbf{w}_{j, r}^{(0)}, \bm{\xi}_i|\}$. By \cref{lem:w0_v_xi}, with probability at least $1-\delta$, we can bound $\beta$ as follows:$\beta \leq 4 \sqrt{\log (\frac{8 m n}{\delta})} \cdot \sigma_0 \cdot \max \{\|\mathbf{v}\|_2, \sigma_{\xi} \sqrt{d}\}$.
Using $\sigma_0$ (we can add additional $\log$ constrain on the $\sigma_0$ in the main), and \cref{eq:d_domain}, it can be obtained that 
\begin{equation}\label{eq:beta_alpha_1}
    4 \max \left\{\beta, 8 n \sqrt{\frac{\log (\frac{4 n^2}{\delta})}{d}} \alpha, { \frac{\eta  C T_p^* \mu \log(1/\delta) ((\|\mathbf{v}\|_2 + \|\bm{\xi}\|_2))}{mn\varepsilon}},  \right\} \leq 1.
\end{equation}

Given that the above conditions hold, we claim that the following property is satisfied for $0 \leq t \leq T_p^*$.

\begin{proposition}\label{pro:bounds_of_coeeficients}
    Under Condition 4.2, for $0 \leq t \leq T_p^*$, we have that
    \begin{align}\label{eq:Gamma_Phi_leq_alpha}
        & 0 \leq \Gamma_{j, r}^{(t)}, \bar{\Phi}_{j, r, i}^{(t)} \leq \alpha, \\ \label{eq:Phi_geq_alpha}
        & 0 \geq \underline{\Phi}_{j, r, i}^{(t)} \geq-\beta-16 n \sqrt{\frac{\log (4 n^2 / \delta)}{d}} \alpha -0.2 \geq-\alpha .
    \end{align}
for all $r \in[m], j \in\{ \pm 1\}$ and $i \in[n]$.
\end{proposition}

\textbf{Bounds of coefficients.} In the following, we first prove the bounds of coefficients.

\begin{lemma}\label{lem:w_t-w_0_v}
    For any $t \geq 0$, it holds that $\langle\mathbf{w}_{j, r}^{(t)}-\mathbf{w}_{j, r}^{(0)}, \mathbf{v}\rangle=j \cdot \Gamma_{j, r}^{(t)} - \eta \sum_{s=1}^{t} \langle \mathbf{z}_s,  \mathbf{v}\rangle$ for all $r \in[m], j \in\{ \pm 1\}$.
\end{lemma}

\begin{proof}[\bf Proof of \cref{lem:w_t-w_0_v}]
    For any time $t \geq 0$, according to the decomposition \cref{eq:w_t_decom}, it holds that:
    $$
    \begin{aligned}
    \langle\mathbf{w}_{j, r}^{(t)}-\mathbf{w}_{j, r}^{(0)}, \mathbf{v}\rangle & =j \cdot \Gamma_{j, r}^{(t)}+\sum_{i^{\prime}=1}^n \bar{\Phi}_{j, r, i^{\prime}}^{(t)}\|\bm{\xi}_{i^{\prime}}\|_2^{-2} \cdot\langle\bm{\xi}_{i^{\prime}}, \mathbf{v}\rangle+\sum_{i^{\prime}=1}^n \underline{\Phi}_{j, r, i^{\prime}}^{(t)}\|\bm{\xi}_{i^{\prime}}\|_2^{-2} \cdot\langle\bm{\xi}_{i^{\prime}}, \mathbf{v}\rangle - \eta \sum_{s=1}^{t} \langle \mathbf{z}_s,  \mathbf{v}\rangle \\
    & =j \cdot \Gamma_{j, r}^{(t)} - \eta \sum_{s=1}^{t} \langle \mathbf{z}_s,  \mathbf{v}\rangle.
    \end{aligned}
    $$
\end{proof}

\begin{lemma}\label{lem:w_t-w_0_xi}
    Under Parameter Choices, suppose \cref{pro:bounds_of_coeeficients} holds at iteration t. Then, it holds that
    $$\text{When $y_i \neq j $:} \quad \langle\mathbf{w}_{j, r}^{(t)}-\mathbf{w}_{j, r}^{(0)}, \bm{\xi}_i\rangle
    \{\begin{array}{l}
    \leq \underline{\Phi}_{j, r, i}^{(t)}+8 n \sqrt{\frac{\log (4 n^2 / \delta)}{d}} \alpha - \eta \sum_{s=1}^{t} \langle \mathbf{z}_s,  \bm{\xi}_i \rangle, \\
    \geq \underline{\Phi}_{j, r, i}^{(t)}-8 n \sqrt{\frac{\log (4 n^2 / \delta)}{d}} \alpha - \eta \sum_{s=1}^{t} \langle \mathbf{z}_s,  \bm{\xi}_i \rangle\\
    \end{array}.
    $$

    $$\text{When $y_i = j $:} \quad \langle\mathbf{w}_{j, r}^{(t)}-\mathbf{w}_{j, r}^{(0)}, \bm{\xi}_i\rangle
    \{\begin{array}{l}
    \leq \bar{\Phi}_{j, r, i}^{(t)}+8 n \sqrt{\frac{\log (4 n^2 / \delta)}{d}} \alpha - \eta \sum_{s=1}^{t} \langle \mathbf{z}_s,  \bm{\xi}_i \rangle, \\
    \geq \bar{\Phi}_{j, r, i}^{(t)}-8 n \sqrt{\frac{\log (4 n^2 / \delta)}{d}} \alpha - \eta \sum_{s=1}^{t} \langle \mathbf{z}_s,  \bm{\xi}_i \rangle\\
    \end{array}.
    $$
for all $r \in[m], j \in\{ \pm 1\}$ and $i \in[n]$.
\end{lemma}

\begin{proof}[\bf Proof of \cref{lem:w_t-w_0_xi}]
For $j \neq y_i$, it holds that $\bar{\Phi}_{j, r, i}^{(t)}=0$ and
$$
\begin{aligned}
\langle\mathbf{w}_{j, r}^{(t)}-\mathbf{w}_{j, r}^{(0)}, \bm{\xi}_i\rangle & =\sum_{i^{\prime}=1}^n \bar{\Phi}_{j, r, i^{\prime}}^{(t)}\|\bm{\xi}_{i^{\prime}}\|_2^{-2} \cdot\langle \bm{\xi}_{i^{\prime}}, \bm{\xi}_i\rangle+\sum_{i^{\prime}=1}^n \Phi_{j, r, i^{\prime}}^{(t)}\|\bm{\xi}_{i^{\prime}}\|_2^{-2} \cdot\langle\bm{\xi}_{i^{\prime}}, \bm{\xi}_i\rangle - \eta \sum_{s=1}^{t} \langle \mathbf{z}_s,  \bm{\xi}_i \rangle \\
& \overset{\text{(i)}}{\leq} 4 \sqrt{\frac{\log (4 n^2 / \delta)}{d}} \sum_{i^{\prime} \neq i}|\bar{\Phi}_{j, r, i^{\prime}}^{(t)}|+4 \sqrt{\frac{\log (4 n^2 / \delta)}{d}} \sum_{i^{\prime} \neq i}|\underline{\Phi}_{j, r, i^{\prime}}^{(t)}|+\underline{\Phi}_{j, r, i}^{(t)}  - \eta \sum_{s=1}^{t} \langle \mathbf{z}_s,  \bm{\xi}_i \rangle\\
& \overset{\text{(ii)}}{\leq} \underline{\Phi}_{j, r, i}^{(t)}+8 n \sqrt{\frac{\log (4 n^2 / \delta)}{d}} \alpha - \eta \sum_{s=1}^{t} \langle \mathbf{z}_s,  \bm{\xi}_i \rangle.
\end{aligned}
$$
The second inequality (i) is derived by \cref{lem:xi_bound} and the last inequality (ii) holds due to \cref{pro:bounds_of_coeeficients}. Similarly, for $y_i = j$, it holds that $\underline{\Phi}_{j, r, i}^{(t)}=0$ and 
$$
\begin{aligned}
\langle\mathbf{w}_{j, r}^{(t)}-\mathbf{w}_{j, r}^{(0)}, \bm{\xi}_i\rangle & =\sum_{i^{\prime}=1}^n \bar{\Phi}_{j, r, i^{\prime}}^{(t)}\|\bm{\xi}_{i^{\prime}}\|_2^{-2} \cdot\langle \bm{\xi}_{i^{\prime}}, \bm{\xi}_i\rangle+\sum_{i^{\prime}=1}^n \Phi_{j, r, i^{\prime}}^{(t)}\|\bm{\xi}_{i^{\prime}}\|_2^{-2} \cdot\langle\bm{\xi}_{i^{\prime}}, \bm{\xi}_i\rangle - \eta \sum_{s=1}^{t} \langle \mathbf{z}_s,  \bm{\xi}_i \rangle \\
& \overset{\text{(iii)}}{\leq} 4 \sqrt{\frac{\log (4 n^2 / \delta)}{d}} \sum_{i^{\prime} \neq i}|\bar{\Phi}_{j, r, i^{\prime}}^{(t)}|+4 \sqrt{\frac{\log (4 n^2 / \delta)}{d}} \sum_{i^{\prime} \neq i}|\underline{\Phi}_{j, r, i^{\prime}}^{(t)}|+\bar{\Phi}_{j, r, i}^{(t)}  - \eta \sum_{s=1}^{t} \langle \mathbf{z}_s,  \bm{\xi}_i \rangle\\
& \overset{\text{(iv)}}{\leq} \bar{\Phi}_{j, r, i}^{(t)}+8 n \sqrt{\frac{\log (4 n^2 / \delta)}{d}} \alpha - \eta \sum_{s=1}^{t} \langle \mathbf{z}_s,  \bm{\xi}_i \rangle.
\end{aligned}
$$
Moreover, it is clear that, according to \cref{lem:xi_bound}, inequalities (i) can also be bounded by:
$$
\begin{aligned}
\text{When $y_i \neq j $:}  &\overset{\text{(i)}^{\prime}}{\geq} \underline{\Phi}_{j, r, i}^{(t)}- 4 \sqrt{\frac{\log (4 n^2 / \delta)}{d}} \sum_{i^{\prime} \neq i}|\bar{\Phi}_{j, r, i^{\prime}}^{(t)}|-4 \sqrt{\frac{\log (4 n^2 / \delta)}{d}} \sum_{i^{\prime} \neq i}|\underline{\Phi}_{j, r, i^{\prime}}^{(t)}|  - \eta \sum_{s=1}^{t} \langle \mathbf{z}_s,  \bm{\xi}_i \rangle\\
&\overset{\text{(ii)}^{\prime}}{\geq} \underline{\Phi}_{j, r, i}^{(t)}-8 n \sqrt{\frac{\log (4 n^2 / \delta)}{d}} \alpha - \eta \sum_{s=1}^{t} \langle \mathbf{z}_s,  \bm{\xi}_i \rangle, 
\end{aligned}
$$
Similarly, the following also holds for the inequalities (iii):
$$
\begin{aligned}
\text{When $y_i = j $:}  &\overset{\text{(iii)}^{\prime}}{\geq} \underline{\Phi}_{j, r, i}^{(t)}- 4 \sqrt{\frac{\log (4 n^2 / \delta)}{d}} \sum_{i^{\prime} \neq i}|\bar{\Phi}_{j, r, i^{\prime}}^{(t)}|-4 \sqrt{\frac{\log (4 n^2 / \delta)}{d}} \sum_{i^{\prime} \neq i}|\underline{\Phi}_{j, r, i^{\prime}}^{(t)}|  - \eta \sum_{s=1}^{t} \langle \mathbf{z}_s,  \bm{\xi}_i \rangle \\
& \overset{\text{(iv)}^{\prime}}{\geq} \bar{\Phi}_{j, r, i}^{(t)} - 8 n \sqrt{\frac{\log (4 n^2 / \delta)}{d}} \alpha - \eta \sum_{s=1}^{t} \langle \mathbf{z}_s,  \bm{\xi}_i \rangle, 
\end{aligned}
$$
which completed the proof.

\end{proof}

\begin{lemma}\label{lem:jneqy_F_1}
    Under Parameter Choices, suppose \cref{pro:bounds_of_coeeficients} holds at iteration t. Then, it holds that
    \begin{align*}
        \langle\mathbf{w}_{j, r}^{(t)}, y_i\mathbf{v}\rangle & \leq \langle\mathbf{w}_{j, r}^{(t)}, y_i\mathbf{v}\rangle  - \eta \sum_{s=1}^{t} \langle \mathbf{z}_s,  y_i \mathbf{v}\rangle, \\
        \langle\mathbf{w}_{j, r}^{(t)}, \bm{\xi}_i\rangle & \leq \langle\mathbf{w}_{j, r}^{(t)},\bm{\xi}_i\rangle  - \eta \sum_{s=1}^{t} \langle \mathbf{z}_s,  \bm{\xi}_i\rangle + 8n\frac{\log(4n^2/\delta)}{d} \alpha, \\
        F_j(\mathbf{W}_j^{(t)},\mathbf{x}_i) &\leq 1
    \end{align*}
    for all $r \in [m]$ and $j \neq y_i$.
\end{lemma}

\begin{proof}[\bf Proof of \cref{lem:jneqy_F_1}]
    According to \cref{lem:the_evolution_of_coefficient}, it is clear that $\Gamma_{j, r}^{(t)}$ is increasing and $\Gamma_{j, r}^{(t)} \geq 0$ holds. For $y_i \neq j$, it holds that
    $$
    \langle\mathbf{w}_{j, r}^{(t)}, y_i \mathbf{v}\rangle=\langle\mathbf{w}_{j, r}^{(0)}, y_i \mathbf{v}\rangle+y_i \cdot j \cdot \Gamma_{j, r}^{(t)} - \eta \sum_{s=1}^{t} \langle \mathbf{z}_s,  y_i \mathbf{v}\rangle \leq\langle\mathbf{w}_{j, r}^{(0)}, y_i \mathbf{v}\rangle - \eta \sum_{s=1}^{t} \langle \mathbf{z}_s,  y_i \mathbf{v}\rangle.
    $$
    Moreover, according to \cref{lem:w_t-w_0_xi}, we have
    $$
    \langle\mathbf{w}_{j, r}^{(t)}, \bm{\xi}_i\rangle \leq\langle\mathbf{w}_{j, r}^{(0)}, \bm{\xi}_i\rangle+\underline{\Phi}_{j, r, i}^{(t)}+8 n \sqrt{\frac{\log (4 n^2 / \delta)}{d}} \alpha - \eta \sum_{s=1}^{t} \langle \mathbf{z}_s,  \bm{\xi}_i \rangle \leq\langle\mathbf{w}_{j, r}^{(0)}, \bm{\xi}_i\rangle+8 n \sqrt{\frac{\log (4 n^2 / \delta)}{d}} \alpha - \eta \sum_{s=1}^{t} \langle \mathbf{z}_s,  \bm{\xi}_i\rangle,
    $$
    where the last inequality is due to $\underline{\Phi}_{j, r, i}^{(t)} \leq 0$. Therefore, we can further obtain
    $$
    \begin{aligned}
    F_j(\mathbf{W}_j^{(t)}, \mathbf{x}_i) & =\frac{1}{m} \sum_{r=1}^m[\sigma(\langle\mathbf{w}_{j, r}^{(t)},-j \cdot \mathbf{v}\rangle)+\sigma(\langle\mathbf{w}_{j, r}^{(t)}, \bm{\xi}_i\rangle)] \\
    & \leq 2^{q+1} \max _{j, r, i}\{|\langle\mathbf{w}_{j, r}^{(0)}, \mathbf{v}\rangle|,|\langle\mathbf{w}_{j, r}^{(0)}, \bm{\xi}_i\rangle|, 8 n \sqrt{\frac{\log (4 n^2 / \delta)}{d}} \alpha,  { \frac{\eta C T_p^* \mu \log(1/\delta) (\|\mathbf{v}\|_2 + \|\bm{\xi}\|_2)}{mn\varepsilon}}\}^q \\
    & \leq 1.
    \end{aligned}
    $$
    The first inequality derives from the previous two conclusions and the second inequality is by \cref{eq:beta_alpha_1}.
\end{proof}

\begin{lemma}\label{lem:F_1}
    Under Parameter Choices, suppose \cref{pro:bounds_of_coeeficients} holds at iteration t. Then, it holds that
    \begin{align*}
        \langle\mathbf{w}_{j, r}^{(t)}, y_i \mathbf{v}\rangle& =\langle\mathbf{w}_{j, r}^{(0)}, y_i \mathbf{v}\rangle+ \Gamma_{j, r}^{(t)} - \eta \sum_{s=1}^{t} \langle \mathbf{z}_s,  y_i \mathbf{v}\rangle , \\
        \langle\mathbf{w}_{j, r}^{(t)}, \bm{\xi}_i\rangle &\leq\langle\mathbf{w}_{j, r}^{(0)}, \bm{\xi}_i\rangle+\bar{\Phi}_{j, r, i}^{(t)}+8 n \sqrt{\frac{\log (4 n^2 / \delta)}{d}} \alpha - \eta \sum_{s=1}^{t} \langle \mathbf{z}_s,  \bm{\xi}_i \rangle ,
    \end{align*}
    for all $r \in [m]$ and $j \in\{\pm 1\}$ and $i \in [n]$. If $\max\{\bar{\Phi}_{j, r, i}^{(t)}, \Gamma_{j, r}^{(t)} \} = O(1)$, we further have $F_j(\mathbf{W}_j^{(t)},\mathbf{x}_i) = O(1)$.
\end{lemma}

\begin{proof}[\bf Proof of \cref{lem:F_1}]
    According to \cref{lem:the_evolution_of_coefficient}, it is clear that, for $y_i = j$, we have
    $$
    \langle\mathbf{w}_{j, r}^{(t)}, y_i \mathbf{v}\rangle=\langle\mathbf{w}_{j, r}^{(0)}, y_i \mathbf{v}\rangle+ \Gamma_{j, r}^{(t)} - \eta \sum_{s=1}^{t} \langle \mathbf{z}_s,  y_i \mathbf{v}\rangle .
    $$
    Moreover, according to \cref{lem:w_t-w_0_xi}, we have
    $$
    \langle\mathbf{w}_{j, r}^{(t)}, \bm{\xi}_i\rangle \leq\langle\mathbf{w}_{j, r}^{(0)}, \bm{\xi}_i\rangle+\bar{\Phi}_{j, r, i}^{(t)}+8 n \sqrt{\frac{\log (4 n^2 / \delta)}{d}} \alpha - \eta \sum_{s=1}^{t} \langle \mathbf{z}_s,  \bm{\xi}_i \rangle.
    $$
    If $\max\{\bar{\Phi}_{j, r, i}^{(t)}, \Gamma_{j, r}^{(t)} \} = O(1)$, we have
    $$
    \begin{aligned}
    F_j(\mathbf{W}_j^{(t)}, \mathbf{x}_i) & =\frac{1}{m} \sum_{r=1}^m[\sigma(\langle\mathbf{w}_{j, r}^{(t)},-j \cdot \mathbf{v}\rangle)+\sigma(\langle\mathbf{w}_{j, r}^{(t)}, \bm{\xi}_i\rangle)] \\
    & \leq 2* 5^{q} \max _{j, r, i}\{|\langle\mathbf{w}_{j, r}^{(0)}, \mathbf{v}\rangle|,|\langle\mathbf{w}_{j, r}^{(0)}, \bm{\xi}_i\rangle|, 8 n \sqrt{\frac{\log (4 n^2 / \delta)}{d}} \alpha,  { \frac{\eta C T_p^* \mu\log(1/\delta) (\|\mathbf{v}\|_2 + \|\bm{\xi}\|_2)}{mn\varepsilon}}\}^q \\
    & =O(1).
    \end{aligned}
    $$
    The first inequality derives from the previous two conclusions and the second inequality is by \cref{eq:beta_alpha_1}.
\end{proof}


    

\begin{lemma}\label{lem:xi_zt2}
    For any iteration $t $, with probability  $1 -\delta$, it holds that 
    \begin{equation}
        |\eta \sum_{s=1}^t\langle\mathbf{z}_s, \mathbf{v}\rangle | \leq {\frac{ \eta C \sqrt{t T}\|\mathbf{v}\|_2^2 \log(1/\delta)}{mn \varepsilon} (1+{\frac{1}{\operatorname{SNR}}})}, \quad |\eta \sum_{s=1}^t\langle\mathbf{z}_s, \bm{\xi}_i \rangle | \leq  {\frac{\eta C\sqrt{t T}\|\bm{\xi}\|_2^2 \log(1/\delta)}{mn \varepsilon} (1+{\operatorname{SNR}})}.
    \end{equation}
\end{lemma}

\begin{proof}[\bf Proof of \cref{lem:xi_zt2}]
    According to \cref{pro:bounds_of_coeeficients} and the update of $\mathbf{w}_{j,r}^{(t)}$ in \cref{eq:w_t_decom_main_2}, it is clear that the sensitivity can be bounded by $\frac{C}{mn}(\|\mathbf{v}\|_2+{\|\bm{\xi}\|_2}) $ with $C = \widetilde{O}(1)$. Therefore, with probability  $1 -\delta$, we have:
    \begin{equation*}
         |\eta \sum_{s=1}^t\langle\mathbf{z}_s, \mathbf{v}\rangle | \leq {\frac{ \eta C \sqrt{t T}\|\mathbf{v}\|_2^2 \log(1/\delta)}{mn \varepsilon} (1+{\frac{1}{\operatorname{SNR}}})}, \quad |\eta \sum_{s=1}^t\langle\mathbf{z}_s, \bm{\xi}_i \rangle | \leq  {\frac{\eta C\sqrt{t T}\|\bm{\xi}\|_2^2 \log(1/\delta)}{mn \varepsilon} (1+{\operatorname{SNR}})}.
    \end{equation*}
    where we use Hoeffding’s inequality and the definition of $\operatorname{SNR} = {\|\mathbf{v}\|_2}/{\|\bm{\xi}\|_2}$.
    
\end{proof}

Now we are ready to provide the proof of \cref{pro:bounds_of_coeeficients}.

\begin{proof}[\bf Proof of \cref{pro:bounds_of_coeeficients}]
    We prove \cref{pro:bounds_of_coeeficients} using an induction process. It is clear that the claim holds for $t=0$ since the coefficients are zero in this case. Suppose there exists $s \leq T_p^*$ such that \cref{pro:bounds_of_coeeficients} holds for all $t \leq s-1$. Our goal is to prove that the claim also holds for $t=s$.

    We first consider when $\underline{\Phi}_{j, r, i}^{(t)} \leq-0.5\beta-8 n \sqrt{\frac{\log (4 n^2 / \delta)}{d}} \alpha - 0.1$. Notice that for any $j = y_i$, we have $\underline{\Phi}_{j, r, i}^{(t)}=0$. Thus, we only need to focus on $j \neq y_i$. Additionally, according to \cref{lem:w_t-w_0_xi} and \cref{lem:xi_zt2}, it holds that
    $$
    \langle\mathbf{w}_{j, r}^{(t)}, \bm{\xi}_i\rangle \leq \underline{\Phi}_{j, r, i}^{(t)}+\langle\mathbf{w}_{j, r}^{(0)}, \bm{\xi}_i\rangle+8 n \sqrt{\frac{\log (4 n^2 / \delta)}{d}} \alpha -\eta \sum_{s=1}^t\langle\mathbf{z}_s, \bm{\xi}_i\rangle \leq 0.
    $$
    Therefore, for $t+1$, we have
    $$
    \begin{aligned}
    \underline{\Phi}_{j, r, i}^{(t+1)} & =\underline{\Phi}_{j, r, i}^{(t)}+\frac{\eta}{n m} \cdot \ell_i^{\prime (t)} \cdot \sigma^{\prime}(\langle\mathbf{w}_{j, r}^{(t)}, \bm{\xi}_i\rangle) \cdot \mathbbm{1}(y_i=-j)\|\bm{\xi}_i\|_2^2  \\
    & =\underline{\Phi}_{j, r, i}^{(t)} \\
    & \geq-\beta-16 n \sqrt{\frac{\log (4 n^2 / \delta)}{d}} \alpha -0.2.
    \end{aligned}
    $$
    When considering $\underline{\Phi}_{j, r, i}^{(t)} \geq-0.5\beta-8 n \sqrt{\frac{\log (4 n^2 / \delta)}{d}} \alpha - 0.1$, it holds that
    \begin{align*}
        \underline{\Phi}_{j, r, i}^{(t+1)} & =\underline{\Phi}_{j, r, i}^{(t)}+\frac{\eta}{n m} \cdot \ell_i^{\prime(t)} \cdot \sigma^{\prime}(\langle\mathbf{w}_{j, r}^{(T-1)}, \bm{\xi}_i\rangle) \cdot \mathbbm{1}(y_i=-j)\|\bm{\xi}_i\|_2^2 \\
        &  \overset{(i)}{\geq }-0.5 \beta-8 n \sqrt{\frac{\log (4 n^2 / \delta)}{d}} \alpha -0.1 -O(\frac{\eta \sigma_{\xi}^2 d}{n m}) \sigma^{\prime}(0.5 \beta+8 n \sqrt{\frac{\log (4 n^2 / \delta)}{d}} \alpha + 0.1) \\
        & \overset{(ii)}{\geq }-0.5 \beta-8 n \sqrt{\frac{\log (4 n^2 / \delta)}{d}} \alpha -0.1 -O(\frac{\eta q \sigma_{\xi}^2 d}{n m})(0.5 \beta+8 n \sqrt{\frac{\log (4 n^2 / \delta)}{d}} \alpha  + 0.1) \\
        & \overset{(iii)}{\geq } -\beta-16 n \sqrt{\frac{\log (4 n^2 / \delta)}{d}} \alpha -0.2,
        \end{align*}
        where $(i)$ holds due to $\ell_i^{\prime(t)} \leq 1$ and $\|\bm{\xi}_i\|_2=O(\sigma_{\xi}^2 d)$; $(ii)$ holds because of $0.5 \beta+8 n \sqrt{\frac{\log (4 n^2 / \delta)}{d}} \alpha +0.1 \leq 1$; $(iii)$ derives from $\eta=O(n m /(q \sigma_{\xi}^2 d))$.

        Now, we proceed to prove \cref{eq:Gamma_Phi_leq_alpha}. Recall the update rule for $\Phi_{j, r,i}^{(t+1)}$ and denote $t_{j,r,i}$ as the first time such that $\Phi_{j, r,i}^{(t)} \geq 0.5 \alpha$, then we can decompose the following
        \begin{equation}\label{eq:pro_alpha}
            \begin{aligned}
           \bar{\Phi}_{j, r,i}^{(t+1)}&=\bar{\Phi}_{j, r,i}^{(t)}-\frac{\eta}{n m} \cdot  \ell_i^{\prime (t)} \cdot \sigma^{\prime}(\langle\mathbf{w}_{j, r}^{(t)}, \bm{\xi}_i\rangle) \cdot \mathbbm{1}(y_i = j)\|\bm{\xi}_i\|_2^2 \\
            &=\bar{\Phi}_{j, r,i}^{(t_{j,r,i})}-\underbrace{\frac{\eta}{n m} \cdot  \ell_i^{\prime(t_{j,r,i})} \cdot \sigma^{\prime}(\langle\mathbf{w}_{j, r}^{(t)}, \bm{\xi}_i\rangle) \cdot\mathbbm{1}(y_i = j)\|\bm{\xi}_i\|_2^2}_{\text{Term 1}}   \\
            &- \underbrace{ \sum_{p = t_{j,r,i} +1}^{t} \frac{\eta}{n m} \ell_i^{\prime (p)} \cdot \sigma^{\prime}(\langle\mathbf{w}_{j, r}^{(p)}, \bm{\xi}_i\rangle) \cdot\mathbbm{1}(y_i = j)\|\bm{\xi}_i\|_2^2}_{\text{Term 2}}. \\
        \end{aligned}
        \end{equation}
        Then, we need to bound Term 1 and Term 2, respectively. For Term 1, we have
        \begin{align*}
            |\text{Term 1}|&={\frac{\eta}{n m} \cdot  \ell_i^{\prime (t_{j,r,i})} \cdot \sigma^{\prime}(\langle\mathbf{w}_{j, r}^{(t)}, \bm{\xi}_i\rangle) \cdot\mathbbm{1}(y_i = j)\|\bm{\xi}_i\|_2^2} \\
             & \overset{(i)}{\leq } 2 q n^{-1} m^{-1} \eta(y_i{\Gamma}_{j, r}^{(t_{j, r})}+ \langle\mathbf{w}_{j, r}^{(0)}, \bm{\xi}_i\rangle +0.2)^{q-1} \|\bm{\xi}_i\|_2^2 \\
             & \overset{(ii)}{\leq } 2^q n^{-1} m^{-1} \eta \alpha^{q-1}\|\bm{\xi}_i\|_2^2 \\
             & \overset{(iii)}{\leq } 0.25 \alpha.
        \end{align*}
        Here, $(i)$ holds due to \cref{lem:w_t-w_0_v}, \cref{lem:w0_v_xi}, \cref{lem:xi_zt2} and the parameter choices of $ T_p^*$; $(ii)$ derives from the parameter choices of $\sigma_0$ and induction hypothesis; $(iii)$ holds by $\eta \leq O(n m /(q 2^{q+2} \alpha^{q-2}\|\bm{\xi}\|_2^2))$.

        For Term 2 and $y_i =j$, we have
        \begin{align*}
            \langle\mathbf{w}_{j, r}^{(t+1)},   \bm{\xi}_i\rangle = & \langle\mathbf{w}_{j, r}^{(0)}, \bm{\xi}_i\rangle  +\bar{\Phi}_{j, r, i}^{(t)}-8 n \sqrt{\frac{\log (4 n^2 / \delta)}{d}} \alpha -\eta \sum_{s=1}^t\langle\mathbf{z}_s,  \cdot \bm{\xi}_i \rangle \\
            \overset{(i)}{\geq } & -0.5\beta + 0.5\alpha - 0.2 
            \geq  0.25 \alpha.
        \end{align*}
        Here, $(i)$ holds due to the definition of $t_{j,r,i}$ and \cref{lem:w0_v_xi}. Similarly, we can also upper bound $\langle\mathbf{w}_{j, r}^{(t+1)},  \bm{\xi}_i\rangle$ as follows:
        \begin{align*}
            \langle\mathbf{w}_{j, r}^{(t+1)}, \bm{\xi}_i \rangle = & \langle\mathbf{w}_{j, r}^{(0)},  \bm{\xi}_i\rangle +\bar{\Phi}_{j, r, i}^{(t)} +8 n \sqrt{\frac{\log (4 n^2 / \delta)}{d}} \alpha -\eta \cdot \sum_{s=1}^t\langle\mathbf{z}_s, \bm{\xi}_i\rangle \\
            {\leq } & 0.5\beta + \alpha + 0.2 \leq  2 \alpha.
        \end{align*}
        Combining the above upper and lower bounds of $\langle\mathbf{w}_{j, r}^{(t+1)},  \bm{\xi}_i\rangle$ into Term 2, it holds that
        \begin{align*}
            |\text{Term 2}|&=| \sum_{p = t_{j,r,i} +1}^{t} \frac{\eta}{n m} \cdot \ell_i^{\prime(p)} \cdot \sigma^{\prime}(\langle\mathbf{w}_{j, r}^{(p)}, \bm{\xi}_i\rangle) \cdot\mathbbm{1}(y_i = j)\|\bm{\xi}_i\|_2^2| \\
             & \overset{(i)}{\leq } \sum_{p = t_{j, r}+1}^{t}  \frac{\eta}{n m} \cdot \exp (-\sigma(\langle\mathbf{w}_{j, r}^{(t)},  \bm{\xi}_i \rangle)+1) \cdot \sigma^{\prime}(\langle\mathbf{w}_{j, r}^{(t)},  \bm{\xi}_i \rangle) \cdot\|\bm{\xi}_i\|_2^2 \\
            & \overset{(ii)}{\leq } {e q 2^q \eta T_p^*} \exp (-\alpha^q / 4^q) \alpha^{q-1} \|\bm{\xi}_i\|_2^2 \\
            & \overset{(i)}{\leq } 0.25 T_p^* \exp (-\alpha^q / 4^q) \alpha \\
            & \overset{(iii)}{\leq } 0.25 T_p^* \exp (-\log (T_p^*)^q) \alpha \\
            & \overset{(iv)}{\leq } 0.25 \alpha.
        \end{align*}
        Here, $(ii)$ follows from \cref{lem:xi_bound}; $(ii)$ is derived from the parameter choice of $\eta$; $(iii)$ and $(iv)$ holds due to the choice $\alpha = 4\log(T_p^*) $ and the fact $\log^q(T_p^*) \geq \log(T_p^*)$. Additionally, $(i)$ is established as follows:
        \begin{align*}
            |\ell_i^{\prime(t)}| &=\frac{1}{1+\exp \{y_i \cdot[F_{+1}(\mathbf{W}_{+1}^{(t)}, \mathbf{x}_i)-F_{-1}(\mathbf{W}_{-1}^{(t)}, \mathbf{x}_i)]\}} \\
            & \leq \exp \{-y_i \cdot[F_{+1}(\mathbf{W}_{+1}^{(t)}, \mathbf{x}_i)-F_{-1}(\mathbf{W}_{-1}^{(t)}, \mathbf{x}_i)]\} \\
            & \leq \exp \{-F_{y_i}(\mathbf{W}_{y_i}^{(t)}, \mathbf{x}_i)+1\},
        \end{align*}   
        where the last inequality follows from \cref{lem:jneqy_F_1}. Combining the bounds for Term 1 and Term 2 into \cref{eq:pro_alpha}, we complete the proof of \cref{eq:Gamma_Phi_leq_alpha}. The same procedure applies to prove $\Gamma \leq \alpha $, with parameter choice $\eta=O(n m /(q 2^{q+2} \alpha^{q-2}\|\mathbf{v}\|_2^2))$.
\end{proof}

\begin{lemma}\label{lem:loss_F}
    Under parameter choices, for $0 \leq t \leq T_p^*$, with probability $1-\delta$, the following result holds.
    $$
    \|\nabla L_D(\mathbf{W}^{(t)})\|_F^2 \leq O(\max \{\|\mathbf{v}\|_2^2, \sigma_{\xi}^2 d\}) L_D(\mathbf{W}^{(t)}) + O(\sigma_{z}^2 d \log (1 / \delta)).
    $$
\end{lemma}

\begin{proof}[\bf Proof of \cref{lem:loss_F}]
    According to the triangle inequality and the definition of noisy gradient, we have
    \begin{equation}
    \|\nabla L_S(\mathbf{W}^{(t)})\|_F^2 \leq 2[\frac{1}{n} \sum_{i=1}^n \ell^{\prime}(y_i f(\mathbf{W}^{(t)}, \mathbf{x}_i))\|\nabla f(\mathbf{W}^{(t)}, \mathbf{x}_i)\|_F]^2 + 2\|\mathbf{z}_t\|_F^2.
    \end{equation}
    The first term is bounded using Lemma C.7 in \cite{cao2022benign}, as it shares the same properties, while the second term is bounded by \cref{lem:xi_bound}. This concludes the proof.
\end{proof}

\section{Signal Learning}

\subsection{First Stage}

\begin{lemma}[Restatement of \cref{lem:signal_first_stage}]\label{lem:app_signal_first_stage}
    Under the same conditions as signal learning, in particular, if we choose
    \begin{equation}\label{eq:SNR_signal}
        \operatorname{SNR} \cdot n \varepsilon \geq \frac{4^q \log (16/e^{1/2}\sigma_0 \|\mathbf{v}\|_2)}{C_1 q}, n \cdot \operatorname{SNR}^q \geq C_1 \log (6 / \sigma_0\|\mathbf{v}\|_2) 2^{2 q+6}[4 \log (8 m n / \delta)]^{(q-1) / 2}
    \end{equation}
    where $C_1=O(1)$ is a positive constant, there exists $T_1=\frac{\log (16 / \sigma_0\|\mathbf{v}\|_2) 4^{q-1} m}{C_1 \eta q \sigma_0^{q-2}\|\mathbf{v}\|_2^q}$ such that
    \begin{itemize}
        \item $\max _r \Gamma_{j, r}^{(T_1)}=\Omega(1)$ for $j \in\{ \pm 1\}$.
        \item $|\Phi_{j, r, i}^{(t)}|=O(\sigma_0 \sigma_{\xi} \sqrt{d})$ for all $j \in\{ \pm 1\}, r \in[m], i \in[n]$ and $0 \leq t \leq T_1$.
    \end{itemize}
\end{lemma}

\begin{proof}[\bf Proof of \cref{lem:app_signal_first_stage}]
    First, when $t \leq T_1^{+}=\min \{\frac{n m \eta^{-1} \sigma_0^{2-q} \sigma_{\xi}^{-q} d^{-q / 2}}{2^{q+4} q[4 \log ({{8mn}}/{\delta})]^{(q-1) / 2}}, \frac{\sigma_0 m n \varepsilon}{\eta \|\bm{\xi}\|_2(\|\mathbf{v}\|_2 + \|\bm{\xi}\|_2)} ,\frac{\sigma_0 m n \varepsilon}{\eta (\|\mathbf{v}\|_2 + \|\bm{\xi}\|_2)} \}$, it can be noticed that the noise remains well controlled. To proceed, define $\Psi^{(t)}=\max _{j, r, i}|\Phi_{j, r, i}^{(t)}|=\max _{j, r, i}\{\bar{\Phi}_{j, r, i}^{(t)},-\underline{\Phi}_{j, r, i}^{(t)}\}$. and assume $\Psi^{(t)} \leq \frac{\sigma_0 \sigma_{\xi} \sqrt{d}}{2}$ for all $0 \leq t \leq T_1^{+}$. Then, we aim to prove that the same holds for $t+1$ using an induction process.
    Recall the update of $\bar{\Phi}$ and $\underline{\Phi}$ as follows:
    \begin{align*}
        & \bar{\Phi}_{j, r, i}^{(t)}=-\sum_{s=0}^{t-1} \frac{\eta}{n m} \cdot \ell_i^{\prime (s)} \cdot \sigma^{\prime}(\langle\mathbf{w}_{j, r}^{(s)}, \bm{\xi}_i\rangle) \cdot\|\bm{\xi}_i\|_2^2 \cdot \mathbbm{1}(y_i=j), \\
        & \underline{\Phi}_{j, r, i}^{(t)}=-\sum^{t-1} \frac{\eta}{n m} \cdot \ell_i^{\prime(s)} \cdot \sigma^{\prime}(\langle\mathbf{w}_{j, r}^{(s)}, \bm{\xi}_i\rangle) \cdot\|\bm{\xi}_i\|_2^2 \cdot \mathbbm{1}(y_i=-j).
    \end{align*}
    Moreover, according to \cref{def:decom_coefficient_main}, we have $$\mathbf{w}_{j, r}^{(t)}=\mathbf{w}_{j, r}^{(0)}+j \cdot \Gamma_{j, r}^{(t)} \cdot \frac{\mathbf{v}}{\|\mathbf{v}\|_2^2}+\sum_{i=1}^n \Phi_{j, r, i}^{(t)} \cdot \frac{\bm{\xi}_i}{\|\bm{\xi}_i\|_2^2}-\eta \sum_{s=1}^t \mathbf{z}_s.$$ 
    Substituting this expression into the updates for $\bar{\Phi}$ and $\Phi$, it follows that:
    \begin{align*}
        \bar{\Phi}_{j, r, i}^{(t+1)} & =\bar{\Phi}_{j, r, i}^{(t)}-\frac{\eta}{n m} \cdot \ell_i^{\prime (t)} \sigma^{\prime}(\langle\mathbf{w}_{j, r}^{(0)}, \bm{\xi}_i\rangle+\sum_{i^{\prime}=1}^n \bar{\Phi}_{j, r, i^{\prime}}^{(t)} \frac{\langle\bm{\xi}_{i^{\prime}}, \bm{\xi}_i\rangle}{\|\bm{\xi}_{i^{\prime}}\|_2^2}+\sum_{i^{\prime}=1}^n \underline{\Phi}_{j, r, i^{\prime}}^{(t)} \frac{\langle\bm{\xi}_{i^{\prime}}, \bm{\xi}_i\rangle}{\|\bm{\xi}_{i^{\prime}}\|_2^2}-\eta \sum_{s=1}^t \langle\mathbf{z}_s, \bm{\xi}_i \rangle) \cdot\|\bm{\xi}_i\|_2^2  , \\
        \underline{\Phi}_{j, r, i}^{(t+1)} &=\underline{\Phi}_{j, r, i}^{(t)}+\frac{\eta}{n m} \cdot \ell_i^{\prime (t)} \sigma^{\prime}(\langle\mathbf{w}_{j, r}^{(0)}, \bm{\xi}_i\rangle+\sum_{i^{\prime}=1}^n \bar{\Phi}_{j, r, i^{\prime}}^{(t)} \frac{\langle\bm{\xi}_{i^{\prime}}, \bm{\xi}_i\rangle}{\|\bm{\xi}_{i^{\prime}}\|_2^2}+\sum_{i^{\prime}=1}^n \underline{\Phi}_{j, r, i^{\prime}}^{(t)} \frac{\langle\bm{\xi}_{i^{\prime}}, \bm{\xi}_i\rangle}{\|\bm{\xi}_{i^{\prime}}\|_2^2} -\eta \sum_{s=1}^t \langle\mathbf{z}_s, \bm{\xi}_i \rangle ) \cdot\|\bm{\xi}_i\|_2^2.
    \end{align*}
    Therefore, it holds that 
    \begin{align*}
        \Psi^{(t+1)} & \leq \Psi^{(t)}+\max _{j, r, i}\{\frac{\eta}{n m} \cdot|\ell_i^{\prime (t)}| \cdot \sigma^{\prime}(\langle\mathbf{w}_{j, r}^{(0)}, \bm{\xi}_i\rangle+\sum_{i^{\prime}=1}^n \Psi^{(t)} \cdot \frac{|\langle\bm{\xi}_{i^{\prime}}, \bm{\xi}_i\rangle|}{\|\bm{\xi}_{i^{\prime}}\|_2^2}\\
        &+\sum_{i^{\prime}=1}^n \Psi^{(t)} \cdot \frac{|\langle\bm{\xi}_{i^{\prime}}, \bm{\xi}_i\rangle|}{\|\bm{\xi}_{i^{\prime}}\|_2^2} -\eta \sum_{s=1}^t \langle\mathbf{z}_s, \bm{\xi}_i \rangle ) \cdot\|\bm{\xi}_i\|_2^2\} \\
        & \overset{(i)}{\leq} \Psi^{(t)}+\max _{j, r, i}\{\frac{\eta}{n m} \cdot \sigma^{\prime}(\langle\mathbf{w}_{j, r}^{(0)}, \bm{\xi}_i\rangle+2 \cdot \sum_{i^{\prime}=1}^n \Psi^{(t)} \cdot \frac{|\langle\bm{\xi}_{i^{\prime}}, \bm{\xi}_i\rangle|}{\|\bm{\xi}_{i^{\prime}}\|_2^2} -\eta \sum_{s=1}^t \langle\mathbf{z}_s, \bm{\xi}_i \rangle ) \cdot\|\bm{\xi}_i\|_2^2\} \\
        & \overset{(ii)}{=} \Psi^{(t)}+\max _{j, r, i}\{\frac{\eta}{n m} \cdot \sigma^{\prime}(\langle\mathbf{w}_{j, r}^{(0)}, \bm{\xi}_i\rangle+2 \Psi^{(t)}+2 \cdot \sum_{i^{\prime} \neq i}^n \Psi^{(t)} \cdot \frac{|\langle\bm{\xi}_{i^{\prime}}, \bm{\xi}_i\rangle|}{\|\bm{\xi}_{i^{\prime}}\|_2^2} -\eta \sum_{s=1}^t \langle\mathbf{z}_s, \bm{\xi}_i \rangle ) \cdot\|\bm{\xi}_i\|_2^2\} \\
        & \overset{(iii)}{\leq} \Psi^{(t)}+\frac{\eta q}{n m} \cdot[2 \cdot \sqrt{\log (8 m n / \delta)} \cdot \sigma_0 \sigma_{\xi} \sqrt{d} \\
        &+(2+\frac{4 n \sigma_{\xi}^2 \cdot \sqrt{d \log (4 n^2 / \delta)}}{\sigma_{\xi}^2 d / 2} ) \cdot \Psi^{(t)} -\eta  \sum_{s=1}^t  \langle\mathbf{z}_s, \bm{\xi}_i \rangle ]^{q-1} \cdot 2 \sigma_{\xi}^2 d \\
        & \overset{(iv)}{\leq} \Psi^{(t)}+\frac{\eta q}{n m} \cdot(2 \cdot \sqrt{\log (8 m n / \delta)} \cdot \sigma_0 \sigma_{\xi} \sqrt{d}+4 \Psi^{(t)} -\eta \sum_{s=1}^t  \langle\mathbf{z}_s, \bm{\xi}_i \rangle  )^{q-1} \cdot 2 \sigma_{\xi}^2 d \\
        & \overset{(v)}{\leq} \Psi^{(t)}+\frac{\eta q}{n m} \cdot(4 \cdot \sqrt{\log (8 m n / \delta)} \cdot \sigma_0 \sigma_{\xi} \sqrt{d})^{q-1} \cdot 2 \sigma_{\xi}^2 d.
        \end{align*}
        Here, the inequality $(i)$ holds due to the fact $|\ell_i^{\prime (t)}| \leq 1$; the equality $(ii)$ is the decomposition of index $i$; the inequality $(iii)$ comes from \cref{lem:xi_bound}; $(iv)$ is derived from the condition of $d \geq 16 n^2 \log (4 n^2 / \delta)$; $(v)$ follows from the induction hypothesis, \cref{lem:xi_zt2} and $T_1^{+}$. By applying a telescoping sum over $t$, we obtain the following result:
        \begin{align*}
        \Psi^{(t+1)} & \leq (t+1) \cdot \frac{\eta q}{n m} \cdot(4 \cdot \sqrt{\log (8 m n / \delta)} \cdot \sigma_0 \sigma_{\xi} \sqrt{d})^{q-1} \cdot 2 \sigma_{\xi}^2 d \\
        & \leq T_1^{+} \cdot \frac{\eta q}{n m} \cdot(4 \cdot \sqrt{\log (8 m n / \delta)} \cdot \sigma_0 \sigma_{\xi} \sqrt{d})^{q-1} \cdot 2 \sigma_{\xi}^2 d \\
        & \leq \frac{\sigma_0 \sigma_{\xi} \sqrt{d}}{2},
        \end{align*}
        where the second inequality follows from the induction hypothesis, while the last inequality is due to the range of $T_1^{+}$.

        Now, we move forward to the proof of $\Gamma$. Without loss of generality, we first consider $j=1$. Let $T_{1,1}$ be the first time such that $\max _r \Gamma_{1, r}^{(t)} \leq 2$ in the period of $[0,T_1^{+}]$, then it also holds that $\max _{j, r, i}\{|\Phi_{j, r, i}^{(t)}|\}=O(\sigma_0 \sigma_{\xi} \sqrt{d})=O(1)$.

        According to \cref{lem:w_t-w_0_v} and \cref{lem:w_t-w_0_xi}, it holds that $F_{-1}(\mathbf{W}_{-1}^{(t)}, \mathbf{x}_i), F_{+1}(\mathbf{W}_{+1}^{(t)}, \mathbf{x}_i)=O(1)$ for all $i$ with $y_i=1$ with the conditions that $\max _r \Gamma_{1, r}^{(t)} , \max _{j,i,r}|\Phi_{j, r, i}^{(t)}|$ are bounded. Moreover, we know that $|\ell_i^{\prime(t)}|=\frac{1}{1+\exp \{y_i \cdot[F_{+1}(\mathbf{W}_{+1}^{(t)}, \mathbf{x}_i)-F_{-1}(\mathbf{W}_{-1}^{(t)}, \mathbf{x}_i)]\}}$, which implies the existence of a positive constant $C_1$ such that $-\ell_i^{\prime (t)} \geq C_1$ for all $i$ with $y_i=1$.

        Thus, with the update of $\Gamma$ in \cref{def:decom_coefficient}, we have
        \begin{align*}
        \Gamma_{1, r}^{(t+1)} & =\Gamma_{1, r}^{(t)}-\frac{\eta}{n m} \cdot \sum_{i=1}^n \ell_i^{\prime (t)} \cdot \sigma^{\prime}(y_i \cdot\langle\mathbf{w}_{1, r}^{(0)}, \mathbf{v}\rangle+y_i \cdot \Gamma_{1, r}^{(t)} - -\eta \sum_{s=1}^t\langle\mathbf{z}_s, y_i \mathbf{v}\rangle) \cdot\|\mathbf{v}\|_2^2 \\
        & \geq \Gamma_{1, r}^{(t)}+\frac{C_1 \eta}{n m} \cdot \sum_{w=1} \sigma^{\prime}(\langle\mathbf{w}_{1, r}^{(0)}, \mathbf{v}\rangle+\Gamma_{1, r}^{(t)} -\eta \sum_{s=1}^t\langle\mathbf{z}_s,  \mathbf{v}\rangle) \cdot\|\mathbf{v}\|_2^2.
        \end{align*}
        It is noticed that $T_1^{+} \leq \frac{\sigma_0 m n \varepsilon }{4\eta  (\|\mathbf{v}\|_2 + \|\bm{\xi}\|_2)}$ and $\max _r\langle\mathbf{w}_{1, r}^{(0)}, \mathbf{v}\rangle \geq \sigma_0\|\mathbf{v}\|_2 / 2$ due to \cref{lem:w0_v_xi}, then it holds that $\max _r\langle\mathbf{w}_{1, r}^{(0)}, \mathbf{v}\rangle -\eta \sum_{s=1}^t\langle\mathbf{z}_s,  \mathbf{v}\rangle \geq \sigma_0\|\mathbf{v}\|_2 / 4$. Denote $\widehat{\Gamma}_{1, r}^{(t)}=\Gamma_{1, r}^{(t)}+\langle\mathbf{w}_{1, r}^{(0)}, \mathbf{v}\rangle -\eta \sum_{s=1}^t\langle\mathbf{z}_s,  \mathbf{v}\rangle $ and let $A^{(t)} = \max \widehat{\Gamma}_{1, r}^{(t)}$, then it follows that 
        \begin{align*}
        A^{(t+1)} & \geq A^{(t)}+\frac{C_1 \eta}{n m} \cdot \sum_{y_i=1} \sigma^{\prime}(A^{(t)}) \cdot\|\mathbf{v}\|_2^2 -\eta \langle\mathbf{z}_{t+1},  \mathbf{v}\rangle  \\
        & \overset{(i)}{\geq} A^{(t)}+\frac{C_1 \eta q\|\mathbf{v}\|_2^2}{4 m}[A^{(t)}]^{q-1} -\eta \langle\mathbf{z}_{t+1},  \mathbf{v}\rangle  \\
        & \overset{(ii)}{\geq} (1+\frac{C_1 \eta q\|\mathbf{v}\|_2^2}{4 m}[A^{(0)}]^{q-2}) A^{(t)} -\eta \langle\mathbf{z}_{t+1},  \mathbf{v}\rangle  \\
        & \overset{(iii)}{\geq} (1+\frac{C_1 \eta q \sigma_0^{q-2}\|\mathbf{v}\|_2^q}{4^{q-1} m}) A^{(t)} -\eta \langle\mathbf{z}_{t+1},  \mathbf{v}\rangle \\
        \text{Let} \quad q_{\Gamma} = (1+\frac{C_1 \eta q \sigma_0^{q-2}\|\mathbf{v}\|_2^q}{4^{q-1} m}), \text{then} \quad & \overset{(iv)}{=} q_{\Gamma}^t  A^{(0)} -(\frac{q_{\Gamma}^t -1}{q_{\Gamma}-1})\eta \langle\mathbf{z}_{t+1},  \mathbf{v}\rangle .
        \end{align*}
        Here, $(i)$ holds due to the lower bound on the number of positive data in \cref{lem:pos_data_number}; $(ii)$ and $(iii)$ follow from that the facts $A^{(t)}$ is increasing and $\max _r\langle\mathbf{w}_{1, r}^{(0)}, \mathbf{v}\rangle -\eta \sum_{s=1}^t\langle\mathbf{z}_s,  \mathbf{v}\rangle \geq \sigma_0\|\mathbf{v}\|_2 / 4$; $(iv)$ is the summation of geometric series. In addition, we know that $1+z \geq \operatorname{exp}(z/2)$ for $z\leq 2$ and $1+z \leq \operatorname{exp}(z)$ for $z\geq 0$, then the following inequality holds
        \begin{align*}
            A^{(t)} &\geq (1+\frac{C_1 \eta q \sigma_0^{q-2}\|\mathbf{v}\|_2^q}{4^{q-1} m})^t A^{(0)} - \frac{4^{q-1}m}{C_1  q \sigma_0^{q-2}\|\mathbf{v}\|_2^q}(q_{\Gamma}^t -1 ) \cdot \frac{ \|\mathbf{v}\|_2(\|\mathbf{v}\|_2+\|\bm{\xi}\|_2)}{mn\varepsilon}\\
            & \geq \exp (\frac{C_1 \eta q \sigma_0^{q-2}\|\mathbf{v}\|_2^q}{4^{q-1}*2 m} t) \frac{\sigma_0 \|\mathbf{v}\|_2}{2} -  \frac{4^{q-1}m}{C_1  q \sigma_0^{q-2}\|\mathbf{v}\|_2^q}\cdot \exp  (\frac{C_1 \eta q \sigma_0^{q-2}\|\mathbf{v}\|_2^q}{4^{q-1} m} t) \cdot \frac{ \|\mathbf{v}\|_2(\|\mathbf{v}\|_2+\|\bm{\xi}\|_2)}{mn\varepsilon} \\
            & = ( \frac{ e^{1/2} \sigma_0 \|\mathbf{v}\|_2}{2} - \frac{4^{q-1}(\|\mathbf{v}\|_2+\|\bm{\xi}\|_2) }{C_1  q \sigma_0^{q-2}\|\mathbf{v}\|_2^{q-1} n \varepsilon } ) \cdot \exp  (\frac{C_1 \eta q \sigma_0^{q-2}\|\mathbf{v}\|_2^q}{4^{q-1} m} t)\\
            & \geq  \frac{ e^{1/2} \sigma_0 \|\mathbf{v}\|_2}{4}  \cdot \exp  (\frac{C_1 \eta q \sigma_0^{q-2}\|\mathbf{v}\|_2^q}{4^{q-1} m} t),
        \end{align*}
        where the last inequality holds due to the choice of $\sigma_0 \geq O({\|\mathbf{v}\|_2^{-1}(n\varepsilon)^{-1/q-1}})$. Therefore, it is clear that $A^{(t)}$ will reach $4$ within $T_1=\frac{\log (16 / \sigma_0\|\mathbf{v}\|_2) 4^{q-1} m}{C_1 \eta q \sigma_0^{q-2}\|\mathbf{v}\|_2^q}$ iterations, which indicates that $\max _r \Gamma_{1, r}^{(t)}$ will reach 2 within $T_1$ iterations. Moreover, we can verify that 
        $$ T_1=\frac{\log (16 / \sigma_0\|\mathbf{v}\|_2) 4^{q-1} m}{C_1 \eta q \sigma_0^{q-2}\|\mathbf{v}\|_2^q} \leq \eta^{-1} \sigma_0 m n \varepsilon(\|\mathbf{v}\|_2+\|\bm{\xi}\|_2)^{-1} \leq T_1^+,$$
        The inequality follows from the SNR condition in \cref{eq:SNR_signal} and the choice of $\sigma_0$. Moreover, since we also have $\sigma_0 \leq O(\|\bm{\xi}\|_2^{-1}\varepsilon^{-1/q})$, it holds that
        $$ T_1=\frac{\log (16 / \sigma_0\|\mathbf{v}\|_2) 4^{q-1} m}{C_1 \eta q \sigma_0^{q-2}\|\mathbf{v}\|_2^q} \leq \frac{n m \eta^{-1} \sigma_0^{2-q} \sigma_{\xi}^{-q} d^{-q / 2}}{2^{q+4} q[4 \log ({{8mn}}/{\delta})]^{(q-1) / 2}} \leq T_1^+.$$
        Hence, by the definition of $T_{1,1}$, $T_{1,1} \leq T_1 \leq T_1^{+} / 2$ holds. A similar proof applies for $j=-1$, where we can prove that $\max _r \Gamma_{-1, r}^{(T_{1,-1})} \geq 2$ with $T_{1,-1} \leq T_1 \leq T_1^{+} / 2$, thereby completing the proof.
\end{proof}

\subsection{Second Stage}

It is clear that we have the following results at the end of the first stage:
$$
\mathbf{w}_{j, r}^{(T_1)}=\mathbf{w}_{j, r}^{(0)}+j \cdot \Gamma_{j, r}^{(T_1)} \cdot \frac{\mathbf{v}}{\|\mathbf{v}\|_2^2}+\sum_{i=1}^n \bar{\Phi}_{j, r, i}^{(T_1)} \cdot \frac{\bm{\xi}_i}{\|\bm{\xi}_i\|_2^2}+\sum_{i=1}^n \underline{\Phi}_{j, r, i}^{(T_1)} \cdot \frac{\bm{\xi}_i}{\|\bm{\xi}_i\|_2^2}-\eta \sum_{s=1}^{T_1} \mathbf{z}_s
$$

Meanwhile, at the beginning of the second stage, we have the following results:
\begin{itemize}
    \item $\max _r \Gamma_{j, r}^{(T_1)} \geq 2, \forall j \in\{ \pm 1\}$.
    \item $\max _{j, r, i}|\Phi_{j, r, i}^{(T_1)}| \leq \widehat{\beta}$ where $\widehat{\beta}=\sigma_0 \sigma_{\xi} \sqrt{d} / 2$.
\end{itemize}
Based on \cref{lem:the_evolution_of_coefficient_main} and \cref{lem:xi_zt2_main}, we conclude that signal learning does not deteriorate over time. Specifically, for any $T_1 \leq t \leq T_p^*$, it holds that $\Gamma_{j, r}^{(t+1)} \geq \Gamma_{j, r}^{(t)}$, which implies $\max _r \Gamma_{j, r}^{(t)} \geq 2$.
If we consider $
\mathbf{w}_{j, r}^*=\mathbf{w}_{j, r}^{(0)}+2 q m \log (2 q / \kappa) \cdot j \cdot \frac{\mathbf{v}}{\|\mathbf{v}\|_2^2}$, then we can derived:

\begin{lemma}\label{lem:W_T1-W_*_signal}
    Under the same conditions as signal learning, we have that $\|\mathbf{W}^{(T_1)}-\mathbf{W}^*\|_F \leq$ $\widetilde{O}(m^{3 / 2}\|\mathbf{v}\|_2^{-1})$.
\end{lemma}

\begin{proof}[\bf Proof of \cref{lem:W_T1-W_*_signal}]
According to triangle inequality, we have
    \begin{align*}
    \|\mathbf{W}^{(T_1)}-\mathbf{W}^*\|_F & \leq\|\mathbf{W}^{(T_1)}-\mathbf{W}^{(0)}\|_F+\|\mathbf{W}^{(0)}-\mathbf{W}^*\|_F \\
    & \overset{(i)}{\leq} \sum_{j, r} \frac{\Gamma_{j, r}^{(T_1)}}{\|\mathbf{v}\|_2}+\sum_{j, r, i} \frac{|\bar{\Phi}_{j, r, r}^{(T_1)}|}{\|\bm{\xi}_i\|_2}+\sum_{j, r, i} \frac{|\underline{\Phi}_{j, r, i}^{(T_1)}|}{\|\bm{\xi}_i\|_2} + \sum_{j, r} |\eta \sum_{s=1}^{T_1} \mathbf{z}_s| +O(m^{3 / 2} \log (1 / \kappa))\|\mathbf{v}\|_2^{-1} \\
    & \overset{(ii)}{\leq}\widetilde{O}(m\|\mathbf{v}\|^{-1})+O(n m \sigma_0) +O( {m \sigma_0}) +O(m^{3 / 2} \log (1 / \kappa))\|\mathbf{v}\|_2^{-1} \\
    & \overset{(iii)}{\leq} \widetilde{O}(m^{3 / 2}\|\mathbf{v}\|_2^{-1})+ O(n m \sigma_0)\\
    &\overset{(iv)}{\leq} \widetilde{O}(m^{3 / 2}\|\mathbf{v}\|_2^{-1}).
    \end{align*}
    Here, $(i)$ holds due to the decomposition of $\mathbf{w}$ in \cref{def:decom_coefficient_main} and the definition of $\mathbf{W}^*$; $(ii)$ follows from \cref{pro:bounds_of_coeeficients}, \cref{lem:app_signal_first_stage} and \cref{lem:xi_zt2}; $(iii)$ comes from the conditions of $\sigma_0$ in signal learning; $(iv)$ holds due to the choice of $\sigma_0$.
\end{proof}

\begin{lemma}\label{lem:nabla_f_W_*_signal}
    Under the same conditions as signal learning, we have that $y_i\langle\nabla f(\mathbf{W}^{(t)}, \mathbf{x}_i), \mathbf{W}^*\rangle \geq$ $q \log (2 q / \kappa)$ for all $i \in[n]$ and $T_1 \leq t \leq T^*$.
\end{lemma}

\begin{proof}[\bf Proof of \cref{lem:nabla_f_W_*_signal}]
    We know that 
    $$f(\mathbf{W}^{(t)}, \mathbf{x}_i)=(1 / m) \sum_{j, r} j \cdot[\sigma(\langle\mathbf{w}_{j, r}, y_i \cdot \mathbf{v}\rangle)+\sigma(\langle\mathbf{w}_{j, r}, \bm{\xi}_i\rangle)].$$
    Therefore, it holds that
    \begin{align*} 
    y_i\langle\nabla f(\mathbf{W}^{(t)}, \mathbf{x}_i), \mathbf{W}^*\rangle= & \frac{1}{m} \sum_{j, r} \sigma^{\prime}(\langle\mathbf{w}_{j, r}^{(t)}, y_i \mathbf{v}\rangle)\langle\mathbf{v}, j \mathbf{w}_{j, r}^*\rangle+\frac{1}{m} \sum_{j, r} \sigma^{\prime}(\langle\mathbf{w}_{j, r}^{(t)}, \bm{\xi}_i\rangle)\langle y_i \bm{\xi}_i, j \mathbf{w}_{j, r}^*\rangle \\ 
    \overset{(i)}{=} & \frac{1}{m} \sum_{j, r} \sigma^{\prime}(\langle\mathbf{w}_{j, r}^{(t)}, y_i \mathbf{v}\rangle) 2 q m \log (2 q / \kappa)+\frac{1}{m} \sum_{j, r} \sigma^{\prime}(\langle\mathbf{w}_{j, r}^{(t)}, y_i \mathbf{v}\rangle)\langle\mathbf{v}, j \mathbf{w}_{j, r}^{(0)}\rangle \\ 
    +& \frac{1}{m} \sum_{j, r} \sigma^{\prime}(\langle\mathbf{w}_{j, r}^{(t)}, \bm{\xi}_i\rangle)\langle y_i \bm{\xi}_i, j \mathbf{w}_{j, r}^{(0)}\rangle \\
     \overset{(ii)}{\geq}& \frac{1}{m} \sum_{j, r} \sigma^{\prime}(\langle\mathbf{w}_{j, r}^{(t)}, y_i \mathbf{v}\rangle) 2 q m \log (2 q / \kappa)-\frac{1}{m} \sum_{j, r} \sigma^{\prime}(\langle\mathbf{w}_{j, r}^{(t)}, y_i \mathbf{v}\rangle) \widetilde{O}(\sigma_0\|\mathbf{v}\|_2) \\
    -& \frac{1}{m} \sum_{j, r} \sigma^{\prime}(\langle\mathbf{w}_{j, r}^{(t)}, \bm{\xi}_i\rangle) \widetilde{O}(\sigma_0 \sigma_{\xi} \sqrt{d}),    
    \end{align*}
    where $(i)$ holds due to the definition of $\mathbf{w}^*$ and $(ii)$ follows from \cref{lem:w0_v_xi}. Moreover, according to \cref{lem:w_t-w_0_v}, we have that for $j = y_i$:
    $$
    \max _r\{\langle\mathbf{w}_{j, r}^{(t)}, y_i \mathbf{v}\rangle\}=\max _r\{\Gamma_{j, r}^{(t)}+\langle\mathbf{w}_{j, r}^{(0)}, y_i \mathbf{v}\rangle - \eta \sum_{s=1}^t\langle\mathbf{z}_s, y_i\mathbf{v}\rangle\} \overset{(i)}{\geq} 2-\widetilde{O}(\sigma_0\|\mathbf{v}\|_2) \overset{(ii)}{\geq} 1 .
    $$
    Here, $(i)$ holds due to the analysis in \cref{lem:app_signal_first_stage} and $(ii)$ comes from $\sigma_0 \leq \widetilde{O}(n^{-1/2}\|\mathbf{v}\|_2) $. Additionally, we can also have
    \begin{align*}
        & |\langle\mathbf{w}_{j, r}^{(t)}, \mathbf{v}\rangle| \stackrel{(i)}{\leq}|\langle\mathbf{w}_{j, r}^{(0)}, \mathbf{v}\rangle|+|\Gamma_{j, r}^{(t)}| + |\eta \sum_{s=1}^t\langle\mathbf{z}_s, \mathbf{v}\rangle\ | \stackrel{(i i)}{\leq} \widetilde{O}(1) \\
        & |\langle\mathbf{w}_{j, r}^{(t)}, \bm{\xi}_i\rangle| \stackrel{(i i i)}{\leq}|\langle\mathbf{w}_{j, r}^{(0)}, \bm{\xi}_i\rangle|+|\underline{\Phi}_{j, r, i}^{(t)}|+|\bar{\Phi}_{j, r, i}^{(t)}|+8 n \sqrt{\frac{\log (4 n^2 / \delta)}{d}} \alpha + |\eta \sum_{s=1}^t\langle\mathbf{z}_s, \bm{\xi}_i\rangle\ | \stackrel{(i v)}{\leq} \widetilde{O}(1).
        \end{align*}
    Here, $(i)$ holds due to \cref{lem:w_t-w_0_v} and \cref{lem:xi_zt2}; $(ii)$ is given by \cref{lem:w_t-w_0_xi}; $(ii)$ and $(iv)$ follows from \cref{pro:bounds_of_coeeficients} and \cref{lem:xi_zt2}. Combining these results, we return to $y_i\langle\nabla f(\mathbf{W}^{(t)}, \mathbf{x}_i), \mathbf{W}^*\rangle$, which gives:
    $$
    y_i\langle\nabla f(\mathbf{W}^{(t)}, \mathbf{x}_i), \mathbf{W}^*\rangle \geq 2 q \log (2 q / \kappa)-\widetilde{O}(\sigma_0\|\mathbf{v}\|_2)-\widetilde{O}(\sigma_0 \sigma_{\xi} \sqrt{d}) \geq q \log (2 q / \kappa),
    $$
    where the last inequality is driven by the conditions of $\sigma_0$ as signal learning.
\end{proof}

\begin{lemma}\label{lem:signal_wt-w*-}
    Under the same conditions as signal learning, it holds that
    $$
    \|\mathbf{W}^{(t)}-\mathbf{W}^*\|_F^2-\|\mathbf{W}^{(t+1)}-\mathbf{W}^*\|_F^2 \geq(2 q-1) \eta L_D(\mathbf{W}^{(t)})-\eta \kappa -\eta^2 \widetilde{O}(d\sigma_z^2) - \eta \widetilde{O}(\sigma_z m^{3/2} \|\mathbf{v}\|_2^{-1})
    $$
for all $T_1 \leq t \leq T^*$.
\end{lemma}

\begin{proof}[\bf Proof of \cref{lem:signal_wt-w*-}]
According to the optimization properties, we know the first equality holds:
    \begin{equation}
    \begin{aligned}
     \|\mathbf{W}^{(t)}&-\mathbf{W}^*\|_F^2-\|\mathbf{W}^{(t+1)}-\mathbf{W}^*\|_F^2 \\
    & =2 \eta\langle\nabla L_S(\mathbf{W}^{(t)}), \mathbf{W}^{(t)}-\mathbf{W}^*\rangle-\eta^2\|\nabla L_S(\mathbf{W}^{(t)})\|_F^2 \\
    & \overset{(i)}{=}\frac{2 \eta}{n} \sum_{i=1}^n \ell_i^{\prime(t)}[q y_i f(\mathbf{W}^{(t)}, \mathbf{x}_i)-\langle\nabla f(\mathbf{W}^{(t)}, \mathbf{x}_i), \mathbf{W}^*\rangle] + {\eta}{\langle \mathbf{z}_t, \mathbf{W}^{(t)}-\mathbf{W}^*\rangle}\\
    & -\eta^2(O(\max \{\|\mathbf{v}\|_2^2, \sigma_{\xi}^2 d\}) L_D(\mathbf{W}^{(t)}) + O(\sigma_{z}^2 d \log (1 / \delta))) \\
    & \overset{(ii)}{\geq} \frac{2 \eta}{n} \sum_{i=1}^n \ell_i^{\prime(t)}[q y_i f(\mathbf{W}^{(t)}, \mathbf{x}_i)-q \log (2 q / \kappa)] \\
    &  + {\eta}{\langle \mathbf{z}_t, \mathbf{W}^{(t)}-\mathbf{W}^*\rangle} -\eta^2(O(\max \{\|\mathbf{v}\|_2^2, \sigma_{\xi}^2 d\}) L_D(\mathbf{W}^{(t)}) + O(\sigma_{z}^2 d \log (1 / \delta))) \\
    & \overset{(iii)}{\geq}  \frac{2 q \eta}{n} \sum_{i=1}^n[\ell(y_i f(\mathbf{W}^{(t)}, \mathbf{x}_i))-\kappa /(2 q)]\\
    & - \eta\widetilde{O}(\sigma_z m^{3/2} \|\mathbf{v}\|_2^{-1}) -\eta^2(O(\max \{\|\mathbf{v}\|_2^2, \sigma_{\xi}^2 d\}) L_D(\mathbf{W}^{(t)}) + O(\sigma_{z}^2 d \log (1 / \delta))) \\
    & \overset{(iv)}{\geq} (2 q-1) \eta L_D(\mathbf{W}^{(t)})-\eta \kappa -\eta^2 \widetilde{O}(d\sigma_z^2) - \eta \widetilde{O}(\sigma_z m^{3/2} \|\mathbf{v}\|_2^{-1}).
    \end{aligned}
    \end{equation}
    Here, $(i)$ holds due to the definition of noisy gradient, the neural network is $q$ homogeneous, and \cref{lem:loss_F}; $(ii)$ is driven from \cref{lem:nabla_f_W_*_signal}; $(iii)$ is due to the convexity of the cross entropy function; $(iv)$ comes from definition of $L_D$.
\end{proof}

\begin{lemma}[Restatement of \cref{thm:main_signal_loss}]\label{thm:main_signal_loss_app}
    Let $T, T_1$ be defined in respectively. Then under the same conditions as signal learning, for any $t \in[T_1, T]$, it holds that $|\Gamma_{j, r}^{(t)}| \leq \sigma_0 \|\mathbf{v}\|_2$ for all $j \in\{ \pm 1\}$ and $r \in[m]$. Moreover, let $\mathbf{W}^*$ be the collection of CNN parameters with convolution filters $\mathbf{w}_{j, r}^*=\mathbf{w}_{j, r}^{(0)}+2 q m \log (2 q / \kappa) \cdot j \cdot\|\mathbf{v}\|_2^{-2} \cdot \mathbf{v}$. Then the following bound holds
    \begin{align*}
        \frac{1}{t-T_1+1} \sum_{s=T_1}^t L_D(\mathbf{W}^{(s)}) &\leq \frac{\|\mathbf{W}^{(T_1)}-\mathbf{W}^*\|_F^2}{(2 q-1) \eta(t-T_1+1)}  +\frac{\kappa}{(2 q-1)} + \underbrace{\frac{\eta d \sigma_z^2 + \widetilde{O}(\sigma_z m^{3/2} \|\mathbf{v}\|_2^{-1})}{(2 q-1)} }_{\text{Private terms}}
    \end{align*}
    for all $t \in[T_1, T]$, where we denote $\|\mathbf{W}\|_F=\sqrt{\|\mathbf{W}_{+1}\|_F^2+\|\mathbf{W}_{-1}\|_F^2}$.
\end{lemma}

\begin{proof}[\bf Proof of \cref{thm:main_signal_loss_app}]
    According to \cref{lem:signal_wt-w*-}, we have, for any $t\leq T$:
    $$
        \|\mathbf{W}^{(t)}-\mathbf{W}^*\|_F^2-\|\mathbf{W}^{(t+1)}-\mathbf{W}^*\|_F^2 \geq(2 q-1) \eta L_D(\mathbf{W}^{(t)})-\eta \kappa -\eta^2 \widetilde{O}(d\sigma_z^2) - \eta \widetilde{O}(\sigma_z m^{3/2} \|\mathbf{v}\|_2^{-1}).
    $$
    By summing over all terms and dividing $t-T_1-1$ on both sides, we obtain:
    \begin{align*}
        \frac{1}{t-T_1+1} \sum_{s=T_1}^t L_D(\mathbf{W}^{(s)}) \leq \frac{\|\mathbf{W}^{(T_1)}-\mathbf{W}^*\|_F^2}{(2 q-1) \eta(t-T_1+1)}  +\frac{\kappa}{(2 q-1)} + \frac{\eta d \sigma_z^2 + \widetilde{O}(\sigma_z m^{3/2} \|\mathbf{v}\|_2^{-1})}{(2 q-1)}. 
    \end{align*}
    If we have  $T=T_1+\lfloor\frac{\|\mathbf{W}^{(T_1)}-\mathbf{W}^*\|_E^2}{2 \eta \kappa}\rfloor = \frac{C m n \varepsilon}{\eta \mu(\|\mathbf{v}\|_2+\|\boldsymbol{\xi}\|_2)} \geq \kappa^{-1}$, then it holds that
    $$
    \frac{\|\mathbf{W}^{(T_1)}-\mathbf{W}^*\|_F^2}{(2 q-1) \eta(T-T_1+1)}+\frac{\kappa}{2 q-1} \leq \frac{3 \kappa}{2 q-1},
    $$
    and with $\sigma_z=\frac{1}{\eta \mu \sqrt{T}}$:
    $$
    \frac{\eta d \sigma_z^2 + \widetilde{O}(\sigma_z m^{3/2} \|\mathbf{v}\|_2^{-1})}{(2 q-1)} \leq \frac{d}{\eta \mu^2 T(2q-1)} + \frac{m^{3/2} \|\mathbf{v}\|_2^{-1}}{\eta \mu \sqrt{T}(2q-1)} \overset{(i)}{\leq} \frac{\kappa}{(2q-1)},
    $$
    where $(i)$ comes from the assumption of $\eta$. 
    Therefore, combining above results, we conclude that 
    \begin{align*}
        \frac{1}{t-T_1+1} \sum_{s=T_1}^t L_D(\mathbf{W}^{(s)}) \leq \kappa. 
    \end{align*}
    Next, we will use induction to prove that $\Psi_t = \max_{i,j,t} |\Phi_{i,j,r}^t| \leq 2\sigma_0\|\bm{\xi}\|_2$ holds for all $t\in [T_1,T]$. According to \cref{lem:app_signal_first_stage}, we know it holds for $T_1$. Now, assume it holds for some $t \in [T_1, T)$, and we will show that it also holds for $t+1$.
    \begin{align*}
    \Psi^{(t+1)} & \overset{(i)}{\leq} \Psi^{(t)}+\max _{j, r, i}\{\frac{\eta}{n m} \cdot|\ell_i^{(t)}| \cdot \sigma^{\prime}(\langle\mathbf{w}_{j, r}^{(0)}, \boldsymbol{\xi}_i\rangle+2 \sum_{i^{\prime}=1}^n \Psi^{(t)} \cdot \frac{|\langle\boldsymbol{\xi}_{i^{\prime}}, \boldsymbol{\xi}_i\rangle|}{\|\boldsymbol{\xi}_{i^{\prime}}\|_2^2} -\eta \sum_{s=1}^t \langle\mathbf{z}_s, \bm{\xi}_i \rangle) \cdot\|\boldsymbol{\xi}_{i^{\prime}}\|_2^2\} \\
    \overset{(i)}{=} & \Psi^{(t)}+\max _{j, r, i}\{\frac{\eta}{n m} \cdot|\ell_i^{\prime(t)}| \cdot \sigma^{\prime}(\langle\mathbf{w}_{j, r}^{(0)}, \boldsymbol{\xi}_i\rangle+2 \Psi^{(t)}+2 \sum_{i^{\prime} \neq i}^n \Psi^{(t)} \cdot \frac{|\langle\boldsymbol{\xi}_{i^{\prime}}, \boldsymbol{\xi}_i\rangle|}{\|\boldsymbol{\xi}_{i^{\prime}}\|_2^2} -\eta \sum_{s=1}^t \langle\mathbf{z}_s, \bm{\xi}_i \rangle) \cdot\|\boldsymbol{\xi}_{i^{\prime}}\|_2^2, \\
    \overset{(ii)}{\leq} & \Psi^{(t)}+\frac{\eta q}{n m} \cdot \max _i|\ell_i^{(t)}| \cdot[4 \cdot \sqrt{\log (8 m n / \delta)} \cdot \sigma_0 \sigma_{\xi} \sqrt{d}+(2+\frac{4 n \sigma_{\xi}^2 \cdot \sqrt{d \log (4 n^2 / \delta)}}{\sigma_{\xi}^2 d / 2}) \cdot \Psi^{(t)}]^{q-1} \cdot 2 \sigma_{\xi}^2 d \\
    \overset{(iii)}{\leq} & \Psi^{(t)}+\frac{\eta q}{n m} \cdot \max _i|\ell_i^{(t)}| \cdot(4 \cdot \sqrt{\log (8 m n / \delta)} \cdot \sigma_0 \sigma_{\xi} \sqrt{d}+4 \cdot \Psi^{(t)})^{q-1} \cdot 2 \sigma_{\xi}^2 d.
    \end{align*}
    Here, $(i)$ holds due to \cref{def:decom_coefficient}; $(ii)$ comes from \cref{lem:w0_v_xi}, \cref{lem:xi_bound} and the choice of $T$. $(iii)$ is due to the condition of $d$. By summing the above over $t$ , we have:
    \begin{align*}
     \Psi^{(t)} &\stackrel{(i)}{\leq} \Psi^{(T_1)}+\frac{\eta q}{n m} \sum_{s=T_1}^{t-1} \max _i|\ell_i^{(s)}| \widetilde{O}(\sigma_{\xi}^2 d) (\sigma_0\|\bm{\xi}\|_2)^{q-1} \\
    & \stackrel{(i i)}{\leq} \Psi^{(T_1)}+\frac{\eta q}{n m} \widetilde{O}(\sigma_{\xi}^2 d) (\sigma_0\|\bm{\xi}\|_2)^{q-1} \sum_{s=T_1}^{t-1} \max _i \ell_i^{(s)} \\
    &  \stackrel{(i i i)}{\leq} \Psi^{(T_1)}+\widetilde{O}(\eta m^{-1} \sigma_{\xi}^2 d) (\sigma_0\|\bm{\xi}\|_2)^{q-1} \sum_{s=T_1}^{t-1} L_S(\mathbf{W}^{(s)}) \\
    &  \stackrel{(i v)}{\leq} \Psi^{(T_1)}+\widetilde{O}(m^2 \operatorname{SNR}^{-2}) (\sigma_0\|\bm{\xi}\|_2)^{q-1} \\
    &  \stackrel{(v)}{\leq} (\sigma_0\|\bm{\xi}\|_2)+\widetilde{O}(m^2 (n\varepsilon)^2 (n\varepsilon)^{-q-3-2/q}) (\sigma_0\|\bm{\xi}\|_2) \\
    & \stackrel{( vi)}{\leq} 2 (\sigma_0\|\bm{\xi}\|_2).
    \end{align*}
    Here, $(i)$ holds due to induction hypothesis; $(ii)$ is by $|\ell^{\prime}| \leq \ell$, $(iii)$ comes from $\max _i \ell_i^{(s)} \leq \sum_i \ell_i^{(s)}=n L_D(\mathbf{W}^{(s)})$; $({iv})$ is due to $\sum_{s=T_1}^{t-1} L_D(\mathbf{W}^{(s)}) \leq \sum_{s=T_1}^T L_D(\mathbf{W}^{(s)})=\widetilde{O}(\eta^{-1} m^3\|\mathbf{v}\|_2^2)$ from the choice of $T$; $(v)$ is by the conditions of $\sigma_0$ and SNR; $(vi)$ is due to $(n\varepsilon)^{q+1} \geq m$.
\end{proof}

Now we consider the generalization performance of the privately trained model. Given a new data point $(\mathbf{x}, y)$ drawn from the distribution defined in \cref{def:data_dis}, we assume $\mathbf{x}=[y \mathbf{v}, \boldsymbol{\xi}]$ without loss of generality.

\begin{lemma}\label{lem:max_w_xi}
     Under the same conditions as signal learning, we have that $\max _{j, r}|\langle\mathbf{w}_{j, r}^{(t)}, \boldsymbol{\xi}_i\rangle| \leq 1 / 2$ for all $0 \leq t \leq T$.
\end{lemma}

\begin{proof}[\bf Proof of \cref{lem:max_w_xi}]
    According to the signal-noise decomposition, the private model satisfies:
    \begin{align}\nonumber
        \mathbf{w}_{j, r}^{(t)}&=\mathbf{w}_{j, r}^{(0)}+j \cdot \Gamma_{j, r}^{(t)} \cdot\|\mathbf{v}\|_2^{-2} \cdot \mathbf{v}+\sum_{i=1}^n \bar{\Phi}_{j, r, i}^{(t)} \cdot\|\bm{\xi}_i\|_2^{-2} \cdot \bm{\xi}_i +\sum_{i=1}^n \underline{\Phi}_{j, r, i}^{(t)} \cdot\|\bm{\xi}_i\|_2^{-2} \cdot \bm{\xi}_i - \eta \sum_{s=1}^{t} \mathbf{z}_s.
    \end{align}
    Then, we have
    \begin{align*}
    |\langle\mathbf{w}_{j, r}^{(t)}, \boldsymbol{\xi}_i\rangle| & \stackrel{(i)}{\leq}|\langle\mathbf{w}_{j, r}^{(0)}, \boldsymbol{\xi}_i\rangle|+|\underline{\Phi}_{j, r, i}^{(t)}|+|\bar{\Phi}_{j, r, i}^{(t)}|+8 n \sqrt{\frac{\log (4 n^2 / \delta)}{d}} \alpha +0.1 \\ 
    & \stackrel{(i i)}{\leq} 2 \sqrt{\log (8 m n / \delta)} \cdot \sigma_0 \sigma_{\xi} \sqrt{d}+\sigma_0 \sigma_{\xi} \sqrt{d}+8 n \sqrt{\frac{\log (4 n^2 / \delta)}{d}} \alpha + 0.1\\ 
    & \stackrel{(i i i)}{\leq} 1 / 2 .
    \end{align*}
    Here, $(i)$ holds due to \cref{lem:xi_zt2} and the assumption of training iterations; $(ii)$ comes from \cref{thm:main_noise_loss_app}; $(iii)$ is driven from the condition of $\sigma_0$ and the assumptions of $n\varepsilon, d$.
\end{proof}

\begin{lemma}\label{lem:prob_w_xi}
    Under the same assumptions as Theorem 4.3, with probability at least  $1 -4 m T \exp (-C_1^{-1} \sigma_0^{-2} \sigma_{\xi}^{-2} d^{-1})$, we have $\max _{j, r}|\langle\mathbf{w}_{j, r}^{(t)}, \boldsymbol{\xi}\rangle| \leq 1 / 2$ for all $0 \leq t \leq T$, where $C_1=$ $\widetilde{O}(1)$.
\end{lemma}

\begin{proof}[\bf Proof of \cref{lem:prob_w_xi}]
    Define $\widetilde{\mathbf{w}}_{j, r}^{(t)}=\mathbf{w}_{j, r}^{(t)}-j \cdot \Gamma_{j, r}^{(t)} \cdot \frac{\mathbf{v}}{\|\mathbf{v}\|_2^2}$. It follows that $\langle\widetilde{\mathbf{w}}_{j, r}^{(t)}, \boldsymbol{\xi}\rangle=\langle\mathbf{w}_{j, r}^{(t)}, \boldsymbol{\xi}\rangle$.
    Additionally, we have:
    $$
    \|\widetilde{\mathbf{w}}_{j, r}^{(t)}\|_2 \leq \widetilde{O}(\sigma_0 \sqrt{d}+n \sigma_0 + \sigma_0\sigma_{\xi} \sqrt{d})=\widetilde{O}(\sigma_0 \sqrt{d}),
    $$
    where the equality holds due to the condition $d \geq \widetilde{\Omega}(m^2 n^4)$ and the analysis in \cref{lem:signal_first_stage}.
    Thus, we know that $\max _{j, r}\|\widetilde{\mathbf{w}}_{j, r}^{(t)}\|_2 \leq C_1 \sigma_0 \sqrt{d}$, where $C_1=\widetilde{O}(1)$. Since $\langle\widetilde{\mathbf{w}}_{j, r}^{(t)}, \boldsymbol{\xi}\rangle$ follows a Gaussian distribution with mean zero and standard deviation bounded by $C_1 \sigma_0 \sigma_{\xi} \sqrt{d}$, the probability of deviation can be bounded as:
    $$\mathbb{P}(|\langle\widetilde{\mathbf{w}}_{j, r}^{(t)}, \boldsymbol{\xi}\rangle| \geq 1 / 2) \leq 2 \exp (-\frac{1}{8 C_1^2 \sigma_0^2 \sigma_{\xi}^2 d})
    $$
    By applying a union bound over all indices $j, r$, and $t$, the proof is complete.
\end{proof}

\begin{lemma}[Restatement of \cref{coro:popu_sign}]
    Under the same conditions as above, for any $t \leq T$ with $L_D(\mathbf{W}^{(t)}) \leq \kappa$, with at least probability $1-1/d$, it holds that $L_{\mathcal{D}}(\mathbf{W}^{(t)}) \leq 6 \kappa+\exp (\varepsilon^{-2/q})$.
\end{lemma}

\begin{proof}[\bf Proof of \cref{coro:popu_sign}]
    Let $\mathcal{K}$ represent the event where \cref{lem:prob_w_xi} holds. We can partition $L_{\mathcal{D}}(\mathbf{W}^{(t)})$ into two components:
    $$
    \mathbb{E}[\ell(y f(\mathbf{W}^{(t)}, \mathbf{x}))]=\underbrace{\mathbb{E}[\mathbf{1}(\mathcal{K}) \ell(y f(\mathbf{W}^{(t)}, \mathbf{x}))]}_{I_1}+\underbrace{\mathbb{E}[\mathbf{1}(\mathcal{K}^c) \ell(y f(\mathbf{W}^{(t)}, \mathbf{x}))]}_{I_2}
    $$
    We will now bound $I_1$ and $I_2$ separately. The term $I_1$ can be bounded by $6 L_D(\mathbf{W}^{t}) \leq \kappa$ according to Lemma D.8 in \cite{cao2022benign}.
    Next, we bound the second term $I_2$. We select an arbitrary training data point ( $\mathbf{x}_{i^{\prime}}, y_{i^{\prime}}$ ) such that $y_{i^{\prime}}=y$. Then, we have:
    $$
    \begin{aligned}
    \ell(y f(\mathbf{W}^{(t)}, \mathbf{x})) & \leq \log (1+\exp (F_{-y}(\mathbf{W}^{(t)}, \mathbf{x}))) \\
    & \overset{(i)}{\leq} 1+F_{-y}(\mathbf{W}^{(t)}, \mathbf{x}) \\
    & \overset{(ii)}{=}1+\frac{1}{m} \sum_{j=-y, r \in[m]} \sigma(\langle\mathbf{w}_{j, r}^{(t)}, y \mathbf{v}\rangle)+\frac{1}{m} \sum_{j=-y, r \in[m]} \sigma(\langle\mathbf{w}_{j, r}^{(t)}, \boldsymbol{\xi}\rangle) \\
    &\overset{(iii)}{\leq} 1+F_{-y_{i^{\prime}}}(\mathbf{W}_{-y_{i^{\prime}}}, \mathbf{x}_{i^{\prime}})+\frac{1}{m} \sum_{j=-y, r \in[m]} \sigma(\langle\mathbf{w}_{j, r}^{(t)}, \boldsymbol{\xi}\rangle) \\
    & \overset{(v)}{\leq} 2+\frac{1}{m} \sum_{j=-y, r \in[m]} \sigma(\langle\mathbf{w}_{j, r}^{(t)}, \boldsymbol{\xi}\rangle) \\
    & \overset{(iv)}{\leq} 2+\widetilde{O}((\sigma_0 \sqrt{d})^q)\|\boldsymbol{\xi}\|^q.
    \end{aligned}
    $$
    Here, $(i)$ is due to $F_y(\mathbf{W}^{(t)}, \mathbf{x}) \geq 0$; $(ii)$ follows from the property of the logarithmic function; $(iii)$ holds due to $$\frac{1}{m} \sum_{j=-y, r \in[m]} \sigma(\langle\mathbf{w}_{j, r}^{(t)}, y \mathbf{v}\rangle) \leq F_{-y}(\mathbf{W}_{-y}, \mathbf{x}_{i^{\prime}})=F_{-y_{i^{\prime}}}(\mathbf{W}_{-y_{i^{\prime}}}, \mathbf{x}_{i^{\prime}});$$ $(v)$ is by \cref{lem:jneqy_F_1}; $(iv)$ comes from \cref{lem:prob_w_xi} that $ \|\widetilde{\mathbf{w}}_{j, r}^{(t)}\|_2 \leq \widetilde{O}(\sigma_0 \sqrt{d})$.
    Therefore, we can bound term 2 as follows:
    \begin{align*}
    I_2 & \overset{(i)}{\leq} \sqrt{\mathbb{E}[\mathbbm{1}(\mathcal{K}^c)]} \cdot \sqrt{\mathbb{E}[\ell(y f(\mathbf{W}^{(t)}, \mathbf{x}))^2]} \\
    & \overset{(ii)}{\leq} \sqrt{\mathbb{P}(\mathcal{K}^c)} \cdot \sqrt{4+\widetilde{O}((\sigma_0 \sqrt{d})^{2 q}) \mathbb{E}[\|\boldsymbol{\xi}\|_2^{2 q}]} \\
    & \overset{(iii)}{\leq} \exp [-\widetilde{\Omega}(\sigma_0^{-2} \sigma_{\xi}^{-2} d^{-1})+\operatorname{poly} \log (d)] \\
    & \overset{(v)}{\leq} \exp ({(n\varepsilon)}^{-1-1/q}).
    \end{align*}
    Here, $(i)$ holds due to Cauchy-Schwartz inequality; $(ii)$ and $(iii)$ come from the fact $\sqrt{4+\widetilde{O}((\sigma_0 \sqrt{d})^{2 q}) \mathbb{E}[\|\boldsymbol{\xi}\|_2^{2 q}]}=O(\operatorname{poly}(d))$ and \cref{lem:prob_w_xi}; $(v)$ is by the condition of $\sigma_0 \leq (n\varepsilon)^{-1-1/q} \|\mathbf{\xi}\|_2^{-1}$. This completes the proof.  
\end{proof}

\section{Noise Memorization}

\begin{lemma}\label{lem:sigma0_geq}
    Under the same conditions as noise memorization, then it holds that $\bar{\beta} \geq \sigma_0 \sigma_{\xi} \sqrt{d} / 4 \geq 20 n \sqrt{\frac{\log (4 n^2 / \delta)}{d}} \alpha$, if we have
    $
    \sigma_0 \geq 80 n \sqrt{\frac{\log (4 n^2 / \delta)}{d}} \alpha \cdot \min \{(\sigma_{\xi} \sqrt{d})^{-1},\|\mathbf{v}\|_2^{-1}\}.
    $

\end{lemma}

\begin{proof}[\bf Proof of \cref{lem:sigma0_geq}]
   Given the SNR condition in noise memorization $\sigma_{\xi}^q(\sqrt{d})^q \geq \widetilde{\Omega}(n\|\mathbf{v}\|_2^q)$, it follows that: $\sigma_{\xi} \sqrt{d} \geq\|\mathbf{v}\|_2$.
    Thus, we have:
    $$
    \begin{aligned}
    \bar{\beta}  \geq \frac{\sigma_0 \sigma_{\xi} \sqrt{d}}{4}  =\frac{\sigma_0}{4} \cdot \max \{\sigma_{\xi} \sqrt{d},\|\mathbf{v}\|_2\}  \geq 20 n \sqrt{\frac{\log (4 n^2 / \delta)}{d}} \alpha.
    \end{aligned}
    $$
where the first inequality follows from and the last inequality is a result of the lower bound condition on $\sigma_0$ stated in noise memorization. 
    
\end{proof}

\subsection{First Stage}

\begin{lemma}[Restatement of \cref{lem:noise_first_stage}]\label{lem:noise_first_stage_app}
    Under the same conditions as noise memorization, in particular, if we choose
\begin{equation}\label{eq:snr_noise}
    n^{-1} \operatorname{SNR}^{-q} \geq \frac{C 2^{q+2} \log (20 /(\sigma_0 \sigma_{\xi} \sqrt{d}))(\sqrt{2 \log (8 m / \delta)})^{q-2}}{0.15^{q-2}}, \frac{ \varepsilon}{ (1+\operatorname{SNR})} \geq \frac{C \log(10/\sigma_0\sigma_{\xi}\sqrt{d})}{0.15^{q-2}q}
\end{equation}
where $C=O(1)$ is a positive constant, then there exist
$$
T_1=\frac{C \log (10 /(\sigma_0 \sigma_{\xi} \sqrt{d})) 4 m n}{0.15^{q-2} \eta q \sigma_0^{q-2}(\sigma_{\xi}^2 \sqrt{d})^q}
$$
such that
\begin{itemize}
    \item $\max _{j, r} \bar{\Phi}_{j, r, i}^{(T_1)} \geq 2$ for all $i \in[n]$.
    \item $\max _{j, r} \Gamma_{j, r}^{(t)}=\widetilde{O}(\sigma_0\|\mathbf{v}\|_2)$ for all $0 \leq t \leq T_1$.
    \item $\max _{j, r, i}|\underline{\Phi}_{j, r, i}^{(t)}|=\widetilde{O}(\sigma_0 \sigma_{\xi} \sqrt{d})$ for all $0 \leq t \leq T_1$.
\end{itemize}
\end{lemma}

\begin{proof}[\bf Proof of \cref{lem:noise_first_stage_app}]
    First, let $T_1^{+}=\min \{\frac{m}{\eta q 2^{q-1}(\sqrt{2 \log (8 m / \delta)})^{q-2} \sigma_0^{q-2}\|\mathbf{v}\|_2^q}, \frac{\sigma_0 m n \varepsilon}{\eta (\|\mathbf{v}\|_2 + \|\bm{\xi}\|_2)} \}$. According to the proof of \cref{pro:bounds_of_coeeficients}, it  follows that $\underline{\Phi}_{j, r, i}^{(t)} \geq-\beta-16 n \sqrt{\frac{\log (4 n^2 / \delta)}{d}} \alpha - 0.2 $. Notably, the constant $0.2$ here is chosen for simplicity in the proof and can be replaced with any value. For example, if we take $\underline{\Phi}_{j, r, i}^{(t)} \geq-\beta-16 n \sqrt{\frac{\log (4 n^2 / \delta)}{d}} \alpha - \widetilde{O}(\sigma_0 \sigma_{\xi}\sqrt{d}) $, the proof of \cref{pro:bounds_of_coeeficients} still holds, provided that $T \leq \frac{\sigma_0 mn \varepsilon}{\eta(\|\mathbf{v}\|_2 + \|\bm{\xi}\|_2)}$. Moreover, we have $\Phi_{j, r, i}^{(t)} \leq 0$ and $\bar{\beta} \leq \beta=\widetilde{O}(\sigma_0 \sigma_{\xi} \sqrt{d})$. Therefore, $\max _{j, r, i}|\underline{\Phi}_{j, r, i}^{(t)}|=\widetilde{O}(\sigma_0 \sigma_{\xi} \sqrt{d})$.

    Now, we proceed to prove the dynamics of $\Gamma_{j, r}^{(t+1)}$. Similar to the signal learning, we define $A^{(t)}=\max _{j, r}\{\Gamma_{j, r}^{(t)}+|\langle\mathbf{w}_{j, r}^{(0)}, \mathbf{v}\rangle| - y_i \eta \sum_{s=1}^t\langle\mathbf{z}_s, \mathbf{v}\rangle \}$, then it holds that
    \begin{equation} \label{eq:noise_A_t}
    \begin{aligned}
        \Gamma_{j, r}^{(t+1)} & =\Gamma_{j, r}^{(t)}-\frac{\eta}{n m} \cdot \sum_{i=1}^n \ell_i^{\prime (t)} \cdot \sigma^{\prime}(\langle\mathbf{w}_{j, r}^{(t)}, y_i \cdot \mathbf{v}\rangle)\|\mathbf{v}\|_2^2 \\ 
        & \leq \Gamma_{j, r}^{(t)}+\frac{\eta}{n m} \cdot \sum_{i=1}^n \sigma^{\prime}(|\langle\mathbf{w}_{j, r}^{(0)}, \mathbf{v}\rangle|+\Gamma_{j, r}^{(t)} - \eta y_i \sum_{s=1}^t\langle\mathbf{z}_s, \mathbf{v}\rangle)\|\mathbf{v}\|_2^2 \\
        A^{(t+1)} & \leq A^{(t)} +\frac{\eta q \|\mathbf{v}\|_2^2}{m}[A^{(t)}]^{q-1} + \eta y_i \langle\mathbf{z}_{t+1}, \mathbf{v}\rangle.
    \end{aligned}
    \end{equation}
    Next, we will prove $A^{(0)} \leq 3 A^{(0)}$ for $t \leq T_1^{+}$ by induction. First, $A^{(0)} \leq 3 A^{(0)}$ holds at $t =0$ due to the definition and we assume it holds for $t$. Now suppose that there exists some $t \leq T_1^{+}$ such that $A^{(s)} \leq 2 A^{(0)}$ holds $0 \leq s \leq t-1$. Applying a telescoping sum to \cref{eq:noise_A_t} yields:
    \begin{align*}
    A^{t} & \leq A^{(0)}+\sum_{s=0}^{t} \frac{\eta q\|\mathbf{v}\|_2^2}{m}[A^{(s)}]^{q-1} + \sum_{s=0}^{t} \eta y_i \langle\mathbf{z}_{s+1}, \mathbf{v}\rangle  \\
    & \overset{(i)}{\leq} A^{(0)}+\frac{\eta q\|\mathbf{v}\|_2^2 T_1^{+} 3^{q-1}}{m}[A^{(0)}]^{q-1} + \widetilde{O}(\sigma_0 \|\mathbf{v}\|_2 )\\
    & \overset{(ii)}{\leq} A^{(0)}+\frac{\eta q\|\mathbf{v}\|_2^2 T_1^{+} 3^{q-1}}{m}[\sqrt{2 \log (8 m / \delta)} \cdot \sigma_0\|\mathbf{v}\|_2]^{q-2} A^{(0)} +  A^{(0)} \\
    & \overset{(ii)}{\leq} 3 A^{(0)}.
    \end{align*}
    Here, $(i)$ holds due to the induction hypothesis and $T_1^{+} \leq \frac{\sigma_0 mn \varepsilon}{\eta (\|\mathbf{v}\|_2 + \|\bm{\xi}\|_2)}$; $(ii)$ follows \cref{lem:w0_v_xi} and $(iii)$ is derived from $T_1^{+}$. Moreover, we have $\max _{j, r} \Gamma_{j, r}^{(t)} \leq A^{(t)}+\max _{j, r}\{|\langle\mathbf{w}_{j, r}^{(0)}, \mathbf{v}\rangle| + y_i \eta \sum_{s=1}^t\langle\mathbf{z}_s, \mathbf{v}\rangle  \} \leq 5 A^{(0)} = \widetilde{O}(\sigma_0 \|\mathbf{v}\|_2 )$. 

    Now, we consider proving that the maximum of noise memorization is larger than $2$. For $y_i = j$, according to \cref{lem:w_t-w_0_xi}, we have:
    \begin{align*}
    \langle\mathbf{w}_{j, r}^{(t)}, \bm{\xi}_i\rangle & \geq\langle\mathbf{w}_{j, r}^{(0)}, \bm{\xi}_i\rangle+\bar{\Phi}_{j, r, i}^{(t)}-8 n \sqrt{\frac{\log (4 n^2 / \delta)}{d}} \alpha -\sum_{s=1}^{t}\langle \mathbf{z}_s, \bm{\xi}_i\rangle \\
    & \geq \bar{\Phi}_{j, r, i}^{(t)}+\langle\mathbf{w}_{j, r}^{(0)}, \bm{\xi}_i\rangle-0.4 \bar{\beta} -\sum_{s=1}^{t}\langle \mathbf{z}_s, \bm{\xi}_i\rangle.
    \end{align*}
    Similar to the proof of signal learning, let $B_i^{(t)}=$ $\max _{j=y_i, r}\{\bar{\Phi}_{j, r, i}^{(t)}+\langle\mathbf{w}_{j, r}^{(0)}, \bm{\xi}_i\rangle-0.4 \bar{\beta}-\sum_{s=1}^t\langle\mathbf{z}_s, \bm{\xi}_i\rangle\}$. For each $i$, let $T_1^{(i)}$ denote the first time in the period $[0, T_1^{+}]$such that $\bar{\Phi}_{j, r, i}^{(t)} \geq 2$.
    For $t \leq T_1^{(i)}$, it holds that $\max _{j, r}\{|\bar{\Phi}_{j, r, i}^{(t)}|,|\Phi_{j, r, i}^{(t)}|\}=O(1)$ and $\max _{j, r} \Gamma_{j, r}^{(t)} \leq 4 A^{(0)}=O(1)$. Therefore, by \cref{lem:F_1} and \cref{lem:jneqy_F_1}, we have $F_{-1}(\mathbf{W}^{(t)}, \mathbf{x}_i), F_{+1}(\mathbf{W}^{(t)}, \mathbf{x}_i)=$ $O(1)$. As a result, there exists a positive constant $C_1$ such that $-\ell_i^{\prime (t)} \geq C_1$ for all $0 \leq t \leq T_1^{(i)}$.
    Additionally, it is clear that $B_i^{(0)} \geq 0.6 \bar{\beta} \geq 0.15 \sigma_0 \sigma_{\xi}\sqrt{d}$. We can then analyze the dynamics of $B_i^{(t)}$.

    \begin{align*}
        B_i^{(t+1)} & {\geq} B_i^{(t)}+\frac{C_1 \eta q\|\bm{\xi}\|_2^2}{2 mn}[B_i^{(t)}]^{q-1} -\eta \langle\mathbf{z}_{t+1},  \bm{\xi}\rangle  \\
        & \overset{(i)}{\geq} (1+\frac{C_1 \eta q\|\bm{\xi}\|_2^2}{2 mn}[B_i^{(0)}]^{q-2}) B_i^{(t)} -\eta \langle\mathbf{z}_{t+1},  \bm{\xi}\rangle  \\
        & \overset{(ii)}{\geq} (1+\frac{C_1 0.15^{q-2} \eta q \sigma_0^{q-2}\|\bm{\xi}\|_2^q}{ mn}) B_i^{(t)} -\eta \langle\mathbf{z}_{t+1},  \bm{\xi}\rangle \\
        \text{Let} \quad q_{\Phi} = (1+\frac{C_1 0.15^{q-2} \eta q \sigma_0^{q-2}\|\bm{\xi}\|_2^q}{ mn}), \text{then} \quad & \overset{(iii)}{=} q_{\Phi}^t  B_i^{(0)} -(\frac{q_{\Phi}^t -1}{q_{\Phi}-1})\eta \langle\mathbf{z}_{t+1},  \bm{\xi}\rangle .
        \end{align*}
    Here, $(i)$ and $(ii)$ follow from that the facts $A^{(t)}$ is increasing and $\max _r\langle\mathbf{w}_{1, r}^{(0)}, \bm{\xi}\rangle -\eta \sum_{s=1}^t\langle\mathbf{z}_s,  \bm{\xi}\rangle \geq 0$; $(iii)$ is the summation of geometric series. Additionally, we know that $1+z \geq \operatorname{exp}(z/2)$ for $z\leq 2$ and $1+z \leq \operatorname{exp}(z)$ for $z\geq 0$, then the following inequality holds
    \begin{align*}
        B_i^{(t)} &\geq (1+\frac{C_1 \eta q 0.15^{q-2} \sigma_0^{q-2}\|\bm{\xi}\|_2^q}{4 mn})^t B_i^{(0)} - \frac{4 mn}{C_1 0.15^{q-2} q \sigma_0^{q-2}\|\bm{\xi}\|_2^q}(q_{\Phi}^t -1 ) \cdot \frac{ \|\bm{\xi}\|_2(\|\mathbf{v}\|_2+\|\bm{\xi}\|_2)}{mn\varepsilon}\\
        & \geq \exp (\frac{C_1 \eta q \sigma_0^{q-2} 0.15^{q-2}\|\bm{\xi}\|_2^q}{4 mn} t) \cdot {0.15 \sigma_0 \|\bm{\xi}\|_2} \\
        &-  \frac{4 mn}{C_1  q 0.15^{q-2} \sigma_0^{q-2}\|\bm{\xi}\|_2^q}\cdot \exp  (\frac{C_1 \eta q \sigma_0^{q-2}0.15^{q-2}\|\bm{\xi}\|_2^q}{4 mn} t) \cdot \frac{ \|\bm{\xi}\|_2(\|\mathbf{v}\|_2+\|\bm{\xi}\|_2)}{mn\varepsilon} \\
        & = ( {  0.15 \sigma_0  \|\bm{\xi}\|_2} - \frac{4(\|\mathbf{v}\|_2+\|\bm{\xi}\|_2) }{C_1  q \sigma_0^{q-2}\|\bm{\xi}\|_2^{q-1}  \varepsilon } ) \cdot \exp  (\frac{C_1 \eta q 0.15^{q-2} \sigma_0^{q-2}\|\bm{\xi}\|_2^q}{4 mn} t)\\
        & \geq  0.1 \sigma_0  \|\bm{\xi}\|_2 \cdot \exp  (\frac{C_1 \eta q 0.15^{q-2} \sigma_0^{q-2}\|\bm{\xi}\|_2^q}{4 mn} t),
    \end{align*}
    where the last inequality holds due to the choice of $\sigma_0$. Therefore, it is clear that $B_i^{(t)}$ will reach $3$ within $T_1=\frac{C \log (10 /(\sigma_0 \sigma_{\xi} \sqrt{d})) 4 m n}{0.15^{q-2} \eta q \sigma_0^{q-2}(\sigma_{\xi}^2 \sqrt{d})^q}$ iterations, which indicates that $\max _{j=y_i,r} \bar{\Phi}_{j, r,i}^{(t)}$ will reach $2$ within $T_1^(i)$ iterations. Moreover, we can verify that 
    $$ T_1=\frac{C \log (10 /(\sigma_0 \sigma_{\xi} \sqrt{d})) 4 m n}{0.15^{q-2} \eta q \sigma_0^{q-2}(\sigma_{\xi}^2 \sqrt{d})^q} \leq \eta^{-1} \sigma_0 m n \varepsilon(\|\mathbf{v}\|_2+\|\bm{\xi}\|_2)^{-1} = T_1^+,$$
    The inequality follows from the SNR condition in \cref{eq:snr_noise}. Hence, by the definition of $T_{1}^{(i)}$, $T_{1}^{(i)} \leq T_1 \leq T_1^{+} / 2$ holds.

\end{proof}

\subsection{Second Stage}

It is clear that we have the following results at the end of the first stage:
$$
\mathbf{w}_{j, r}^{(T_1)}=\mathbf{w}_{j, r}^{(0)}+j \cdot \Gamma_{j, r}^{(T_1)} \cdot \frac{\mathbf{v}}{\|\mathbf{v}\|_2^2}+\sum_{i=1}^n \bar{\Phi}_{j, r, i}^{(T_1)} \cdot \frac{\bm{\xi}_i}{\|\bm{\xi}_i\|_2^2}+\sum_{i=1}^n \underline{\Phi}_{j, r, i}^{(T_1)} \cdot \frac{\bm{\xi}_i}{\|\bm{\xi}_i\|_2^2}-\eta \sum_{s=1}^{T_1} \mathbf{z}_s
$$

Meanwhile, at the beginning of the second stage, we have the following results:
\begin{itemize}
    \item $\max _{j, r} \bar{\Phi}_{j, r, i}^{(T_1)} \geq 2$ for all $i \in[n]$.
    \item $\max _{j, r} \Gamma_{j, r}^{(t)}=\widetilde{O}(\sigma_0\|\mathbf{v}\|_2)$ for all $0 \leq t \leq T_1$.
    \item $\max _{j, r, i}|\Phi_{j, r, i}^{(t)}|=\widetilde{O}(\sigma_0 \sigma_{\xi} \sqrt{d})$ for all $0 \leq t \leq T_1$.
\end{itemize}

Based on \cref{lem:the_evolution_of_coefficient_main} and \cref{lem:xi_zt2_main}, we conclude that noise memorization $ \bar{\Phi}_{j, r, i}^{(T_1)}$ does not deteriorate over time. Specifically, for any $T_1 \leq t \leq T_p^*$, it holds that $\bar{\Phi}_{j, r, i}^{(t+1)} \geq \bar{\Phi}_{j, r, i}^{(t)}$, which implies $\max _{j, r} \bar{\Phi}_{j, r, i}^{(t)} \geq 2$.
If we consider $\mathbf{w}_{j, r}^*=\mathbf{w}_{j, r}^{(0)}+2 q m \log (2 q / \kappa))[\sum_{i=1}^n \mathbbm{1}(j=y_i) \cdot \frac{\bm{\xi}_i}{\|\bm{\xi}_i\|_2}]$, then we can derived:

\begin{lemma}\label{lem:W_T1-W_*_noise}
    Under the same conditions as noise memorization, we have that $\|\mathbf{W}^{(T_1)}-\mathbf{W}^*\|_F \leq \widetilde{O}(m^2 n^{1 / 2} \sigma_{\xi}^{-1} d^{-1 / 2})+ O(n m \sigma_0)$.
\end{lemma}

\begin{proof}[\bf Proof of \cref{lem:W_T1-W_*_noise}]
    According to triangle inequality, we have
    \begin{align*}
    \|\mathbf{W}^{(T_1)}-\mathbf{W}^*\|_F & \leq\|\mathbf{W}^{(T_1)}-\mathbf{W}^{(0)}\|_F+\|\mathbf{W}^{(0)}-\mathbf{W}^*\|_F \\
    & \overset{(i)}{\leq} \sum_{j, r} \frac{\Gamma_{j, r}^{(T_1)}}{\|\mathbf{v}\|_2}+\sum_{j, r, i} \frac{|\bar{\Phi}_{j, r, r}^{(T_1)}|}{\|\bm{\xi}_i\|_2}+\sum_{j, r, i} \frac{|\underline{\Phi}_{j, r, i}^{(T_1)}|}{\|\bm{\xi}_i\|_2} + \sum_{j, r} |\eta \sum_{s=1}^{T_1} \mathbf{z}_s| +O(m^{3 / 2} \log (1 / \kappa))\|\mathbf{v}\|_2^{-1} \\
    & \overset{(ii)}{\leq}\widetilde{O}(m\|\mathbf{v}\|^{-1})+\widetilde{O}(n \sqrt{m} \sigma_{\xi}^{-1}d^{-1/2}) +O( {m \sigma_0}) + O(m^{3 / 2} n^{1 / 2} \log (1 / \kappa) \sigma_{\xi}^{-1} d^{-1 / 2}) \\
    & \overset{(iii)}{\leq} \widetilde{O}(m^2 n^{1 / 2} \sigma_{\xi}^{-1} d^{-1 / 2}).
    \end{align*}
    Here, $(i)$ holds due to the decomposition of $\mathbf{w}$ in \cref{def:decom_coefficient_main} and the definition of $\mathbf{W}^*$; $(ii)$ follows from \cref{pro:bounds_of_coeeficients}, \cref{lem:noise_first_stage_app} and \cref{lem:xi_zt2}; $(iii)$ comes from the conditions of $\sigma_0$ in noise memorization.
\end{proof}

\begin{lemma}\label{lem:nabla_f_W_*_noise}
    Under the same conditions as noise memorization, we have that $y_i\langle\nabla f(\mathbf{W}^{(t)}, \mathbf{x}_i), \mathbf{W}^*\rangle \geq$ $q \log (2 q / \kappa)$ for all $i \in[n]$ and $T_1 \leq t \leq T^*$.
\end{lemma}

\begin{proof}[\bf Proof of \cref{lem:nabla_f_W_*_noise}]
    We know that 
    $$f(\mathbf{W}^{(t)}, \mathbf{x}_i)=(1 / m) \sum_{j, r} j \cdot[\sigma(\langle\mathbf{w}_{j, r}, y_i \cdot \mathbf{v}\rangle)+\sigma(\langle\mathbf{w}_{j, r}, \bm{\xi}_i\rangle)].$$
    Therefore, it holds that
    \begin{align*} 
    y_i\langle\nabla f(\mathbf{W}^{(t)}, \mathbf{x}_i), \mathbf{W}^*\rangle= & \frac{1}{m} \sum_{j, r} \sigma^{\prime}(\langle\mathbf{w}_{j, r}^{(t)}, y_i \mathbf{v}\rangle)\langle\mathbf{v}, j \mathbf{w}_{j, r}^*\rangle+\frac{1}{m} \sum_{j, r} \sigma^{\prime}(\langle\mathbf{w}_{j, r}^{(t)}, \bm{\xi}_i\rangle)\langle y_i \bm{\xi}_i, j \mathbf{w}_{j, r}^*\rangle \\
    \overset{(i)}{=} & \frac{1}{m}  \sum_{j, r} \sum_{i^{\prime}=1}^n \sigma^{\prime}(\langle\mathbf{w}_{j, r}^{(t)}, \bm{\xi}_i\rangle) 2 q m \log (2 q / \kappa) \mathbbm{1}(j=y_{i^{\prime}}) \cdot \frac{\langle\bm{\xi}_{i^{\prime}}, \bm{\xi}_i\rangle}{\|\bm{\xi}_{i^{\prime}}\|_2} \\ 
    & +\frac{1}{m} \sum_{j, r} \sigma^{\prime}(\langle\mathbf{w}_{j, r}^{(t)}, y_i \mathbf{v}\rangle)\langle\mathbf{v}, j \mathbf{w}_{j, r}^{(0)}\rangle+\frac{1}{m} \sum_{j, r} \sigma^{\prime}(\langle\mathbf{w}_{j, r}^{(t)}, \bm{\xi}_i\rangle)\langle y_i \bm{\xi}_i, j \mathbf{w}_{j, r}^{(0)}\rangle \\ 
     \overset{(ii)}{\geq}& \frac{1}{m} \sum_{j, r} \sigma^{\prime}(\langle\mathbf{w}_{j, r}^{(t)}, y_i \mathbf{v}\rangle) 2 q m \log (2 q / \kappa)-\frac{1}{m} \sum_{j, r} \sigma^{\prime}(\langle\mathbf{w}_{j, r}^{(t)}, y_i \mathbf{v}\rangle) \widetilde{O}(\sigma_0\|\mathbf{v}\|_2) \\
    -& \frac{1}{m} \sum_{j, r} \sigma^{\prime}(\langle\mathbf{w}_{j, r}^{(t)}, \bm{\xi}_i\rangle) \widetilde{O}(\sigma_0 \sigma_{\xi} \sqrt{d})-\frac{1}{m} \sum_{j, r} \sigma^{\prime}(\langle\mathbf{w}_{j, r}^{(t)}, \bm{\xi}_i \rangle) \widetilde{O}(m n d^{-1 / 2}),    
    \end{align*}
    where $(i)$ holds due to the definition of $\mathbf{w}^*$ and $(ii)$ follows from \cref{lem:xi_bound}. Moreover, according to \cref{lem:w_t-w_0_v}, we have that for $j = y_i$:
    $$
    \max _r\{\langle\mathbf{w}_{j, r}^{(t)}, \bm{\xi}_i \rangle\}=\max _r\{\bar{\Phi}_{j, r}^{(t)}+\langle\mathbf{w}_{j, r}^{(0)}, \bm{\xi}_i\rangle -8 n \sqrt{\frac{\log (4 n^2 / \delta)}{d}} \alpha - \eta \sum_{s=1}^t\langle\mathbf{z}_s, y_i\mathbf{v}\rangle\} \overset{(i)}{\geq} 1 .
    $$
    Here, $(i)$ holds due to the analysis in \cref{lem:noise_first_stage_app}. Additionally, we can also have
    \begin{align*}
        & |\langle\mathbf{w}_{j, r}^{(t)}, \mathbf{v}\rangle| \stackrel{(i)}{\leq}|\langle\mathbf{w}_{j, r}^{(0)}, \mathbf{v}\rangle|+|\Gamma_{j, r}^{(t)}| + |\eta \sum_{s=1}^t\langle\mathbf{z}_s, \mathbf{v}\rangle\ | \stackrel{(i i)}{\leq} \widetilde{O}(1) \\
        & |\langle\mathbf{w}_{j, r}^{(t)}, \bm{\xi}_i\rangle| \stackrel{(i i i)}{\leq}|\langle\mathbf{w}_{j, r}^{(0)}, \bm{\xi}_i\rangle|+|\underline{\Phi}_{j, r, i}^{(t)}|+|\bar{\Phi}_{j, r, i}^{(t)}|+8 n \sqrt{\frac{\log (4 n^2 / \delta)}{d}} \alpha + |\eta \sum_{s=1}^t\langle\mathbf{z}_s, \bm{\xi}_i\rangle\ | \stackrel{(i v)}{\leq} \widetilde{O}(1).
        \end{align*}
    Here, $(i)$ holds due to \cref{lem:w_t-w_0_v} and \cref{lem:xi_zt2}; $(ii)$ is given by \cref{lem:w_t-w_0_xi}; $(ii)$ and $(iv)$ follows from \cref{pro:bounds_of_coeeficients} and \cref{lem:xi_zt2}. Combining these results, we return to $y_i\langle\nabla f(\mathbf{W}^{(t)}, \mathbf{x}_i), \mathbf{W}^*\rangle$, which gives:
    $$
    y_i\langle\nabla f(\mathbf{W}^{(t)}, \mathbf{x}_i), \mathbf{W}^*\rangle \geq 2 q \log (2 q / \kappa)-\widetilde{O}(\sigma_0\|\mathbf{v}\|_2)-\widetilde{O}(\sigma_0 \sigma_{\xi} \sqrt{d}) -\widetilde{O}(m n d^{-1 / 2}) \geq q \log (2 q / \kappa),
    $$
    where the last inequality is driven by the conditions as noise memorization that $\varepsilon \geq 1/{q\log(2q/\kappa)}$.
\end{proof}

\begin{lemma}\label{lem:noise_wt-w*-}
    Under the same conditions as noise memorization, it holds that
    \begin{align*}
        \|\mathbf{W}^{(t)}-\mathbf{W}^*\|_F^2-\|\mathbf{W}^{(t+1)}-\mathbf{W}^*\|_F^2 &\geq(2 q-1) \eta L_D(\mathbf{W}^{(t)})-\eta \kappa -\eta^2 \widetilde{O}(d\sigma_z^2) \\
        &- \eta \widetilde{O}(\sigma_z m^2 n^{1 / 2} \sigma_{\xi}^{-1} d^{-1 / 2})-{O}(\sigma_z n m \sigma_0)
    \end{align*}
for all $T_1 \leq t \leq T^*$.
\end{lemma}

\begin{proof}[\bf Proof of \cref{lem:noise_wt-w*-}]
According to the optimization properties, we know the first equality holds:
    \begin{equation}
    \begin{aligned}
     \|\mathbf{W}^{(t)}&-\mathbf{W}^*\|_F^2-\|\mathbf{W}^{(t+1)}-\mathbf{W}^*\|_F^2 \\
    & =2 \eta\langle\nabla L_S(\mathbf{W}^{(t)}), \mathbf{W}^{(t)}-\mathbf{W}^*\rangle-\eta^2\|\nabla L_S(\mathbf{W}^{(t)})\|_F^2 \\
    & \overset{(i)}{=}\frac{2 \eta}{n} \sum_{i=1}^n \ell_i^{\prime(t)}[q y_i f(\mathbf{W}^{(t)}, \mathbf{x}_i)-\langle\nabla f(\mathbf{W}^{(t)}, \mathbf{x}_i), \mathbf{W}^*\rangle] + {\eta}{\langle \mathbf{z}_t, \mathbf{W}^{(t)}-\mathbf{W}^*\rangle}\\
    & -\eta^2(O(\max \{\|\mathbf{v}\|_2^2, \sigma_{\xi}^2 d\}) L_D(\mathbf{W}^{(t)}) + O(\sigma_{z}^2 d \log (1 / \delta))) \\
    & \overset{(ii)}{\geq} \frac{2 \eta}{n} \sum_{i=1}^n \ell_i^{\prime(t)}[q y_i f(\mathbf{W}^{(t)}, \mathbf{x}_i)-q \log (2 q / \kappa)] \\
    &  + {\eta}{\langle \mathbf{z}_t, \mathbf{W}^{(t)}-\mathbf{W}^*\rangle} -\eta^2(O(\max \{\|\mathbf{v}\|_2^2, \sigma_{\xi}^2 d\}) L_D(\mathbf{W}^{(t)}) + O(\sigma_{z}^2 d \log (1 / \delta))) \\
    & \overset{(iii)}{\geq}  \frac{2 q \eta}{n} \sum_{i=1}^n[\ell(y_i f(\mathbf{W}^{(t)}, \mathbf{x}_i))-\kappa /(2 q)]\\
    & \eta \widetilde{O}(\sigma_z m^2 n^{1 / 2} \sigma_{\xi}^{-1} d^{-1 / 2}) -\eta^2(O(\max \{\|\mathbf{v}\|_2^2, \sigma_{\xi}^2 d\}) L_D(\mathbf{W}^{(t)}) + O(\sigma_{z}^2 d \log (1 / \delta))) \\
    & \overset{(iv)}{\geq} (2 q-1) \eta L_D(\mathbf{W}^{(t)})-\eta \kappa -\eta^2 \widetilde{O}(d\sigma_z^2) - \eta \widetilde{O}(\sigma_z m^2 n^{1 / 2} \sigma_{\xi}^{-1} d^{-1 / 2}).
    \end{aligned}
    \end{equation}
    Here, $(i)$ holds due to the definition of noisy gradient, the neural network is $q$ homogeneous, and \cref{lem:loss_F}; $(ii)$ is driven from \cref{lem:nabla_f_W_*_noise}; $(iii)$ is due to the convexity of the cross entropy function; $(iv)$ comes from definition of $L_D$.
\end{proof}

\begin{lemma}[Restatement of \cref{thm:main_noise_loss}]\label{thm:main_noise_loss_app}
    Let $T, T_1$ be defined in respectively. Then under the same conditions as signal learning, for any $t \in[T_1, T]$, it holds that $|\Gamma_{j, r}^{(t)}| \leq \sigma_0 \|\mathbf{v}\|_2$ for all $j \in\{ \pm 1\}$ and $r \in[m]$. Moreover, let $\mathbf{W}^*$ be the collection of CNN parameters with convolution filters $\mathbf{w}_{j, r}^*=\mathbf{w}_{j, r}^{(0)}+2 q m \log (2 q / \kappa))[\sum_{i=1}^n \mathbbm{1}(j=y_i) \cdot \frac{\bm{\xi}_i}{\|\bm{\xi}_i\|_2}]$. Then the following bound holds
    \begin{align*}
        \frac{1}{t-T_1+1} \sum_{s=T_1}^t L_D(\mathbf{W}^{(s)}) &\leq \frac{\|\mathbf{W}^{(T_1)}-\mathbf{W}^*\|_F^2}{(2 q-1) \eta(t-T_1+1)}  +\frac{\kappa}{(2 q-1)} + \underbrace{\frac{\eta d \sigma_z^2 + \widetilde{O}(\sigma_z m^2 n^{1 / 2} \sigma_{\xi}^{-1} d^{-1 / 2})}{(2 q-1)} }_{\text{Private terms}}
    \end{align*}
    for all $t \in[T_1, T]$, where we denote $\|\mathbf{W}\|_F=\sqrt{\|\mathbf{W}_{+1}\|_F^2+\|\mathbf{W}_{-1}\|_F^2}$.
\end{lemma}

\begin{proof}[\bf Proof of \cref{thm:main_noise_loss_app}]
    According to \cref{lem:signal_wt-w*-}, we have, for any $t\leq T$:
    $$
        \|\mathbf{W}^{(t)}-\mathbf{W}^*\|_F^2-\|\mathbf{W}^{(t+1)}-\mathbf{W}^*\|_F^2 \geq(2 q-1) \eta L_D(\mathbf{W}^{(t)})-\eta \kappa -\eta^2 \widetilde{O}(d\sigma_z^2) - \eta \widetilde{O}(\sigma_z m^{3/2} \|\mathbf{v}\|_2^{-1}).
    $$
    By summing over all terms and dividing $t-T_1-1$ on both sides, we obtain:
    \begin{align*}
        \frac{1}{t-T_1+1} \sum_{s=T_1}^t L_D(\mathbf{W}^{(s)}) \leq \frac{\|\mathbf{W}^{(T_1)}-\mathbf{W}^*\|_F^2}{(2 q-1) \eta(t-T_1+1)}  +\frac{\kappa}{(2 q-1)} + \frac{\eta d \sigma_z^2 + \widetilde{O}(\sigma_z m^2 n^{1 / 2} \sigma_{\xi}^{-1} d^{-1 / 2})}{(2 q-1)}. 
    \end{align*}
    If we have  $T=T_1+\lfloor\frac{\|\mathbf{W}^{(T_1)}-\mathbf{W}^*\|_E^2}{2 \eta \kappa}\rfloor = \frac{C m n \varepsilon}{\eta \mu(\|\mathbf{v}\|_2+\|\boldsymbol{\xi}\|_2)} \geq \kappa^{-1}$, then it holds that
    $$
    \frac{\|\mathbf{W}^{(T_1)}-\mathbf{W}^*\|_F^2}{(2 q-1) \eta(T-T_1+1)}+\frac{\kappa}{2 q-1} \leq \frac{3 \kappa}{2 q-1},
    $$
    and with $\sigma_z=\frac{1}{\eta \mu \sqrt{T}}$:
    $$
    \frac{\eta d \sigma_z^2 + \widetilde{O}(\sigma_z m^2 n^{1 / 2} \sigma_{\xi}^{-1} d^{-1 / 2})}{(2 q-1)} \leq \frac{d}{\eta \mu^2 T(2q-1)} + \frac{m^2 n^{1/2} \|\bm{\xi}\|_2^{-1} \|\mathbf{v}\|_2^{-1}}{\eta \mu \sqrt{T}(2q-1)} \overset{(i)}{\leq} \frac{\kappa}{(2q-1)},
    $$
    where $(i)$ comes from the assumption of $\eta$. 
    Therefore, combining above results, we conclude that 
    \begin{align*}
        \frac{1}{t-T_1+1} \sum_{s=T_1}^t L_D(\mathbf{W}^{(s)}) \leq \kappa. 
    \end{align*}
    Moreover, we will use induction to prove that $\max_{j,t} |\Gamma_{j,r}^t| \leq 2\sigma_0\|\mathbf{v}\|_2$ holds for all $t\in [T_1,T]$. According to \cref{lem:app_signal_first_stage}, we know it holds for $T_1$. Now, assume it holds for some $t \in [T_1, T)$, and we will show that it also holds for $t+1$.
    \begin{align*}
     \Gamma_{j, r}^{(t)}& =\Gamma_{j, r}^{(T_1)}-\frac{\eta}{n m} \sum_{s=T_1}^{t-1} \sum_{i=1}^n \ell_i^{(t)} \cdot \sigma^{\prime}(\langle\mathbf{w}_{j, r}^{(0)}, y_i \cdot \mathbf{v}\rangle + \Gamma_{j,r}^s- \langle\mathbf{z}_s, \mathbf{v} \rangle)\|\mathbf{v}\|_2^2, \\
    & \stackrel{(i)}{\leq} \Gamma_{j, r}^{(T_1)}+\frac{q 5^{q-1} \eta}{n m}\|\mathbf{v}\|_2^2 (\sigma_0 \|\mathbf{v}\|_2)^{q-1} \sum_{s=T_1}^{t-1} \sum_{i=1}^n|\ell_i^{(t)}| \\
    & \stackrel{(i i)}{\leq} \Gamma_{j, r}^{(T_1)}+q 5^{q-1} \eta m^{-1}\|\mathbf{v}\|_2^2 (\sigma_0 \|\mathbf{v}\|_2)^{q-1} \sum_{s=T_1}^{t-1} L_S(\mathbf{W}^{(s)}) \\
    & \stackrel{(i i i)}{\leq} \Gamma_{j, r}^{(T_1)}+(\sigma_0 \|\mathbf{v}\|_2)^{q-1} \widetilde{O}(m^2 n \operatorname{SNR}^2) \\
    & \stackrel{(i v)}{\leq} \Gamma_{j, r}^{(T_1)}+(\sigma_0 \|\mathbf{v}\|_2)(n \varepsilon)^{-(q-2)/q} \widetilde{O}(m^2 n^{1-2 / q}) \\
    & \stackrel{(v)}{\leq} 2 \widehat{\beta}^{\prime}.
    \end{align*}
    Here, $(i)$ is due to induction hypothesis and the choice of $T$; $(ii)$ holds by $|\ell^{\prime}|\leq \ell$; $(iii)$ comes from \cref{lem:W_T1-W_*_noise} and the choice of $T$; $(iv)$ is driven from $\sigma_0 \leq (n \varepsilon)^{-1/q} \|\mathbf{v}\|_2^{-1}$ and $\operatorname{SNR}$; $(v)$ holds due to $n^{1/q}\varepsilon \geq m$.
\end{proof}

\begin{lemma}[Restatement of \cref{coro:popu_noise}]\label{lem:popu_noise}
    Under the same conditions as data noise memorization, within $T$ iterations, regardless of how the sample size $n$ and privacy budget $\varepsilon$ chosen, with at least probability $1-1/d$, we can find $\mathbf{W}^{(\widetilde{T})}$ such that $L_D(\mathbf{W}^{(\widetilde{T})}) \leq \kappa$. Additionally, for any $0 \leq t \leq \widetilde{T}$ we have that $L_{\mathcal{D}}(\mathbf{W}^{(t)}) \geq 0.1$.
\end{lemma}

\begin{proof}[\bf Proof of \cref{lem:popu_noise}]
    Consider a new sample $(\mathbf{x}, y)$ drawn from \cref{def:data_dis}, we have:
    \begin{align*}
    \|\mathbf{w}_{j, r}^{(t)}\|_2 &=\|\mathbf{w}_{j, r}^{(0)}+j \cdot \Gamma_{j, r}^{(t)} \cdot \frac{\mathbf{v}}{\|\mathbf{v}\|_2^2}+\sum_{i=1}^n \bar{\Phi}_{j, r, i}^{(t)} \cdot \frac{\boldsymbol{\xi}_i}{\|\boldsymbol{\xi}_i\|_2^2}+\sum_{i=1}^n \underline{\Phi}_{j, r, i}^{(t)} \cdot \frac{\boldsymbol{\xi}_i}{\|\boldsymbol{\xi}_i\|_2^2} -\eta\sum_{s=1}^t \mathbf{z}_s\|_2\\
    & \overset{(i)}{\leq} \|\mathbf{w}_{j, r}^{(0)}\|_2+\frac{\Gamma_{j, r}^{(t)}}{\|\mathbf{v}\|_2}+\sum_{i=1}^n \frac{\bar{\Phi}_{j, r, i}^{(t)}}{\|\boldsymbol{\xi}_i\|_2}+\sum_{i=1}^n \frac{|\Phi_{j, r, i}^{(t)}|}{\|\boldsymbol{\xi}_i\|_2} + \|\eta \sum_{s=1}^t \mathbf{z}_s\|_2 \\
    &\overset{(ii)}{\leq} O(\sigma_0 \sqrt{d})+\widetilde{O}(n \sigma_{\xi}^{-1} d^{-1 / 2}) + O(\eta \sqrt{td}\sigma_z)
    \end{align*}
    Here, $(i)$ is due to triangle inequality; $(ii)$ holds by \cref{lem:noise_first_stage_app} and \cref{pro:bounds_of_coeeficients}. Additionally, we know that $\langle\mathbf{w}_{j, r}^{(t)}, \boldsymbol{\xi}\rangle \sim \mathcal{N}(0, \sigma_{\xi}^2\|\mathbf{w}_{j, r}^{(t)}\|_2^2)$, it holds that with probability at least $1-1/4$, $|\langle\mathbf{w}_{j, r}^{(t)}, \boldsymbol{\xi}\rangle| \leq \widetilde{O}(\sigma_0 \sigma_{\xi} \sqrt{d}+n d^{-1 / 2} + \frac{\sqrt{d}\sigma_{\xi} \sigma_0}{\mu})$ due to $\sigma_z = \sigma_0 /(\eta \sqrt{T}\mu)$. Moreover, according to \cref{thm:main_noise_loss_app}, we have $\max _{j, r} \Gamma_{j, r}^{(t)} \leq \widetilde{O}(\sigma_0\|\mathbf{v}\|_2)$, which also indicates that $|\langle\mathbf{w}_{j, r}^{(t)}, \mathbf{v}\rangle| \leq\widetilde{O}(\sigma_0\|\mathbf{v}\|_2)$.

    Then, by the union bound, with probability at least $1-1/2$, we have
    \begin{align*}
    F_j(\mathbf{W}_j^{(t)}, \mathbf{x}) & =\frac{1}{m} \sum_{r=1}^m \sigma(\langle\mathbf{w}_{j, r}^{(t)}, y \mathbf{v}\rangle)+\frac{1}{m} \sum_{r=1}^m \sigma(\langle\mathbf{w}_{j, r}^{(t)}, \boldsymbol{\xi}\rangle) \\
    & {\leq} \max _r|\langle\mathbf{w}_{j, r}^{(t)}, \mathbf{v}\rangle|^q+\max _r|\langle\mathbf{w}_{j, r}^{(t)}, \boldsymbol{\xi}\rangle|^q \\
    & {\leq} \widetilde{O}(\sigma_0^q \sigma_{\xi}^q d^{q / 2}+n^q d^{-q / 2}+\sigma_0^q\|\mathbf{v}\|_2^q + \frac{\sigma_0^q \|\bm{\xi}\|_2^q}{\mu^q}) \\
    & \overset{(i)}{\leq} \widetilde{O}(\frac{1}{\varepsilon^q} + n^q d^{-q / 2}+\frac{1}{n\varepsilon} + \frac{1}{\varepsilon^q \mu^q}) \\
    & \overset{(ii)}{\leq} 1.
    \end{align*}
    Here, $(i)$ and $(ii)$ holds due to we restrict $\sigma_0 = \widetilde{O}(\varepsilon^{-1} \|\bm{\xi}\|_2 )$ and $\varepsilon^q \geq \widetilde{O}(1)$. Notice that here $1$ can be any number, and we use $1$ without loss of generality.
    Therefore, with probability at least $1-1/2$, we have $\ell(y \cdot f(\mathbf{W}^{(t)}, \mathbf{x})) \geq \log (1+e^{-1})$, which indicates that $L_{\mathcal{D}}(\mathbf{W}^{(t)}) \geq \log (1+e^{-1}) \cdot 0.5 \geq 0.1$.
\end{proof}

\end{document}